\pdfoutput=1
\documentclass[a4paper,oneside,11pt]{article}
\usepackage{a4wide, dsfont}
	\usepackage{epsfig}
\usepackage{graphicx}
\usepackage{multirow}
\usepackage[utf8]{inputenc}   
\usepackage{amsthm}%
\usepackage{amsmath}%
\usepackage{url}
\usepackage[round]{natbib}
\usepackage{color,soul} %
\usepackage[table]{xcolor}
\colorlet{revisioncolor}{black}
\colorlet{arxivcolor}{black}
\newcommand{\ArxivFlags}{0}
\makeatletter
\newcommand{\arxivflag}[2]{\ifnum\ArxivFlags=1\leavevmode\phantomsection%
  \addcontentsline{lac}{arxivflag}{\protect\numberline{#1}#2}%
  \marginpar{\raggedright\footnotesize\textcolor{arxivcolor}{\textbf{v3:\,#1}}}\fi}
\newcommand{\l@arxivflag}{\@dottedtocline{1}{0em}{2.2em}}
\newcommand{\listofarxivchanges}{\ifnum\ArxivFlags=1{\small\color{arxivcolor}\noindent
  \textbf{arXiv-only changes in this preliminary version} (flagged in the margin; entries link to the passages)\par
  \@starttoc{lac}}\bigskip\fi}
\makeatother
\ifnum\ArxivFlags=1 \fi
\usepackage{diagbox}
\usepackage{booktabs}

\usepackage{amssymb}%
\usepackage{booktabs}%
\usepackage{dcolumn}%
\newcolumntype{.}{D{.}{.}{-1}}%
\newcolumntype{Z}{D{.}{.}{2}}%
\newcolumntype{T}{D{.}{.}{3}}%
\newcolumntype{V}{D{.}{.}{4}}%

\usepackage{setspace}
\usepackage[final,todonotes={colorinlistoftodos,textsize=tiny}]{changes}
\newcommand{\Comments}{0}
\newcommand{\mynote}[2]{\ifnum\Comments=1\textcolor{#1}{#2}\fi}
\newcommand{\mytodo}[2]{\ifnum\Comments=1%
	\todo[linecolor=#1!80!black,backgroundcolor=#1,bordercolor=#1!80!black]{#2}\fi}

\ifnum\Comments=1 %
\fi

\usepackage{fancyhdr}
\usepackage{nicefrac}
\usepackage{marvosym}
\usepackage{enumitem}
\usepackage{ulem}

\newcommand{\R}{{\mathbb R}}

\newcommand{\rP}{{\mathrm P}} %
\newcommand{\PP}{{\mathbb P}} %

\newcommand\UU{ \mathcal{U} } %
\newcommand\LL{ \mathcal{L} } %

\newcommand{\ind}{\text{\bf{1}}} %

\DeclareMathOperator*{\argmax}{argmax}

\newcommand{\dd}{\text{\rm d}} %

\newcommand{\cF}{\mathcal{F}}

\newcommand{\cX}{\mathcal{X}}

\newcommand{\rmnum}[1]{{\ifmmode\text{\upshape\scshape\romannumeral#1}\else\textsc{\romannumeral#1}\fi}}

\theoremstyle{plain}
\newtheorem{theorem}{Theorem}%
\newtheorem{lemma}[theorem]{Lemma}%
\theoremstyle{definition}
\newtheorem{example}[theorem]{Example}%
\newtheorem{remark}[theorem]{Remark}%
\theoremstyle{plain}

\theoremstyle{definition}

\newtheoremstyle{assumpstyle}
  {}{} %
  {\itshape} %
  {} %
  {\bfseries} %
  {} %
  { } %
  {\thmname{#1}\ \thmnumber{#2}\thmnote{ (#3)}} %

\theoremstyle{assumpstyle}
\newtheorem*{assumptionC}{Assumption C}%
\numberwithin{equation}{section}

\renewcommand\theta{\vartheta}%

\renewcommand{\textfraction}{0.1}	
\renewcommand{\floatpagefraction}{0.8}	

\newcommand{\crr}{c_{\mathrm{r}}}	
\newcommand{\preci}{\mathrm{prec}}	

\usepackage[colorlinks, linkcolor=black, citecolor=black, filecolor=black, urlcolor=blue]{hyperref} 
\begin{document}
\onehalfspacing 
\LARGE
\title{Robust performance metrics for imbalanced classification problems}
	\author{Hajo Holzmann\footnote{Corresponding author. Hajo Holzmann, Fachbereich Mathematik und Informatik, Philipps-Universität Marburg, Hans-Meerweinstr.~6, 35043 Marburg, Germany} \\
		\small{Fachbereich Mathematik und Informatik}  \\
		\small{Philipps-Universit\"at Marburg} \\
		\small{holzmann@mathematik.uni-marburg.de}
		\and
		Bernhard Klar  \\
		\small{Institut f\"ur Stochastik}\\		\small{Karlsruher Institut f\"ur Technologie (KIT)} \\
		\small{bernhard.klar@kit.edu} }
	\date{}
	\maketitle
\normalsize
\begin{abstract}
We show that established performance metrics in binary classification, such as Matthews' correlation coefficient (MCC), Cohen's $\kappa$, the F-score or the Jaccard similarity coefficient are not robust to class imbalance in the sense that if the proportion of the minority class tends to $0$, the true positive rate (TPR) of the Bayes classifier under these metrics {becomes very small or even tends to $0$}. Thus, in imbalanced classification problems, these metrics favour classifiers which ignore the minority class. To alleviate this issue we introduce robustified modifications of the MCC, of Cohen's $\kappa$ and of the F-score with an additional tuning parameter which allows to adapt the amount of robustness against class imbalance. As theoretical guarantee we show that the Bayes-optimal classifier for these robustified performance metrics, when expressed in terms of the density ratio $f_1/f_0$ of the class-conditional densities $f_i$, has a threshold parameter which is upper-bounded in terms of the tuning parameters. Therefore, even in strongly imbalanced settings, the TPR associated to this classifier will be bounded away from $0${, provided that $f_1/f_0$ exceeds this bound with positive probability under the positive class}. We numerically illustrate the behaviour of the various performance metrics and the effect of the tuning parameters in simulations as well as on a credit default data set. 
We also discuss connections to the receiver operating characteristic and precision-recall curves, which provide an alternative perspective on the proposed notion of robustness, and give recommendations on how to combine their usage with performance metrics.
\end{abstract}


\bigskip

\bigskip

\section{Introduction}
Classifiers in binary classification are often evaluated using performance metrics. 
{Once this metric is optimized to choose the decision threshold or to select among various methods, it thereby selects the classifier itself.}
Among performance metrics the probability of correct classification, also called accuracy, is most prominent and is widely used in the more theoretically oriented literature \citep{zbMATH00893887, zbMATH05186310}. 
However, for imbalanced data in which one class is much rarer than the other, accuracy has severe shortcomings: even the Bayes-optimal classifier may effectively ignore the minority class when its proportion is sufficiently small.
Therefore, a variety of additional performance metrics such as the Matthews Correlation Coefficient (MCC), Cohen's $\kappa$, the F-score and the Jaccard similarity coefficient (JAC) are becoming predominant in practice%
, with a lively debate over which performance metric  to use in a specific situation \citep{Delgado:2019, Chicco:2020,Chicco:2021}.%

\vspace{2mm}

Recently it has been shown that for wide classes of performance metrics the associated Bayes classifiers, that is classifiers with maximal values of the performance metric, again have a threshold form in the conditional probability of the positive class given the feature \citep{Koyejo:2014, yan2018binary, Singh:2022}{, see \citet{Lipton:2014} for the earlier special case of the $F_1$-score}. %
In contrast to accuracy where it is simply $1/2$, this threshold and hence the optimal classifier in general depends on the performance metric as well as the proportion of the positive class and the class-conditional distributions.  

\vspace{2mm}

Imbalanced data often arises when data are sampled from the class conditional distributions and not from the overall joint distribution of the labels and the features. In such settings, the proportion of positive samples among all samples is effectively set by the data analyst who collects the data and does not reflect the actual proportion of the positive class in the population. To deal with such a sampling design, we derive the Bayes classifier in the form of thresholding the density ratio $1(f_1/f_0 > \delta^*)$ of the class-conditional densities $f_i$. The dependence of the threshold $\delta^*$ on the proportion of the positive class $\pi$ in the data, and in particular, the boundedness of $\delta^*$ as $\pi \to 0$, allows us to assess the robustness of the performance metrics and the associated Bayes classifiers under imbalanced designs. 

\vspace{2mm}

Next we show that unfortunately, popular metrics like MCC, Cohen's $\kappa$, the F-score or the Jaccard similarity coefficient are not robust against class imbalance in this  sense
even in the simple setting of linear discrimination analysis (LDA): the threshold $\delta^*$ becomes very large or even tends to infinity as  $\pi \to 0$. To alleviate this issue we propose robust modifications of the MCC{, Cohen's $\kappa$} and F-score with tuning parameters which allow to adjust the dependence of the optimal threshold $\delta^*$ on the proportion $\pi$. 

\vspace{2mm}

The rest of the paper is organized as follows. We start in Section \ref{sec:motivation} with a numerical illustration in the setting of LDA. First we make the point that the Bayes classifier, which is the target of any training algorithm, and hence its associated optimal confusion matrix depends on the choice of the performance metric as well as on the underlying characteristics of the classification problem. Therefore the effect of the choice of a performance metric should not be assessed by simply applying various metrics to given confusion matrices. Next, we illustrate the non-robustness of popular metrics for small values of $\pi$, and also show the performance of the robust modifications that we propose. Section \ref{sec:binclassbascis} sets up the theoretical framework. In Section \ref{sec:perfmetrics}, we begin by deriving an equation for the threshold $\delta^*$ of the Bayes classifier in terms of the density ratio $1(f_1/f_0 > \delta^*)$. Then we show the non-robustness of MCC, {Cohen's $\kappa$,} F-score and JAC for the population quantities in linear and quadratic discriminant analysis (LDA and QDA). In Section \ref{robust-metrics}, we define robust performance metrics and introduce the robustified modifications of the MCC (Section \ref{ex:mccgen}), of Cohen's $\kappa$ (Section \ref{ex:kappagen}) as well as of the F-score (Section \ref{ex:fgen}).
An extension of the framework of ROC and precision-recall curves is proposed in Section \ref{sec:rocprecisionrecall}, together with connections to the use of performance metrics. To allow for a joint analysis we suggest to rather use the plots of recall against 1-precision, which is on a similar scale and hence can be much better compared to the ROC curve. Optimal classifiers selected by a performance metric can then be visualized as points on these curves. 
Section \ref{data-ex} contains an application of the proposed methodology to a credit default data set. Section \ref{sec:conclude} concludes. {Appendix~\ref{sec:appperfmetrics} contains the proof of Theorem~\ref{th:decisiobnounddens} and further technical results, among them the non-existence of Bayes classifiers for Yule's coefficients under LDA (Appendix~\ref{app:yule}) and the existence of Bayes classifiers for the robust performance metrics (Appendix~\ref{app:robust}). The remaining appendices contain additional numerical results and computational details.} 
{An R-package \texttt{RobustMetrics} with implementation of the methods is available on CRAN\footnote{\url{https://cran.r-project.org/package=RobustMetrics}}, and the scripts for the numerical results are available on GitHub\footnote{\url{https://github.com/BernhardKlar/RobustMetrics-Rscripts}}.}

\vspace{2mm}
Let us conclude the introduction by mentioning that further methods and issues related to class imbalance include up sampling and down sampling techniques, the use of specialized supervised learning algorithms, model selection of the link function in binary regression, and cost-sensitive analysis. See e.g.~\cite{Choi:2010, delaCruz:2019, delaCruz:2025, Dixit:2023, Luque:2019, Yuan:2023, Zou:2016, Buda:2018, Saito:2015, Branco:2016}, among others.
In particular, \cite{delaCruz:2025} empirically compare several performance metrics for distinguishing a correctly specified Power-Cauchy binary regression model from a misspecified logistic model, using confusion
matrices obtained from a threshold selected through Cohen's $\kappa$; they report favorable results for MCC, the G-mean, and Cohen's $\kappa$. Their model-ranking problem differs from the present one, where we study the classifier induced by optimizing each metric itself and the robustness of its Bayes-optimal density-ratio threshold. 

\section{Performance metrics and class imbalance: Numerical illustrations}\label{sec:motivation}

Let us start with a numerical illustration of the main points to be made in this work. When evaluating classifiers one has to be aware that the optimal classifier and hence its associated confusion matrix depends on the choice of the performance measure such as the F-score or the MCC. Thus, the effect of using different performance measures cannot be assessed by simply computing and comparing them for given confusion matrices. Rather, the starting point has to be the distribution of the data to be classified.    

In Example \ref{ex-scen1} we illustrate this in the approximately balanced setting for the popular metrics MCC and $F_\beta$ score. Then in Example \ref{ex-scen2} we proceed to the behaviour of the optimal classifiers for these metrics when data become increasingly imbalanced.
In both cases we use the most simple and well-behaved setting of linear discriminant analysis. Thus, the issues that we highlight do not depend on sophisticated data-generating processes, but are pervasive to the most well-behaved frameworks. Further, let us point out that the formal definitions of the performance metrics, the notion of robustness, and the proposed modifications will be introduced in Sections~\ref{sec:binclassbascis}--\ref{robust-metrics}.

\begin{example}[Approximately balanced case] \label{ex-scen1}
Group 1: $X_{1j}\sim \mathcal N(\mu_1,\Sigma), \, j=1,\ldots,n_1,$ with $n_1=3000$ and
$$
\mu_1 = \begin{pmatrix} 2 \\ 2 \end{pmatrix}, \quad
\Sigma = \begin{pmatrix} \sigma^2 & \sigma\tau\rho \\
                         \sigma\tau\rho & \tau^2  \end{pmatrix}, 
\quad \text{where } \sigma=2, \tau=1, \rho=0.5.
$$
Group 0: $X_{0j}\sim \mathcal N(\mu_0,\Sigma), \, j=1,\ldots,n_0,$ with $n_0=7000$ and $\mu_0=(0,0)^\top$.
\item 
\smallskip
We apply logistic regression to obtain an estimate $\hat \eta(x)$ of the regression function. Let us mention that the logistic regression model is well-specified under the LDA assumption, see \citet{hastie2009elements}, so that estimating the regression function by logistic regression can be expected to work well in our simulation setting. We then use three different classifiers: $\phi_1$ classifies as positive (group 1) if  $\hat \eta(x)\geq 0.244$, $\phi_2$ decides for group 1 if $\hat \eta(x)\geq 0.436$, and $\phi_3$ uses a cutoff probability of $0.636$.
These bounds are asymptotically optimal for the performance measures that we shall consider, namely the $F_\beta$-score with $\beta=1.5$ and $0.5$, as well as the MCC, which we shall formally introduce and investigate in Section \ref{sec:binclassbascis}.
The three empirical confusion matrices for $\phi_1, \phi_2,\phi_3$ are given in Table \ref{tab-scen0}.

\begin{table}[ht]
\begin{center}
\begin{tabular}{cc|cc}
    && \multicolumn{2}{c}{$\phi_1$} \\
    & & 1 & 0 \\ \hline
\multirow{2}{*}{$Y$} & 1 & 2640 & 360 \\
                     & 0 & 1352 & 5648  \\
\end{tabular} \qquad
\begin{tabular}{cc|cc}
    && \multicolumn{2}{c}{$\phi_2$} \\
    & & 1 & 0 \\ \hline
\multirow{2}{*}{$Y$} & 1 & 2289 & 711 \\
                     & 0 & 673  & 6327  \\
\end{tabular} \qquad
\begin{tabular}{cc|cc}
    && \multicolumn{2}{c}{$\phi_3$} \\
    & & 1 & 0 \\ \hline
\multirow{2}{*}{$Y$} & 1 & 1879 & 1121 \\
                     & 0 & 329  & 6671 \\
\end{tabular}
\caption{Empirical confusion matrices for the classifiers in Example \ref{ex-scen1}. \label{tab-scen0}}
\end{center}
\end{table}

Next we apply the $F_\beta$-score with $\beta=1.5$ and $0.5$ as well as the MCC to these classifiers. The results can be found in Table \ref{tab-scen1}. If we choose $F_{1.5}$ for evaluation, the classifier $\phi_1$ with cutoff probability 0.244 performs best, followed by $\phi_2$. Choosing $F_{0.5}$ as the performance metric, the order is reversed, and $\phi_3$ performs best. Finally, under MCC, $\phi_2$ with cutoff 0.436 outperforms the other classifiers. The distinct optimal classifiers for these performance measures induce distinct confusion matrices. 
\begin{table}[ht]
\centering
\begin{tabular}{r|rrr}  \hline
 & $F_{1.5}$ & MCC & $F_{0.5}$  \\  \hline
$\phi_1$ & 0.80 & 0.64 & 0.70 \\ 
$\phi_2$ & 0.77 & 0.67 & 0.77 \\ 
$\phi_3$ & 0.68 & 0.64 & 0.79 \\ 
   \hline
\end{tabular}
\caption{Values of different performance measures for the three classifiers in Example \ref{ex-scen1}. \label{tab-scen1}}
\end{table}

\end{example}

\begin{example}[Imbalanced case] \label{ex-scen2}
Here, we consider the behavior of optimal classifiers for different performance metrics in increasingly more imbalanced classification tasks. We use the same scenario as in Example \ref{ex-scen1}, but with a total sample size of $10^5$.
We start with $n_1=30000$ and $n_0=70000$, so $\hat\pi=n_1/(n_1+n_0)=0.3$. Then we reduce $\hat\pi$ over $0.1,0.05$ and $0.025$ to $\hat\pi=0.01$.

For each sample, we estimate the regression function and determine the optimal thresholds $\tilde\delta$ that maximize the three metrics used in Example \ref{ex-scen1}, using a grid from 0.001 to 0.999 with step size 0.001. We classify  as positive (group 1) if  $\hat \eta(x)\geq \tilde \delta$, where $\hat \eta(x)$ is obtained from logistic regression and $\tilde \delta$ depends on the metric. 
In order to assess whether class $1$ can still be detected under this (optimal) classifier for a given metric in imbalanced settings, we then calculate the empirical true positive and negative rates $n_{ii}/n_i$, where $n_{ii}$ is the number of cases correctly classified as group $i \, (i=1,0)$. 

The values of the metrics, the threshold values $\tilde \delta$ and the empirical true positive and negative rates $n_{ii}/n_i$ are shown in the first four columns of Table \ref{tab-scen2}. While $\tilde \delta$ does decrease for  decreasing  values of $\hat\pi$, which is desirable for imbalanced classification problems in order to be able to detect class $1$, the true positive rate $\hat\pi_{1|1}$ still decreases to small values while the true negative rate 
 $\hat\pi_{0|0}$ increases to $1$ as $\hat\pi$ becomes small. 
This behavior indicates a lack of robustness of the corresponding classifiers with respect to class imbalance, in the sense that the true positive rate deteriorates as the class proportion decreases. This observation motivates the formal notion of robustness introduced in Section~\ref{robust-metrics}.
 
 In Section \ref{sec:binclassbascis} we argue that it is more informative  to compute the threshold values $\delta$ in terms of ratios $f_1/f_0$ of densities $f_i$ of class conditional distributions than for the regression functions in order to make them comparable for distinct proportions $\hat\pi$. See \eqref{eq:relationthreshold} for the relation between the  values of $\delta$ and $\tilde \delta$. Table \ref{tab-scen2} shows that the corresponding threshold values in terms of density ratios actually strongly increase, indicating that all three metrics fail to solve the problem of handling imbalanced data. 

We address this issue by proposing robustified versions of the $F$-score and the MCC in Section \ref{robust-metrics}. 
The last three columns of Table \ref{tab-scen2} show the results in the above setting for these metrics. We note that here, $\tilde\delta$ decreases much faster and the corresponding values of $\delta$ remain bounded, so that a high true positive rate $\hat\pi_{1|1}$ is still achieved.

\begin{table}
\centering
\begin{tabular}{c|c|ccc|ccl} \hline
 \addlinespace[0.5mm]
 $\hat{\pi}$ && $F_{1.5}$ & MCC & $F_{0.5}$ & $F_{\text{rb}}$ & $\text{MCC}_{\text{rb}}^{0.1}$   &$\text{MCC}_{\text{rb}}^{0.05}$  \\  \addlinespace[0.5mm] \hline \addlinespace[0.5mm]
     & value           & 0.794 & 0.665 & 0.797 & 0.788 & 0.667 & 0.665 \\ 
     & $\tilde \delta$ & 0.250 & 0.426& 0.644 &0.257 &0.362  &0.362 \\
0.3  & $\hat\pi_{1|1}$ & 0.871 &0.767 &0.612 &0.867 &0.807  &0.807 \\
     & $\hat\pi_{0|0}$ & 0.810 &0.899 &0.958 &0.814 &0.873  &0.873 \\ 
     & $\delta$   & 0.778 & 1.732 & 4.221 & 0.807 & 1.324  &1.324 \\ \hline \addlinespace[0.5mm]   
     &  value          & 0.640 & 0.570 & 0.667 & 0.680 & 0.597 & 0.584 \\ 
     & $\tilde \delta$ & 0.206 &0.358 &0.526 &0.124 &0.206  &0.221 \\
0.1  & $\hat\pi_{1|1}$ & 0.723 &0.579 &0.445 &0.814 &0.723  &0.707 \\
     & $\hat\pi_{0|0}$ & 0.923 &0.964 &0.985 &0.870 &0.923  &0.929 \\ 
     & $\delta$   & 2.335 & 5.019 & 9.987 & 1.274 & 2.335  &2.553 \\ \hline \addlinespace[0.5mm]
     & value           & 0.540 & 0.497 & 0.575 & 0.631 & 0.549 & 0.524 \\ 
     & $\tilde \delta$ & 0.179 &0.267 &0.496 &0.072 &0.132  &0.141 \\ 
0.05 & $\hat\pi_{1|1}$ & 0.612 &0.513 &0.319 &0.789 &0.678  &0.667 \\
     & $\hat\pi_{0|0}$ & 0.957 &0.976 &0.993 &0.883 &0.937  &0.941 \\ 
     & $\delta$   & 4.143 & 6.921 & 18.698 & 1.474 & 2.889  &3.119 \\ \hline \addlinespace[0.5mm]
     & value           & 0.444 & 0.418 & 0.489 & 0.604 & 0.519 & 0.480 \\ 
     & $\tilde \delta$ & 0.135 &0.242 &0.363 &0.038 &0.052  &0.084 \\
0.025& $\hat\pi_{1|1}$ & 0.532 &0.399 &0.290 &0.784 &0.727  &0.637 \\
     & $\hat\pi_{0|0}$ & 0.971 &0.988 &0.995 &0.888 &0.916  &0.949 \\ 
     & $\delta$   & 6.087 & 12.451 & 22.224 & 1.541 & 2.139  &3.576 \\ \hline \addlinespace[0.5mm]
     & value           & 0.344 & 0.321 & 0.382 & 0.592 & 0.508 & 0.454 \\ 
     & $\tilde \delta$ & 0.099 &0.166 &0.318 &0.018 &0.025  &0.036 \\
0.01 & $\hat\pi_{1|1}$ & 0.429 &0.315 &0.183 &0.764 &0.704  &0.641 \\
     & $\hat\pi_{0|0}$ & 0.986 &0.994 &0.998 &0.904 &0.929  &0.951 \\ 
     & $\delta$   & 10.878 & 19.705 & 46.161 & 1.815 & 2.538  &3.697 \\\hline 
     \end{tabular}
\caption{Values of metrics, values of optimal cutoff $\tilde\delta$ (in terms of regression function) and $\delta$ (in terms of density ratio) and class-conditional probabilities of classifiers that empirically maximize the various performance measures used in Example \ref{ex-scen2}. \label{tab-scen2}}
\end{table}

\end{example}

\section{Binary classification and performance metrics}\label{sec:binclassbascis}

\subsection{Binary classification}

Let us formally introduce the setting of binary classification. Let the random vector $(X,Y)\in \cX \times \{0,1\}$ contain the label $Y$ as well as the feature vector $X$, also called vector of explanatory variables, which may take values in some abstract space $\cX$. The joint distribution of $(X,Y)$ will be called the population distribution. It is determined by the class conditional distributions
$\rP_i = \PP(X | Y = i)$, $i=0,1,$ together with  
the unconditional probability $\pi = \PP(Y=1)$ of class $1$. 

First suppose that we may sample from the population distribution of $(X,Y)$, resulting in the training sample $(X_1,Y_1) \ldots, (X_n,Y_n)$. 
A classifier is a (measurable) map from $\cX$ to $\{0,1\}$ depending on the training sample. In the following we suppress this dependence and simply write $\phi: \cX \to \{0,1\}$. Probabilities and expected values then are computed conditionally on the training sample.  
 Writing
$$\pi_{jk} = \pi_{jk}(\phi)=\PP(Y=j,\phi (X)=k)$$
for the joint probabilities of $Y$ and $\phi(X)$, the $2\times 2$ - matrix $C_\phi=(\pi_{jk}(\phi))_{j,k=1,0}$ is called  population confusion matrix of the classifier $\phi$, summarized in Table \ref{tab-scen00}.
\begin{table}[ht]
\begin{center}
\begin{tabular}{cc|cc|c}
    && \multicolumn{2}{c}{$\phi(X)$} & \\
    & & 1 & 0 & Total \\ \hline
\multirow{2}{*}{$Y$} & 1 & $\pi_{11}(\phi)$  & $\pi_{10}(\phi)$ & $\pi$  \\
  & 0 & $\pi_{01}(\phi)$ & $\pi_{00}(\phi)$ & $1-\pi$ \\ \hline
  & Total & $\gamma$  & $1-\gamma$ & 1
\end{tabular}
\caption{Population confusion matrix of a classifier. \label{tab-scen00}}
\end{center}
\end{table}%

If we evaluate the classifier on a test data set we obtain empirical quantities $\hat \pi_{jk}(\phi)$ and the corresponding empirical confusion matrix $\hat C_\phi=(\hat \pi_{jk}(\phi))_{j,k=1,0}$ as in Section \ref{sec:motivation}. 

The (population) sensitivity of a classifier $\phi$, also called recall or true positive rate is the probability that observations from the positive class are actually detected as positive by the classifier. Formally,  
$$\pi_{1|1}(\phi) : =\pi_{11}/\pi = \PP( \phi (X)=1|Y=1 ) = \rP_1(\phi=1).$$
Note that in contrast to the true positive probability $\PP(Y=1,\phi (X)=1)$, for a given classifier $\phi$ it can be computed merely using $\rP_1$, while the  weight $\pi$ of the positive class is not required. Similarly, the specificity (true negative rate) is
$$\pi_{0|0}(\phi) : =\pi_{00}/(1-\pi) = \PP( \phi (X)=0|Y=0 ) = \rP_0(\phi=0).$$
Now for imbalanced data with small $\pi$, training may result in classifiers for which $\{ \phi=1\}$ is a small set and hence the sensitivity $\pi_{1|1}(\phi)$ then is also small, so that the positive class can no longer be detected. 
Consider regression-thresholding classifiers which have the form 
\begin{align*}
 \phi_{\tilde \delta, \eta}(x) &= \ind (\eta(x) > \tilde \delta ), %
\end{align*}
where $ \eta(x) = \PP(Y=1 \mid X=x)$, $x \in \cX$ is the regression function and $\tilde \delta \in (0,1)$ is a threshold. $\eta$ becomes small with small values of $\pi$, and if $\tilde \delta$ is too large, $\eta(x) > \tilde \delta $ only rarely occurs. Suppose that the class-conditional probabilities have densities $f_i$ w.r.t.~some $\sigma$-finite dominating measure $\mu$.
Then we can express the regression function in terms of the density ratio $f_1(x)/f_0(x)$ and vice versa by using the proportion $\pi$ as 
\begin{equation*} %
 \eta(x) = \frac{\pi\, f_1(x)/f_0(x)}{\pi f_1(x)/f_0(x) + (1-\pi) },\qquad \frac{f_1(x)}{f_0(x)} = \frac{(1-\pi)\, \eta(x)}{\pi\, (1-\eta(x))}, 
\end{equation*}
which implies for the  regression threshold classifier that
\begin{equation}\label{eq:relationthreshold}
    \eta(x)> \tilde \delta \quad \Longleftrightarrow \quad \frac{f_1(x)}{f_0(x)} > \delta \qquad \text{for} \quad \delta = \frac{(1-\pi)\, \tilde \delta}{\pi \, (1-\tilde \delta)},\quad \tilde \delta = \frac{\delta \pi}{\pi\,\delta + (1-\pi)}. 
\end{equation}
The representation of the thresholding classifier as $\phi_{\delta}(x) = \ind (f_1(x)/f_0(x) >  \delta )$ has the advantage that only the threshold $\delta$ may depend on the proportion $\pi$, which makes the effect of $\pi$ on the set $\{ \phi_{\delta} =1\}$ much more explicit: If $\delta$ gets large as $\pi$ decreases, this will make the set $\{ \phi_{\delta} =1\}$ and hence the sensitivity small. 

\subsection{Performance metrics} %
Performance metrics are used to measure the quality of classifiers $\phi$. Such metrics are functions 
$ \widetilde M\big(\pi_{11}, \pi_{01}, \pi_{10}, \pi_{00}\big) $ in the parameters of the confusion matrix of $\phi$ (either the population or an empirical version). We adopt the convention that higher values are desirable. 
Common examples include the following, where we use population quantities, suppress dependence on $\phi$ in $\pi_{jk} = \pi_{jk}(\phi)$ and denote $\gamma = \pi_{11} + \pi_{01}$:
\begin{itemize}
  \item {\itshape accuracy} $\pi_{11}+\pi_{00}$ and {\itshape weighted accuracy} $(w\, \pi_{11}+(1-w)\, \pi_{00})$, $w \in (0,1)$,
  \item {\itshape balanced accuracy} $(\pi_{11}/\pi + \pi_{00}/(1-\pi))/2$, the average of sensitivity and specificity,
  \item {\itshape G-mean score} $\sqrt{\pi_{11}\, \pi_{00}/(\pi\,(1-\pi))}$, the geometric mean of sensitivity and specificity,
  \item  {\itshape Jaccard similarity coefficient (JAC)} \ $\frac{\pi_{11}}{\gamma+\pi_{10}}$,
  \item {\itshape $F_{\beta}$-score} \ $(1+\beta^2)\, \frac{\pi_{11}}{\beta^2 \pi + \gamma}$, 
  \item {\itshape Matthews Correlation Coefficient (MCC)} \
    $\frac{\pi_{11}-\gamma\pi}{\sqrt{\gamma(1-\gamma)\pi(1-\pi)}}$,
  \item {\itshape Cohen's $\kappa$} \
    $\frac{2(\pi_{11}\, \pi_{00} - \pi_{01}\, \pi_{10})}{\pi \, (1-\gamma) + (1-\pi)\, \gamma}$,
   \item {\itshape Yule's coefficients of association and colligation}, with formulas in {Appendix \ref{app:yule}}. 
\end{itemize}

There exists a multitude of further criteria; a list of 76 binary similarity and distance measures can be found in \cite{Choi:2010}.
For obtaining suitable performance metrics the function $ \widetilde M\big(\pi_{11}, \pi_{01}, \pi_{10}, \pi_{00}\big) $ needs to satisfy certain properties. A quite general condition, suggested in \cite{Singh:2022}, %
is that  
\begin{equation}\label{eq:confmatrixall}
\widetilde M\big(\pi_{11}, \pi_{01}, \pi_{10}, \pi_{00}\big) \leq \widetilde M\big(\pi_{11} +\epsilon_1, \pi_{01} -\epsilon_2, \pi_{10} -\epsilon_1, \pi_{00}+\epsilon_2\big)
\end{equation}
for $\epsilon_1\in [0,\pi_{10}]$ and $\epsilon_2\in [0,\pi_{01}]$: Increasing true positives and true negatives when keeping overall positives, that is $\pi = \pi_{11} + \pi_{10}$ (and hence overall negatives) constant is always desirable and should lead to higher performance values. 

While $\widetilde M$ is just a function of four parameters, the confusion matrix $C_\phi$ arises from a given classifier $\phi$, so that   
$$ \widetilde M\big(\pi_{11}(\phi), \pi_{01}(\phi), \pi_{10}(\phi), \pi_{00}(\phi)\big) = \widetilde M(C_\phi) $$
is the value of the performance metric of $\phi$. A Bayes classifier for the performance metric associated with $\widetilde M$ satisfies
$$ \phi^*_{{\widetilde M}} \in \argmax_{\phi} \widetilde M(C_\phi) .$$

\cite{Singh:2022} showed that if the metric satisfies \eqref{eq:confmatrixall}, 
{then if there exists a Bayes classifier it can be chosen as a regression-thresholding classifier
 $\phi_{ \tilde \delta^*, \eta}(x) = \ind (\eta(x) > \tilde \delta^* )$, with an appropriate randomization on $\{\eta(x) = \tilde \delta^*\}$ if required; if $\eta(X)$ has a continuous distribution, no randomization is needed. %
Here,} as already illustrated in Section \ref{sec:motivation}, the optimal threshold $\tilde \delta^*$ and hence the Bayes classifier will depend on the metric $\widetilde M$, and in general also on the parameters of the population distribution of $(X,Y)$, so that $\tilde \delta^* = \tilde \delta^* (\widetilde M, \pi, \rP_0, \rP_1)$.
Using \eqref{eq:relationthreshold} we rewrite this classifier in terms of the density ratio as
\begin{align}
\phi_{ \delta^*}(x) &= \ind (f_1(x)/f_0(x) >  \delta^* ), \label{eq:Bayes3} 	
\end{align}
and investigate the dependence of $\delta^* = \delta^* (\widetilde M, \pi, \rP_0, \rP_1)$, in particular on $\pi$ and $\widetilde M$.

\section{Performance metrics and classification under class imbalance } \label{sec:perfmetrics}

Let us begin by determining the optimal threshold $\delta^*$ in \eqref{eq:Bayes3}. We shall follow the arguments in \citet{yan2018binary}, but present them for the threshold classifier in terms of the density ratio instead of the regression function. This way we show that the analysis does not require the joint population distribution of $(X,Y)$, but rather only the class-conditional distributions $\rP_i$, while the proportion $\pi$ can be treated as an exogenously given parameter. This is useful if, as in Examples \ref{ex-scen1} and \ref{ex-scen2}, a sample from the population distribution is not available and only samples from each class  conditional distribution can be obtained, that is
\begin{equation}\label{eq:samplingcond}
X_{0,1}, \ldots, X_{0,n_0} \ \text{ i.i.d. } \sim \rP_0,\qquad X_{1,1}, \ldots, X_{1,n_1} \ \text{ i.i.d. } \sim \rP_1.
\end{equation}
Considering 
$$(X_{0,1},0), \ldots, (X_{0,n_0},0), (X_{1,1},1), \ldots, (X_{1,n_1},1)$$
as sampled from the overall population would lead to $\pi = n_1 / (n_0 + n_1)$. However, this is no population quantity as $n_0$ and $n_1$ are chosen by the data analyst. Often, collecting data from the null distribution is relatively easy, so that $n_0$ can be taken quite large, while data in the population of positives is costly to obtain, overall resulting in imbalanced samples.   
We rewrite the confusion matrix in the sensitivity $\pi_{1|1}$, specificity $\pi_{0|0}$ and proportion of positives $\pi$ as 
\begin{center}
	\begin{tabular}{cc|cc|c}
		&& \multicolumn{2}{c}{$\phi(X)$} & \\
		& & 1 & 0 & Total \\ \hline
		\multirow{2}{*}{$Y$} & 1 & $\pi_{1|1}\, \pi$  & $(1-\pi_{1|1})\, \pi$ & $\pi$  \\
		& 0 & $(1-\pi_{0|0}) \, (1-\pi)$ & $\pi_{0|0}\, (1-\pi)$ & $1-\pi$ \\ \hline
		& Total & $\gamma$  & $1-\gamma$ & 1
	\end{tabular}
\end{center}
and parametrize the performance metric in these parameters as
\begin{equation}\label{eq:confmatmeas}
M(\pi_{1|1}, \pi_{0|0}, \pi) = \widetilde M\big(\pi_{1|1}\, \pi, (1-\pi_{0|0})\, (1-\pi), (1-\pi_{1|1})\pi,\pi_{0|0} (1-\pi)\big),
\end{equation}
see also \citet{narasimhan2014statistical}. Note again that $\pi_{1|1}$ and $\pi_{0|0}$ are accessible solely from $\rP_0$ and $\rP_1$, that is a sample of the form \eqref{eq:samplingcond}. 

\begin{assumptionC}[Consistency of performance metrics]\label{ass:boundedderv}
    Assume that the partial derivatives
    $$\partial_{\pi_{1|1}} M(\pi_{1|1}, \pi_{0|0}, \pi) \qquad \text{and} \qquad \partial_{\pi_{0|0}} M(\pi_{1|1}, \pi_{0|0}, \pi)$$
    are strictly positive for all $\pi_{1|1}, \pi_{0|0}, \pi \in (0,1)$. 
\end{assumptionC}

Indeed, higher values of the true positive and the true negative rates should increase utility. Note that this corresponds to the monotonicity as assumed in \eqref{eq:confmatrixall}. 
Since $\pi = \pi_{11} + \pi_{10} = (\pi_{11} + \epsilon_1) + (\pi_{10} - \epsilon_1)$ remains fixed, we must have that $\pi_{1|1}$ is increased to $\pi_{1|1} + \epsilon_1/\pi$, and similarly for $\pi_{0|0}$. 

\begin{theorem}\label{th:decisiobnounddens}
    Under Assumption~C, any Bayes classifier 
    $$ \phi^* \in \argmax_\phi M\big(\pi_{1|1}(\phi),\pi_{0|0}(\phi),\pi \big)$$
    {with $\pi_{1|1}(\phi^*), \pi_{0|0}(\phi^*) \in (0,1)$}
    is of the form 
    $$ \phi^*(x) = \phi_{\delta^*}(x) = \left\{ \begin{array}{lr}
        1, \quad  & f_1(x) > \delta^* f_0(x), \\
        0, \quad  & f_1(x) < \delta^* f_0(x),
        \end{array}\right.$$
     where $\delta^*$ necessarily satisfies the fixed-point equation
     \begin{equation}\label{eq:formofdelta}
     \delta^* = \frac{\partial_{\pi_{0|0}} M\big(\pi_{1|1}(\phi_{\delta^*}),\pi_{0|0}(\phi_{\delta^*}),\pi \big)}{\partial_{\pi_{1|1}} M\big(\pi_{1|1}(\phi_{\delta^*}),\pi_{0|0}(\phi_{\delta^*}),\pi \big)}.
     \end{equation}
\end{theorem}
For completeness we provide the proof along the lines of \citet{yan2018binary} in {Appendix \ref{app:proofth3}}. 

Let us discuss the result in Theorem \ref{th:decisiobnounddens}. 
{First we mention that the theorem does not guarantee existence of a Bayes-optimal classifier{: for Yule's coefficients no Bayes classifier exists under LDA (Appendix~\ref{app:yule}), even though the equation \eqref{eq:formofdelta} does have the solution $\delta=1$. For the robust metrics of Section~\ref{robust-metrics}, in contrast, existence is shown in Appendix~\ref{app:robust}}.

Now,  if a Bayes-optimal classifier $\phi^*$ with  $\pi_{1|1}(\phi^*), \pi_{0|0}(\phi^*) \in (0,1)$ exists}, then its
 Bayes-optimal threshold $\delta^*$ which is determined by the fixed point equation  \eqref{eq:formofdelta} depends on the metric $M$, the proportion $\pi$, and also on the class-conditional distributions  $\rP_0$ and $\rP_1$, which determine the true positive and negative rates occurring in \eqref{eq:formofdelta}. The effect of the performance metric on the optimal $\delta^*$ is fully captured by the quotient of its partial derivatives, 
\begin{equation}\label{eq:thequotient}
\frac{\partial_{\pi_{0|0}} M\big(\pi_{1|1},\pi_{0|0},\pi \big)}{\partial_{\pi_{1|1}} M\big(\pi_{1|1},\pi_{0|0},\pi \big)}.
\end{equation}
In particular, only values which can be attained by this quotient can occur for $\delta^*$. Further 
 if this quotient coincides for two performance metrics they have the same optimal thresholds for all parameters $\rP_i$ and $\pi$ of the data-generating mechanism.

Therefore, in {Example \ref{ex-metric-deriv1} in Appendix \ref{app:proofth3}} we give expressions for the quotient in \eqref{eq:thequotient} 
for the G-mean score, for the JAC and the $F_\beta$ score as well as for the MCC and Cohen's $\kappa$. These coincide for the JAC and the $F_1$ - score. Thus interestingly the JAC and the $F_1$ - score always lead to the same  Bayes-optimal classifiers,  {see also \citet{Lipton:2014} and \citet{Zhao:2013}.}

Next we show that even in the most simple classification settings of LDA and QDA, these metrics do not deal appropriately with imbalanced classification problems in that the threshold $\delta^*$ in \eqref{eq:formofdelta} in case of the MCC becomes very large, and for the other metrics even diverges to $\infty$ as $\pi \to 0$, resulting in {small or even} near-zero values for the sensitivity.    

\begin{example}[LDA] \label{ex:lda}
Suppose that $\rP_i = \mathcal N(\mu_i,\Sigma)$ with $\Sigma$ invertible. For the classifier $\phi_\delta = 1(f_1/f_0 > \delta)$, setting $\Delta^2 = (\mu_1 - \mu_0)^\top\, \Sigma^{-1}\, (\mu_1 - \mu_0) $, we have that 
\begin{equation}\label{eq:LDAdependdelta}
 \pi_{0|0}(\delta) = \Phi\Big( \frac{\log \delta + \Delta^2/2}{\Delta}\Big),\qquad \pi_{1|1}(\delta) = \Phi\Big( \frac{-\log \delta + \Delta^2/2}{\Delta}\Big),
\end{equation} 
where $\Phi$ is the distribution function of $\mathcal N(0,1)$.
For completeness we recall the derivation in {Example \ref{ex:ldaqdaetchnical} in Appendix \ref{app:ldaqda}}. Plugging this into the expressions for \eqref{eq:thequotient} and equating with $\delta$, see \eqref{eq:formofdelta} gives optimal values $\delta^*$.
Figure \ref{fig:lda1} shows $\delta^*$ as a function of $\pi$ for  
balanced accuracy, G-mean, JAC and the $F_\beta$ score as well as for the MCC {and Cohen's $\kappa$} for values $\Delta=1$ and $\Delta=2$ of the Mahalanobis distance. 
For  balanced accuracy and the G-mean, the optimal threshold is constant with $\delta^*=1$, independently of $\pi$, which does not adapt to class imbalance and is therefore typically not desirable.

In all other cases, $\delta^*$ becomes very large if $\pi$ tends to zero. For the Jaccard similarity coefficient (and, hence, the $F_1$-score) -- see Table \ref{lim-JAC} as well as the $F_{\beta}$-scores and Cohen's $\kappa$ (see Tables \ref{lim-F15}--\ref{lim-Kap} in Appendix \ref{app:ex:lda}), $\delta^*$ tends to infinity as $\pi\to 0$. For the MCC, $\delta^*$ converges to a finite limit as $\pi\to 0$, {which grows rapidly with $\Delta$,} see Table \ref{lim-MCC} in Appendix \ref{app:ex:lda}. These large values of $\delta^*$ result in very small values of the sensitivity, while the specificity is near $1$, see Table \ref{tab:LDAspesen-part} for a few selected values and Table \ref{tab:LDAspesen} in Appendix \ref{app:ex:lda}. 
{Finally, we mention that for Yule's coefficients of association and colligation, in Example \ref{ex-yule-deriv} in {Appendix \ref{app:yule}} we show that no Bayes classifier exists in the LDA setting.  } 

\medskip
\begin{table}[ht]
\small 
\centering
\begin{tabular}{c|rrrrrrrrrr}
  \hline
    & & \multicolumn{8}{c}{$\pi=\PP(Y=1)$}   \\
$\Delta$ & $10^{-10}$ & $10^{-9}$ & $10^{-8}$ & $10^{-7}$ & $10^{-6}$ & $10^{-5}$ & $10^{-4}$ & $10^{-3}$ & $0.01$ & $0.1$ \\ 
  \hline
  1.00 &269.43 &186.43 &126.15 & 83.11 & 52.99 & 32.42 & 18.79 & 10.11 &  4.85 & 1.86 \\ 
  1.50 & 3.4e3 & 2.0e3 & 1.1e3 &602.13 &311.30 &152.04 & 69.04 & 28.43 & 10.17 & 2.82 \\ 
  2.00 & 3.5e4 & 1.7e4 & 8.0e3 & 3.6e3 & 1.5e3 &595.58 &214.56 & 68.95 & 18.93 & 3.97 \\ 
  3.00 & 2.0e6 & 6.8e5 & 2.3e5 & 7.0e4 & 2.0e4 & 5.3e3 & 1.2e3 &258.47 & 45.80 & 6.33 \\ 
  \hline
\end{tabular}
\caption{Limiting behavior of the optimal threshold $\delta^*$ of JAC and $F_1$-score for $\pi\to 0$ and different values of $\Delta$. \label{lim-JAC}}
\end{table}

\begin{table}[ht]
\small 
\centering
\begin{tabular}{r|rrrrrrrr}  \hline
    & \multicolumn{8}{c}{$\pi=\PP(Y=1)$}   \\
    & $0.0001$ && $0.001$ && $0.01$ && $0.1$ & \\  \hline
JAC & 18.80 &  & 10.10 &  & 4.85 &  & 1.87 &   \\ 
    & 0.01 & 1.00 & 0.03 & 1.00 & 0.14 & 0.98 & 0.45 & 0.87 \\ 
$F_{1.5}$ & 15.30 &  & 7.93 &  & 3.61 &  & 1.26 &  \\ 
    & 0.01 & 1.00 & 0.06 & 0.99 & 0.22 & 0.96 & 0.61 & 0.77 \\ 
$F_{0.5}$ & 26.30 &  & 14.80 &  & 7.66 &  & 3.34 &  \\ 
    & 0.00 & 1.00 & 0.01 & 1.00 & 0.06 & 0.99 & 0.24 & 0.96 \\ 
  \hline
\end{tabular}
\caption{Lines 1,3,5: Optimal threshold $\delta^*$ for JAC, $F_{1.5}$ and $F_{0.5}$, using $\Delta=1$. 
Lines 2,4,6: Corresponding conditional probabilities $\pi_{1|1}$ and $\pi_{0|0}$.}\label{tab:LDAspesen-part}
\end{table}

\begin{figure}   
\centerline{\includegraphics[scale=0.55]{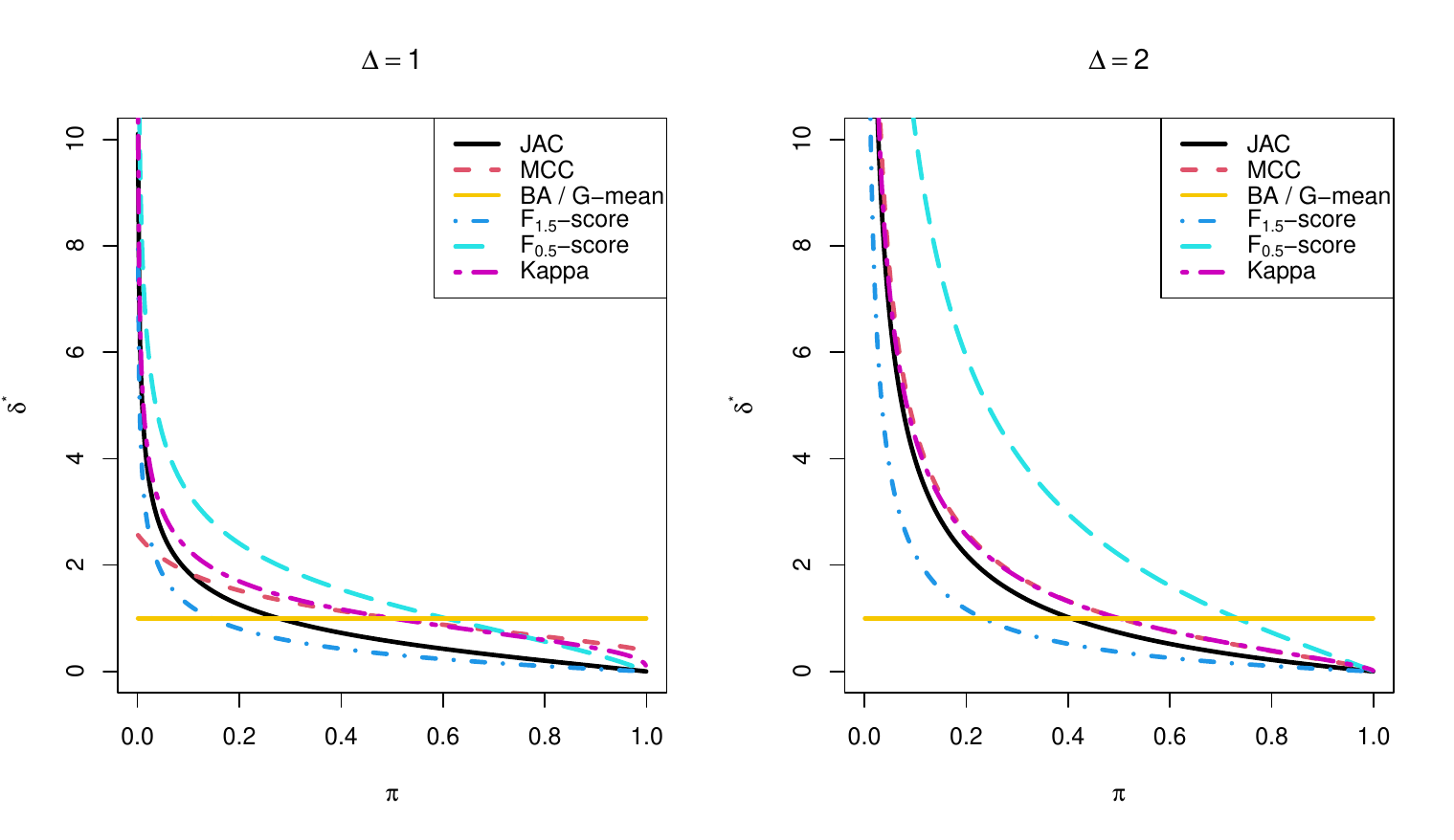}}
\caption{Plots of the optimal threshold $\delta^\ast$ as a function of $\pi$ for the different performance metrics used in Example \ref{ex:lda}}  \label{fig:lda1}
\end{figure}
\end{example}

\begin{example}[QDA]\label{ex:qda}
Let us briefly illustrate for quadratic discriminant analysis that the results of Example \ref{ex:lda} extend to less symmetrical situations.  
Suppose that $\rP_i = \mathcal N(\mu_i,\Sigma_i)$ with $\Sigma_i$ invertible. For the classifier $\phi_\delta(x) = 1(f_1(x)/f_0(x) > \delta)$, we have that $\phi(x) = 1$ if and only if
\begin{align*}
& x^\top \big(\Sigma_1^{-1} - \Sigma_0^{-1} \big) x - 2\, x^\top \big(\Sigma_1^{-1}\mu_1 - \Sigma_0^{-1}\mu_0 \big) + \mu_1^\top \Sigma_1^{-1}\mu_1 - \mu_0^\top \Sigma_0^{-1}\mu_0  \\
< & \, 2\, \log (1/\delta) + \log\Big(\frac{|\Sigma_0|}{|\Sigma_1|} \Big).    
\end{align*}
In Example \ref{ex:ldaqdaetchnical} in {Appendix \ref{app:ldaqda}}, we derive analytical expressions for sensitivity and specificity in  generalized $\chi^2$-distributions, which allow to determine optimal values of $\delta$ with sufficient numerical precision without resorting to simulations. 

We consider the following two scenarios. 

{\bfseries Scenario 1:}
\begin{align*}
\mu_0 = \begin{pmatrix} 0 \\0 \end{pmatrix}, \quad
\mu_1 = \begin{pmatrix} 2.5 \\ 2.5 \end{pmatrix}, \quad
\Sigma_0 = \begin{pmatrix} 2 & 0.5 \\ 0.5 & 1 \end{pmatrix}, \quad
\Sigma_1 = \begin{pmatrix} 1 & -0.5 \\ -0.5 & 2 \end{pmatrix}
\end{align*}

{\bfseries Scenario 2:}
\begin{align*}
\mu_0 = \begin{pmatrix} 0 \\0 \end{pmatrix}, \quad
\mu_1 = \begin{pmatrix} 1.5 \\ 1.5 \end{pmatrix}, \quad
\Sigma_0 = \begin{pmatrix} 2 & 0.3 \\ 0.3 & 1 \end{pmatrix}, \quad
\Sigma_1 = \begin{pmatrix} 1 & -0.9 \\ -0.9 & 2 \end{pmatrix}
\end{align*}

Figure \ref{fig:qda1} shows the results, which are quite similar to those of LDA.  For balanced accuracy, the optimal threshold remains $\delta^*=1$, as in the LDA case, independently of $\pi$. In contrast, the G-mean yields thresholds $\delta^*\approx 1.03$ in Scenario~1 and $\delta^*\approx 1.12$ in Scenario~2, again independent of $\pi$, indicating only a mild dependence on the underlying distributions.

\begin{figure}   
\centerline{\includegraphics[scale=0.6]{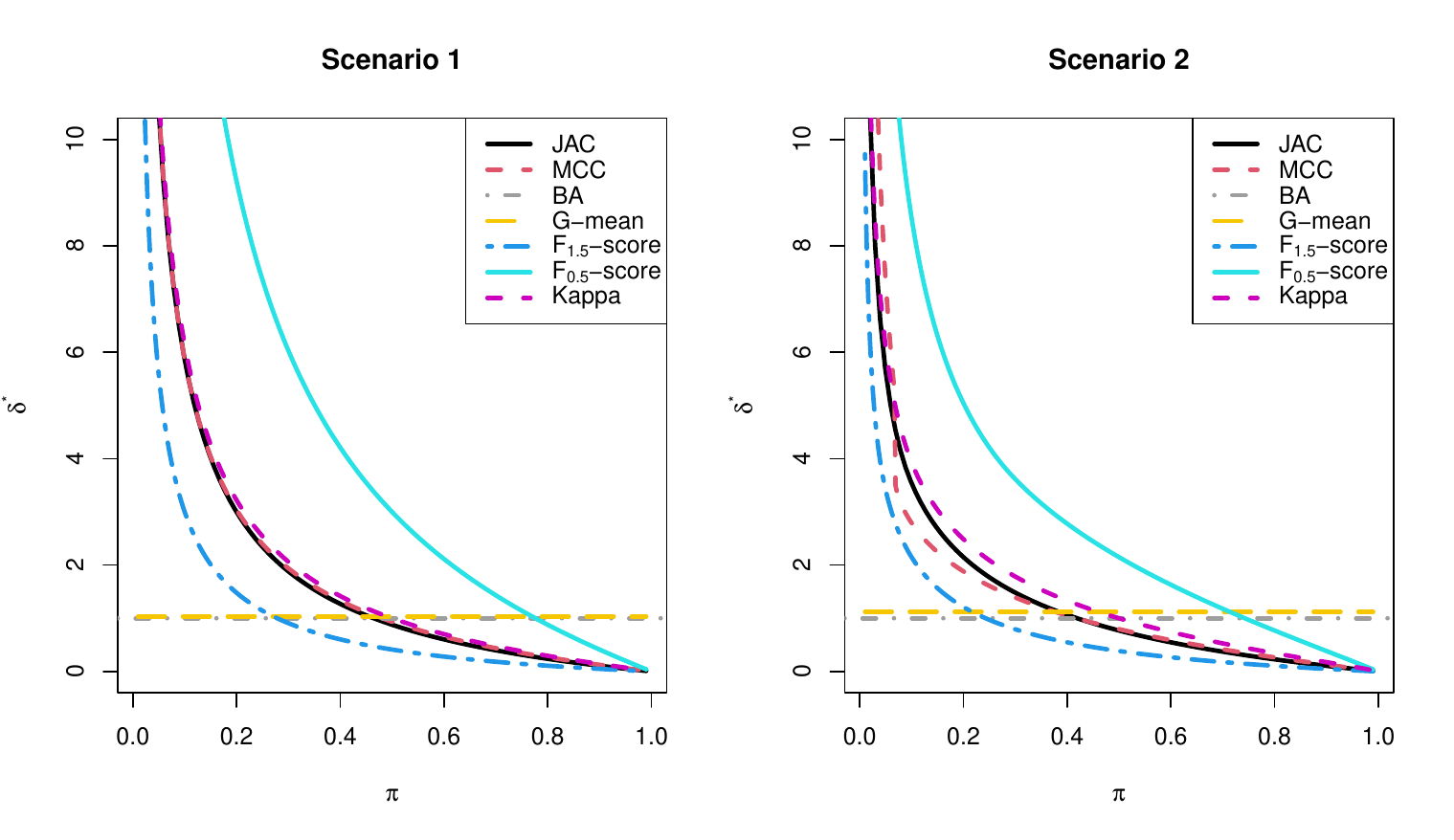}}
\caption{Plots of the optimal threshold $\delta^\ast$ as a function of $\pi$ for the different performance metrics used in Example \ref{ex:qda}.}  \label{fig:qda1}
\end{figure}

\end{example}

\section{Robust performance metrics} \label{robust-metrics}

Based on the result of Theorem \ref{th:decisiobnounddens}, in this section we introduce the notion of robust performance metrics and then construct robustified extensions of the MCC (subsection \ref{ex:mccgen}), of Cohen's $\kappa$ (subsection \ref{ex:kappagen}) as well as of the $F_{\beta}$-score (subsection \ref{ex:fgen}). We illustrate their behaviour on the level of population quantities in the LDA-setting of Example \ref{ex:lda}. This is complemented by a simulation study of the finite-sample behaviour in Section \ref{sec:finsamp}. 

\smallskip

\textbf{Robust metrics.} To introduce the notion of a robust metric we focus on small values of the probability of the positive class $\pi$, the common situation in class imbalanced settings. Let us call $M$ \textit{robust to class imbalance} if its quotient of partial derivatives in \eqref{eq:thequotient} is uniformly bounded by a constant $ \crr > 0$, that is if%
\begin{equation}\label{eq:robust}
\frac{\partial_{\pi_{0|0}} M\big(\pi_{1|1},\pi_{0|0},\pi \big)}{\partial_{\pi_{1|1}} M\big(\pi_{1|1},\pi_{0|0},\pi \big)} \leq \crr \qquad \forall \pi \in (0,1/2] \text{ and } \forall \  \pi_{1|1},\pi_{0|0} \in (0,1).
\end{equation}
From the discussion following Theorem \ref{th:decisiobnounddens}, this implies that under all values of $\pi \in (0,1/2]$  as well as for all class-conditional distributions $ \rP_0$ and $ \rP_1$, the optimal threshold $\delta^*$ is upper-bounded by the constant $\crr $, {and hence the Bayes classifier classifies every case with $f_1/f_0 > \crr$ as positive.}%
 Therefore, the sensitivity of the Bayes-optimal classifier is lower-bounded by $\rP_1(f_1/f_0 > \crr)$ for any $\pi \in (0,1/2]$. Note that this lower bound does depend on the class-conditional distributions $\rP_i$, though, {and it guarantees a positive sensitivity only if $\rP_1(f_1/f_0 > \crr) > 0$, as is the case e.g.~for LDA.}%
{
We observe that the lower bound 
is not affected by a shift in the proportion of positives between the training data and the population in which the classifier is applied, as long as the class-conditional distributions stay the same.}%

\smallskip

\textbf{Examples for \eqref{eq:robust}.} For balanced accuracy $(\pi_{0|0} + \pi_{1|1})/2$ the quotient on the left  in \eqref{eq:robust} and hence the threshold $\delta^*$ is identical to $1$ for all parameter values. Thus balanced accuracy is certainly robust, but the complete independence of its Bayes classifier from all parameters may not be desirable either. Let us point out that  the G-mean score does not satisfy \eqref{eq:robust}. While  it does not depend on $\pi$,  its quotient of partial derivatives still depends on true positive and negative rates, and is not uniformly upper bounded in $\pi_{0|0} \in (0,1)$. 

\smallskip

{\textbf{Fitted scores.} In practice, the density ratio $f_1/f_0$ is typically unknown and not estimated directly. Rather, one works with some score $S = s(X)$ from a fitted model, with higher scores pointing to the minority class, and the classifier is then chosen as $1(S > t)$ for an optimal threshold $t$. The lower bound on the sensitivity for a robust metric carries over to this situation on the scale of the score: Suppose that $S$ has densities $g_0$ and $g_1$ under the two classes. Theorem~\ref{th:decisiobnounddens} holds for an arbitrary feature space and therefore applies to $S$ in place of $X$. Hence, under Assumption~C a Bayes classifier based on $S$ 
is of the form $1(g_1(S)/g_0(S) > \delta^*)$, and for a robust metric $\delta^* \leq \crr$, so that its sensitivity is at least $\rP_1(g_1(S)/g_0(S) > \crr)$. %
Then, if $g_1/g_0$ is strictly increasing, which formalizes that higher scores point to the positive class, the optimal classifier $1(S >t^*)$ in terms of $S$ equals $1(g_1(S)/g_0(S) > \delta^*)$, and hence we retain the lower bound $\rP_1(g_1(S)/g_0(S) > \crr)$ on the sensitivity.%

Moreover, if $S = h\big(f_1(X)/f_0(X)\big)$ is a strictly increasing function $h$ of $f_1(X)/f_0(X)$, such as the regression function $S=\eta(X)$, one can show that $g_1(S)/g_0(S) = f_1(X)/f_0(X)$, and the lower bound $\rP_1(g_1(S)/g_0(S) > \crr)$ equals $\rP_1(f_1(X)/f_0(X) > \crr)$.}%

\subsection{Robust Matthews correlation coefficient (MCC)} \label{ex:mccgen}
The MCC is defined as the Pearson correlation between the binary random variables $Y$ and $\phi(X)$. The quotient of partial derivatives for the MCC, given in {Example \ref{ex-metric-deriv1} in Appendix \ref{app:proofth3}}, does not satisfy \eqref{eq:robust}, and the MCC is not robust with respect to class imbalance. 

We show that a robust version can be obtained if we regularize the variance of $\phi(X)$ so that it is bounded away from $0$. To this end choose the robustness parameter $d>0$ and define %
\begin{align*}
M_{\text{MCC}_{\text{rb}}}&(\pi_{1|1}, \pi_{0|0}, \pi) 
= \left(\frac{d}{\pi(1-\pi)} + 1\right)^{1/2} \cdot \frac{\pi_{11}\, \pi_{00} - \pi_{01}\, \pi_{10}}{\sqrt{\pi\, (1-\pi) \, \big(d + \gamma \, (1-\gamma) \big)}} \\
 & = \big(d + \pi (1-\pi) \big)^{1/2}   \, \frac{\pi_{1|1}\, \pi_{0|0} - (1-\pi_{1|1})\, (1-\pi_{0|0})}{\Big( d + \big(\pi \pi_{1|1} + (1-\pi) (1- \pi_{0|0})\big) \big(\pi (1-\pi_{1|1}) + (1-\pi) \pi_{0|0}\big) \Big)^{1/2}}.
\end{align*}
Note that the robust MCC for $d>0$ is well defined even in the degenerate cases $\gamma \in \{0,1\}$. Furthermore, note that the first factor in the definition of $M_{\text{MCC}_{\text{rb}}}$ leaves the resulting Bayes optimal classifier unchanged. We included it in the definition to retain the normalization at one for perfect classification and to have $M_{\text{MCC}_{\text{rb}}} \in [-1,1]$.
\begin{theorem}
The performance measure $\text{MCC}_{\text{rb}}$ is a robust performance metric, satisfying condition \eqref{eq:robust} with $\crr = (1+2d)/(2d)$. Furthermore, $M_{\text{MCC}_{\text{rb}}}  \in [-1,1]$ for any $d>0$, and $M_{\text{MCC}_{\text{rb}}}(1, 1, \pi) =1$.
\end{theorem}

\begin{proof}
We have that 
$$
\frac{\partial_{\pi_{0|0}} M_{\text{MCC}_{\text{rb}}}(\pi_{1|1}, \pi_{0|0}, \pi)}{\partial_{\pi_{1|1}} M_{\text{MCC}_{\text{rb}}}(\pi_{1|1}, \pi_{0|0}, \pi)}
= \frac{(2\pi_{1|1}-1)\pi_{0|0} - (\pi_{0|0} + \pi_{1|1} - 1)(2\pi_{1|1} - 1)\pi -  \pi_{1|1} + 1 + 2d}{(2\pi_{0|0}-1)(\pi_{0|0} + \pi_{1|1} - 1)\pi + 2\pi_{0|0}\, (1-\pi_{0|0}) + 2d}~.
$$
If we fix $\pi_{0|0}$ and $\pi_{1|1}$, the numerator is monotone in $\pi$ 
and upper bounded by
\begin{align*}
\max\{ 2\pi_{1|1} (1-\pi_{1|1})+2d, \pi_{1|1}\pi_{0|0} + (1-\pi_{1|1})(1-\pi_{0|0})+2d \} 
\leq 1+2d~.
\end{align*}
Likewise, the denominator is lower-bounded by
\begin{align*}
\min\{2\pi_{0|0} (1-\pi_{0|0})+2d, \pi_{1|1}\pi_{0|0} + (1-\pi_{1|1})(1-\pi_{0|0})+2d \} \geq 2d~.
\end{align*}
Therefore,  $M_{\text{MCC}_{\text{rb}}}$ satisfies \eqref{eq:robust} %
with $\crr = (1+2d)/(2d)$. 

The numerator is also lower bounded by $2d$, and since numerator and denominator equal the partial derivatives up to positive factors, Assumption C also holds for the robust MCC.

Concerning the normalization, set $u=\text{Var}(Y) = \pi\, (1-\pi)$, $v = \text{Var}(\phi(X)) = \gamma\, (1-\gamma)$, where $\phi$ is the classifier, and $w=\text{Cov}(Y,\phi(X)) = \pi_{11}\, \pi_{00} - \pi_{01}\, \pi_{10}$. From the definition,
\begin{equation}\label{eq:expressMCC}
M_{\text{MCC}_{\text{rb}}} = \Big(\frac du + 1 \Big)^{1/2}\,\frac{w}{\sqrt{u\, (d+v)}} = \frac{w\,\sqrt{ (d+u)}}{u\,\sqrt{ (d+v)}}.    
\end{equation}
Now for the two binary random variables $Y$ and $\phi(X)$, 
$$ w^2 \leq \min(u^2,v^2) \leq uv,$$
see Lemma \ref{lem:covinequbinary} in {Appendix \ref{app:robust}} for a proof of the first inequality. Since $d\geq 0$ this implies
$$ d(u^2-w^2) + u (uv-w^2) \geq 0,$$
which using \eqref{eq:expressMCC} is seen to be equivalent to $M_{\text{MCC}_{\text{rb}}}  \in [-1,1]$. Finally,  $M_{\text{MCC}_{\text{rb}}}(1, 1, \pi) =1$ is immediate from the second formula in the definition of $M_{\text{MCC}_{\text{rb}}}$.
\end{proof}
\begin{remark}\label{rem:onrobMCC}
The continuous extension at $d=0$ coincides with the original MCC, whereas
\begin{equation}\label{eq:limitmcccohen}
\lim_{d\to\infty}
M_{\text{MCC}_{\text{rb}}}
(\pi_{1|1},\pi_{0|0},\pi)
= \pi_{1|1}\, \pi_{0|0} - (1-\pi_{1|1})\, (1-\pi_{0|0}) = 
\pi_{1|1}+\pi_{0|0}-1.
\end{equation}
The limiting metric is Youden's $J$, equivalently twice the balanced accuracy minus one. Thus, the family continuously interpolates between the MCC and a performance metric that is independent of the class proportion $\pi$. At the same time, the robustness constant $\crr=1+1/(2d)$ decreases as $d$ increases and converges to $1$ as $d \to \infty$. Hence, increasing $d$ strengthens robustness and reduces the dependence of the induced Bayes-optimal threshold on the class proportion. The zero level for $\pi_{1|1}+\pi_{0|0}=1$ and the value $1$ for perfect classification are preserved for all $d$.
\end{remark}

Figure \ref{fig:lda-mcc} shows $\delta^*$ in the setting of Example \ref{ex:lda} for various choices of the parameter $d$. If $d$ increases, the curvature of the solution function decreases. Hence, the value of $\delta^*$ decreases for $\pi\to 0$ and increases for $\pi\to 1$. For each choice of $d$, the cutoff $\delta^*$ equals 1 in the balanced case $\pi=0.5$.

Figure \ref{fig:lda-mccb} shows the corresponding true positive rates $\pi_{1|1}$ as a function of $\pi$. Again, the boundedness of $\delta^\ast$ implies that $\pi_{1|1}$ remains bounded away from $0$.

Numerical results for $\delta^*$ and corresponding results for sensitivity and specificity can be found in Table \ref{tab1:ex:mccgen} in Appendix \ref{tab:ex:mccgen}. 
The behaviour of $\delta^*$ for $\text{MCC}_{\text{rb}}$ in the setting of Example \ref{ex:qda} is shown in Figure \ref{fig:qda-frb} in Appendix \ref{app-qda}.

\begin{figure}   
\centerline{\includegraphics[scale=0.6]{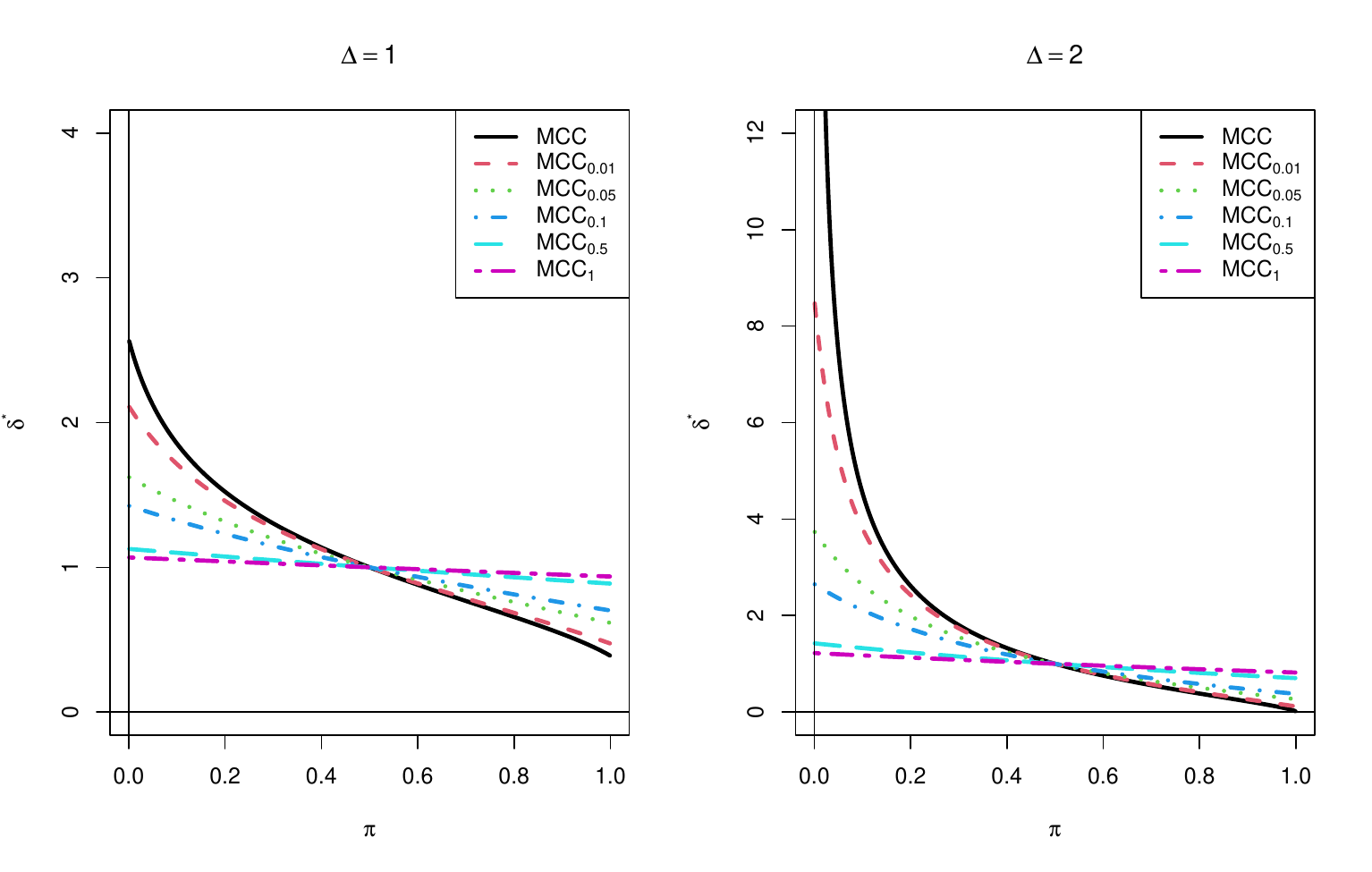}}
\caption{Plots of the optimal threshold $\delta^\ast$ as a function of $\pi$ for $\text{MCC}_{\text{rb}}$ proposed in subsection \ref{ex:mccgen}, using different choices of the parameter $d$.} \label{fig:lda-mcc}
\end{figure}

\begin{figure}   
\centerline{\includegraphics[scale=0.52]{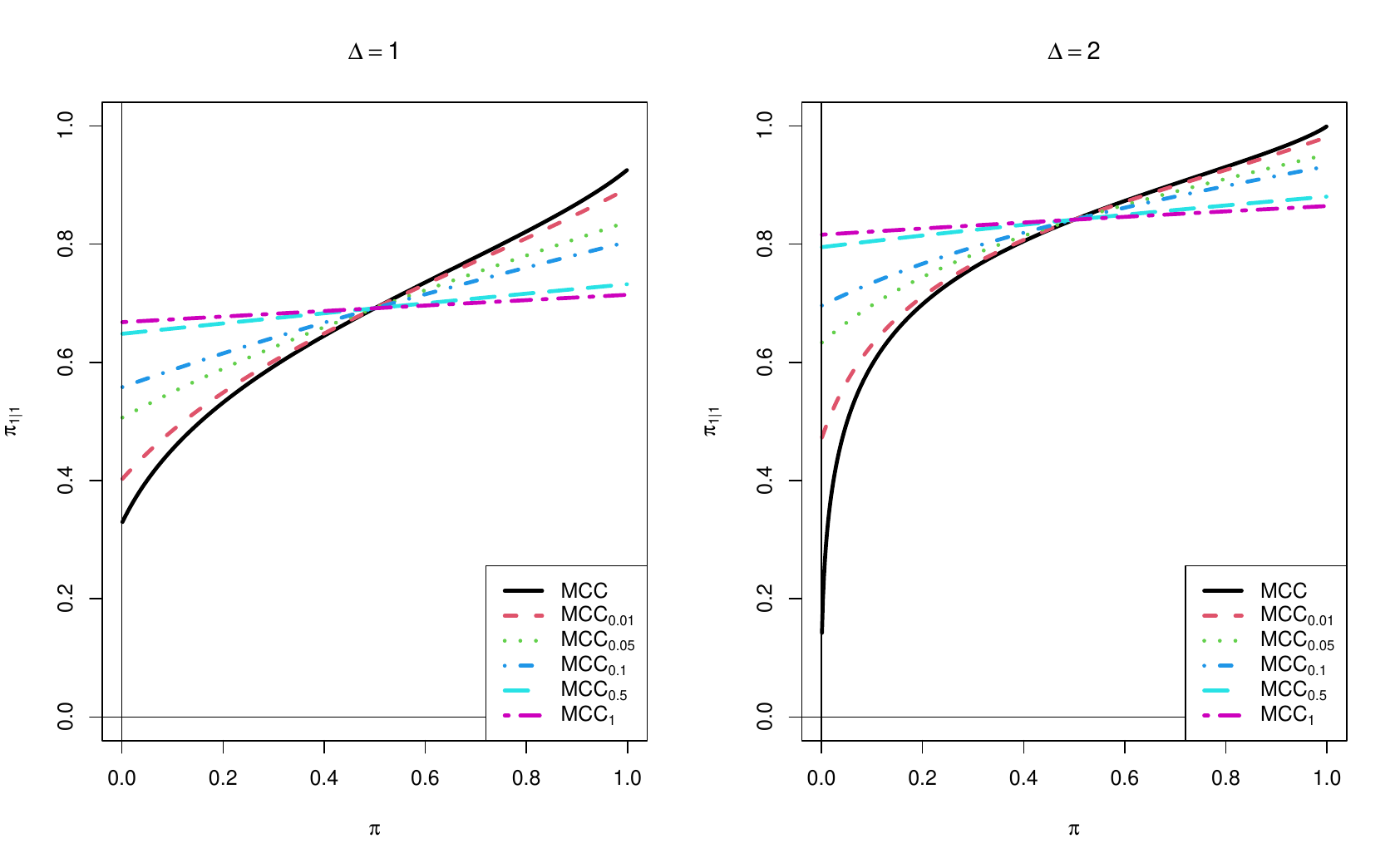}}
\caption{Plots of $\pi_{1|1}$ associated to the optimal threshold $\delta^\ast$ as a function of $\pi$ for $\text{MCC}_{\text{rb}}$ proposed in subsection \ref{ex:mccgen}, using different choices of the parameter $d$.} \label{fig:lda-mccb}
\end{figure}

\subsection{Robust Cohen's kappa} \label{ex:kappagen}
Cohen's $\kappa$ measures the agreement between the binary random
variables $Y$ and $\phi(X)$ after correcting for the agreement expected
under independence. The quotient of partial derivatives for Cohen's
$\kappa$, given in {Example \ref{ex-metric-deriv1} in Appendix \ref{app:proofth3}}, does not satisfy
\eqref{eq:robust}, and Cohen's $\kappa$ is not robust with respect to
class imbalance. As with the MCC, we can obtain a robust version by regularizing the normalizing term.
Given the robustness parameter $d>0$ define
\begin{align*}
M_{\kappa_{\mathrm{rb}}} & (\pi_{1|1},\pi_{0|0},\pi) 
= \left(\frac{d}{2\pi(1-\pi)}+1\right) 
\frac{2(\pi_{11}\pi_{00}-\pi_{01}\pi_{10})}{d+\pi(1-\gamma)+(1-\pi)\gamma} \\
&= \frac{(d+2\pi(1-\pi)) \, (\pi_{1|1}\, \pi_{0|0} - (1-\pi_{1|1})\, (1-\pi_{0|0}))}{d+2\pi(1-\pi)(\pi_{1|1}+\pi_{0|0}-1) -\pi\pi_{1|1}-(1-\pi)\pi_{0|0}+1}.
\end{align*}
The factor in the definition is again included to retain the usual normalization. Setting $D = \pi(1-\gamma)+(1-\pi)\gamma$ it allows to write
\begin{equation}\label{eq:convcomb}
M_{\kappa_{\mathrm{rb}}} = \frac{\big(d + 2\, \pi \, (1-\pi)\big)\, J}{d + D}
     = \frac{d}{d + D} \, J + \frac{D}{d + D} \, M_{\kappa}    
\end{equation}
as a convex combination of Youden's $J = \pi_{1|1}+\pi_{0|0}-1$ and the classical Cohen's $\kappa$, $M_{\kappa} = 2\, \pi \, (1-\pi)\, J/ D$. 
\begin{theorem}
The performance measure $M_{\kappa_{\mathrm{rb}}}$ is a robust performance
metric, satisfying condition \eqref{eq:robust} with $\crr=(1+d)/d.$ Furthermore, $M_{\kappa_{\mathrm{rb}}}  \in [-1,1]$ for any $d>0$, and $M_{\kappa_{\mathrm{rb}}}(1, 1, \pi) =1$. 
\end{theorem}

\begin{proof}
We have that
\begin{align*}
\frac{\partial_{\pi_{0|0}} M_{\kappa_{\mathrm{rb}}}(\pi_{1|1},\pi_{0|0},\pi)}
    {\partial_{\pi_{1|1}} M_{\kappa_{\mathrm{rb}}} (\pi_{1|1},\pi_{0|0},\pi)}
&= \frac{d+(1-\pi)\pi_{1|1}+\pi(1-\pi_{1|1})}{d+(1-\pi)(1-\pi_{0|0})+\pi\pi_{0|0}}.
\end{align*}
Since
\[
0< (1-\pi)\pi_{1|1}+\pi(1-\pi_{1|1}) < 1, \qquad
(1-\pi)(1-\pi_{0|0})+\pi\pi_{0|0} > 0,
\]
we obtain
\[
\frac{\partial_{\pi_{0|0}}M_{\kappa_{\mathrm{rb}}}}{\partial_{\pi_{1|1}}M_{\kappa_{\mathrm{rb}}}}
< \frac{1+d}{d}.
\]
Hence, $M_{\kappa_{\mathrm{rb}}}$ satisfies \eqref{eq:robust} with $\crr=(1+d)/d$. Also, the partial derivatives are strictly positive, so that Assumption C is satisfied. 

The range $M_{\kappa_{\mathrm{rb}}}  \in [-1,1]$ follows since, by \eqref{eq:convcomb}, $M_{\kappa_{\mathrm{rb}}}$ is a convex combination of two quantities in $[-1,1]$. $M_{\kappa_{\mathrm{rb}}}(1, 1, \pi) =1$ is immediate. 
\end{proof}

The interpretation in Remark \ref{rem:onrobMCC} applies analogously to the robust version of Cohen's $\kappa$, with the standard Cohen's $\kappa$ arising for $d=0$ and the asymptotic \eqref{eq:limitmcccohen} being also true for $M_{\kappa_{\mathrm{rb}}}$. 
Figure \ref{fig:lda-kappa} shows $\delta^*$ in the setting of Example \ref{ex:lda} for various choices of the parameter $d$. The curve corresponding to $d=0$ represents the original Cohen's $\kappa$. As $d$ increases, the curves approach the constant threshold $\delta^*=1$, in accordance with the limiting relation to Youden's
$J$. In particular, the values of $\delta^*$ for small $\pi$ decrease, and for every $d>0$ the optimal threshold is bounded above by
$(1+d)/d$. For each choice of $d$, the cutoff $\delta^*$ equals $1$ in
the balanced case $\pi=0.5$.

\begin{figure}   
\centerline{\includegraphics[scale=0.6]{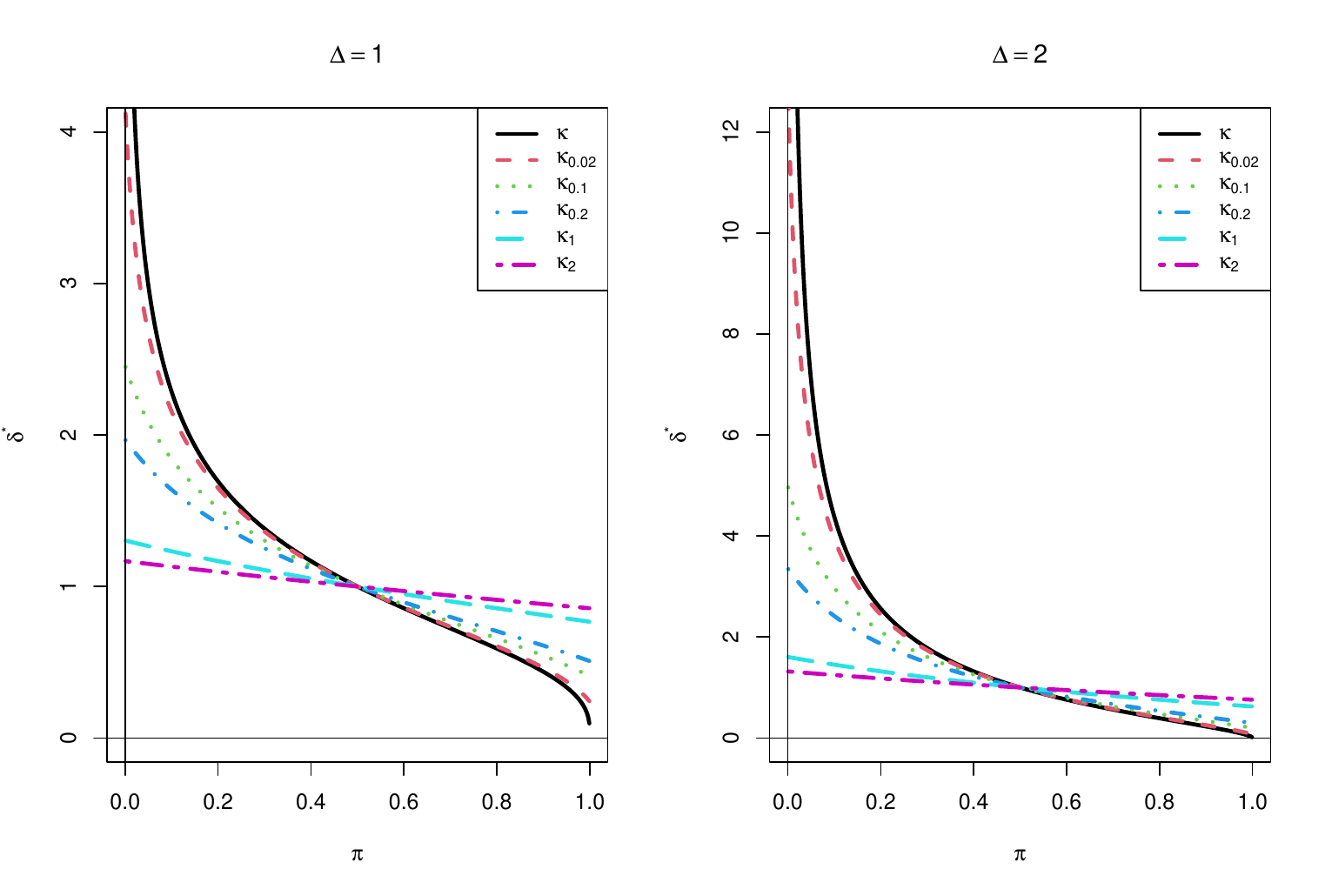}}
\caption{Plots of the optimal threshold $\delta^\ast$ as a function of $\pi$ for $\kappa_{\text{rb}}$ proposed in subsection \ref{ex:kappagen}, using different choices of the parameter $d$.} \label{fig:lda-kappa}
\end{figure}

\subsection{Robust \texorpdfstring{$\boldsymbol{F_\beta}$-score}{Fbeta-score}} \label{ex:fgen}

The $F_\beta$-score can be written as the harmonic mean of recall and precision, 
\begin{align*}
M_{\text{F}_{\beta}}(\pi_{1|1}, \pi_{0|0}, \pi)  = \frac{(1+\beta^2)}{ \beta^2\, \frac{\pi}{\pi_{11}} +  \frac{\gamma}{\pi_{11}} }. 
\end{align*}
The quotient of partial derivatives for the $F_\beta$-score, given in {Example \ref{ex-metric-deriv1} in Appendix \ref{app:proofth3}}, does not satisfy \eqref{eq:robust}, and indeed we illustrate in Example \ref{ex:lda} its sensitivity to class-imbalance even in the LDA setting. 

It turns out that if the true positive probability is included in the weighted harmonic mean, resulting in  
\begin{align*}
M_{\text{F}_{\beta,d_0}}(\pi_{1|1}, \pi_{0|0}, \pi)  
= \frac{(d_0/\pi+\beta^2+1)}{ d_0 \, \frac{1}{\pi_{11}} + \beta^2\, \frac{\pi}{\pi_{11}} +  \frac{\gamma}{\pi_{11}} } 
= \frac{ (d_0/\pi+\beta^2+1)\, \pi \, \pi_{1|1}}{d_0+\beta^2\, \pi + \pi \, \pi_{1|1} + (1-\pi_{0|0})\, (1-\pi)}, 
\end{align*}
for a $d_0>0$, this metric is robust in the sense of \eqref{eq:robust} with $\crr = d_0^{-1}$. Note that the numerator is a normalizing factor which scales the robust F-score to the interval $[0,1]$, with the boundary value $1$ obtained for $\pi = \pi_{11} = \gamma$.  
More generally, we define 
\begin{align*}
M_{\text{F}_{\text{rb}}}(\pi_{1|1}, \pi_{0|0}, \pi) 
&= \frac{(d_0/\pi+d_1+1)}{1+c} \cdot \frac{1+c\, \frac{\pi}{\pi_{11}}}{ d_0 \, \frac{1}{\pi_{11}} + d_1\, \frac{\pi}{\pi_{11}} +  \frac{\gamma}{\pi_{11}} } \\
&= \frac{(d_0/\pi+d_1+1)}{1+c} \cdot \frac{c \, \pi + \pi \, \pi_{1|1}}{d_0+d_1\, \pi + \pi \, \pi_{1|1} + (1-\pi_{0|0})\, (1-\pi)}~, 
\end{align*}
where we assume $c,d_1\geq 0$, $d_0>0$ %
and $d_0+d_1-c>0$, in which case the metric satisfies Assumption~C. Note that these metrics fall into the class studied in \citet{Koyejo:2014}. 
Obviously, for $c=d_0 = 0$ and $d_1 = \beta^2$ we recover the $F_\beta$ score.
\begin{theorem}\label{th:robustFscore}
The performance measure $\mathrm{F}_{\mathrm{rb}}$, and hence in particular $\mathrm{F}_{\beta,d_0}$, both are robust performance metrics, and satisfy condition \eqref{eq:robust} with $\crr = (1+c)/\min(d_0,d_0+d_1-c)$ (respectively $\crr = d_0^{-1}$ for $\mathrm{F}_{\beta,d_0}$).
\end{theorem}

\begin{proof}
We have
$$
\frac{\partial_{\pi_{0|0}} M_{\text{F}_{\text{rb}}}(\pi_{1|1}, \pi_{0|0}, \pi)}{\partial_{\pi_{1|1}} M_{\text{F}_{\text{rb}}}(\pi_{1|1}, \pi_{0|0}, \pi)} = \frac{(c+\pi_{1|1})\, (1-\pi)}{d_0 + (d_1-c)\, \pi + (1-\pi_{0|0})\, (1-\pi)}.
$$
Hence \eqref{eq:robust} holds with $\crr = (1+c)/\min(d_0,d_0+d_1-c)$.
\end{proof}

Figure \ref{fig:lda2} shows $\delta^*$ in the setting of Example \ref{ex:lda} for various choices of the parameters. For the black curves, we used $c=0, d_1=1$ and vary $d_0$, which controls the height of $\delta^\ast$ in the limit $\pi\to 0$. For the red curves, we put $c=0, d_0=0.1$ and vary $d_1$. Now, $\lim_{\pi\to 0} \delta^\ast$ is fixed (for given $\Delta$), but the curvature varies. Finally, we put $c=1$ and vary $d_0$ and $d_1$, which results in different shapes of the curves, but qualitatively similar outcomes as before.

Figure \ref{fig:lda2b} additionally shows the corresponding true positive rates $\pi_{1|1}$ as a function of $\pi$. In accordance with the boundedness of $\delta^\ast$, the true positive rate remains bounded away from $0$ for all parameter choices.

Numerical results for $\delta^*$ and corresponding results for sensitivity and specificity can be found in Tables \ref{tab1:ex:fgen}-\ref{tab3:ex:fgen} in Appendix \ref{tab:ex:fgen}. 
The behaviour of $\delta^*$ for $\text{F}_{\text{rb}}$ in the setting of Example \ref{ex:qda} is shown in Figure \ref{fig:qda-frb} in Appendix \ref{app-qda}.

\begin{figure}   
\centerline{\includegraphics[scale=0.6]{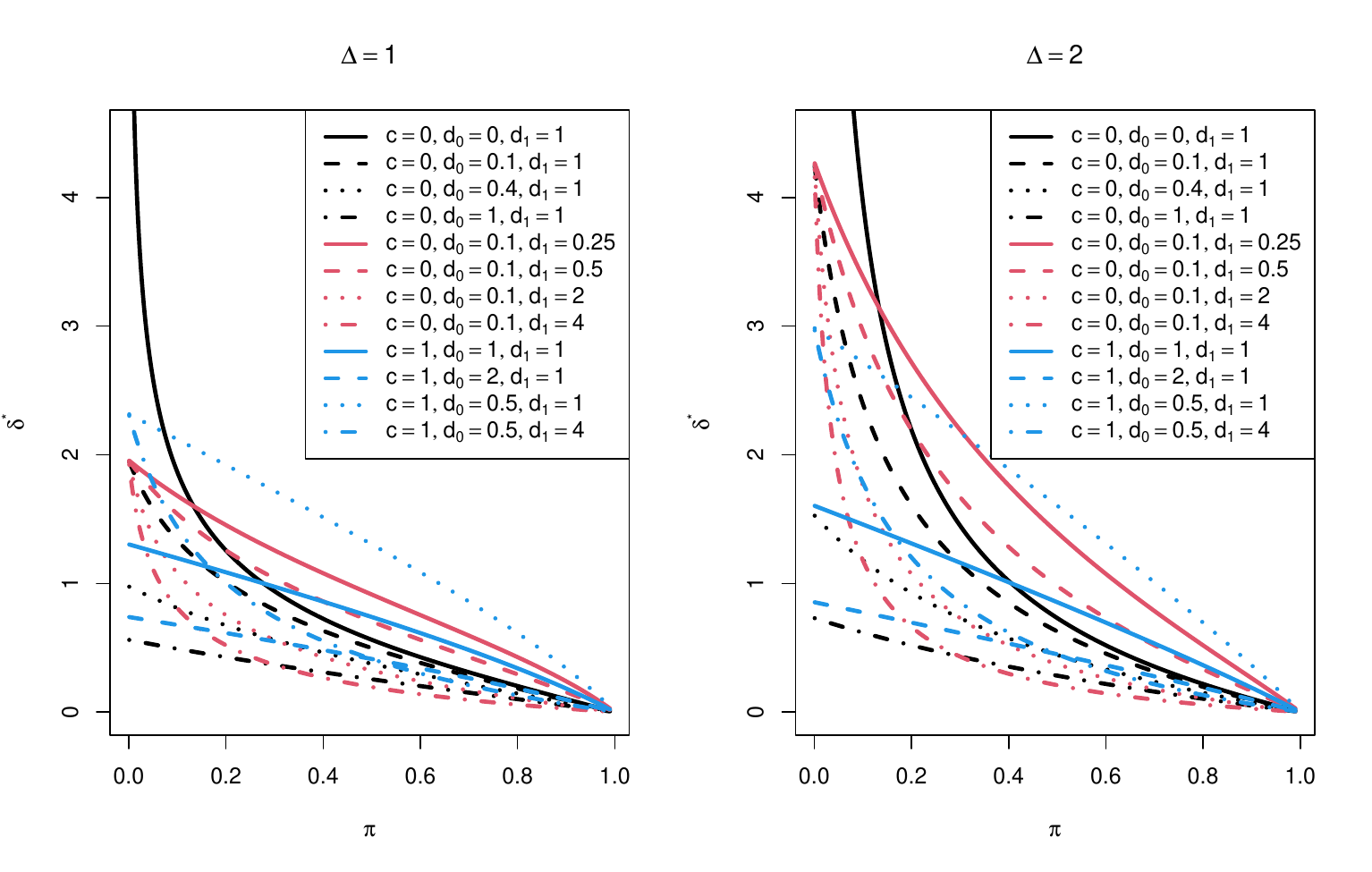}}
\caption{Plots of the optimal threshold $\delta^\ast$ as a function of $\pi$ for $\text{F}_{\text{rb}}$-score proposed in subsection \ref{ex:fgen} with different choices of the parameters.}  \label{fig:lda2}
\end{figure}

\begin{figure}   
\centerline{\includegraphics[scale=0.52]{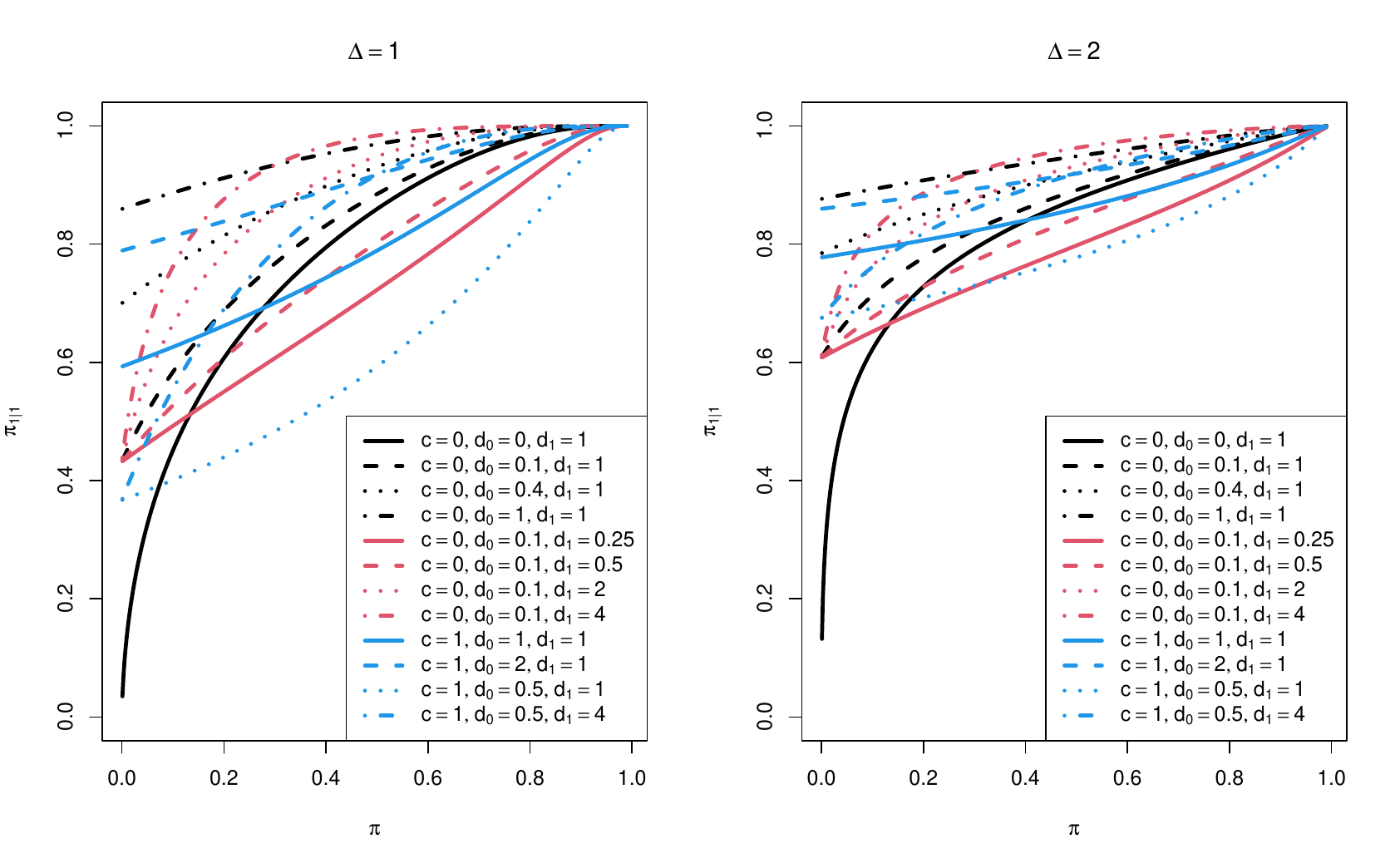}}
\caption{Plots of $\pi_{1|1}$ for the optimal threshold $\delta^\ast$ as a function of $\pi$ for $\text{F}_{\text{rb}}$-score proposed in subsection \ref{ex:fgen} with different choices of the parameters.}  \label{fig:lda2b}
\end{figure}

\smallskip

{\textbf{Existence of Bayes classifiers.}
In general, existence results for Bayes classifiers require randomized classifiers $\phi: \cX \to [0,1]$, as for randomized tests in testing theory. For any class-conditional distributions, the set $\{ (\pi_{1|1}(\phi),\pi_{0|0}(\phi)) \mid \phi: \cX \to [0,1] \text{ measurable}\}$ is compact, by the weak compactness of the set of such functions \citep[Theorem~A.5.1]{Lehmann:2005}. Hence a randomized Bayes classifier $\phi^*$ exists for every performance metric that, for fixed $\pi$, is a continuous function of $(\pi_{1|1},\pi_{0|0})$ on all of $[0,1]^2$. This applies to the robust metrics introduced above, whose denominators are bounded away from zero. It does not apply to Yule's coefficients, which cannot be extended continuously to the two trivial classifiers, and for which no Bayes classifier exists under LDA (Example~\ref{ex-yule-deriv} in Appendix~\ref{app:yule}). Moreover, if $\rP_0 \neq \rP_1$ are mutually absolutely continuous, direct arguments show that every Bayes classifier $\phi^*$ satisfies $\pi_{1|1}(\phi^*),\pi_{0|0}(\phi^*) \in (0,1)$ for the robust MCC and the robust Cohen's $\kappa$, so that Theorem~\ref{th:decisiobnounddens} applies. For the robust F-score this holds under an additional condition on the range of $f_1/f_0$. %
The precise statements and their proofs are given in Lemmas~\ref{lem:exist-mcc-kappa} and~\ref{lem:exist-frb} in Appendix~\ref{app:robust}.} %

\smallskip

{\textbf{Calibration through $\crr$.}The bound $\crr$ provides a common scale for the tuning parameters of the three robust metrics: parameters with the same value of $\crr$ give the same lower bound on the sensitivity of  their respective Bayes classifiers, although the optimal thresholds  themselves can differ in general. In the following we give the form of $\crr$ for the three robust metrics, together with the tuning parameters which lead to $\crr = 6$ that we use in Section~\ref{sec:nonGaussian}.
\begin{center}
\begin{tabular}{lcc}
\hline
Robust metric & Bound $\crr$ & Parameters with $\crr = 6$ \\
\hline
$\text{MCC}_{\text{rb}}$ (Section~\ref{ex:mccgen}) & $(1+2d)/(2d)$ & $d = 0.1$ \\
$\kappa_{\mathrm{rb}}$ (Section~\ref{ex:kappagen}) & $(1+d)/d$ & $d = 0.2$ \\
$\mathrm{F}_{\mathrm{rb}}$ (Section~\ref{ex:fgen}) & $(1+c)/\min(d_0,\, d_0+d_1-c)$ & $c = 0$, $d_0 = 1/6$, $d_1 = 1$ \\
\hline
\end{tabular}
\end{center}
}

\subsection{Non-Gaussian population experiments} \label{sec:nonGaussian}
The robustness condition \eqref{eq:robust} and the bounds established for the three robustified metrics are distribution-free: they hold for arbitrary class-conditional distributions $\rP_0$ and $\rP_1$ and do not rely on the normality assumptions used in Examples~\ref{ex:lda} and~\ref{ex:qda}. To illustrate this point numerically, we consider two non-Gaussian location-shift models. In both cases, let
\[
\rP_0=\mathcal{L}(Z-\Delta/2), \quad \rP_1=\mathcal{L}(Z+\Delta/2), \quad \Delta=2,
\]
where $Z$ has mean zero and variance one. In the first scenario, $Z=T/\sqrt{5/3}$ with $T\sim t_5$, giving a symmetric heavy-tailed model. In the second scenario, $Z$ is the standardized version of
$W\sim \operatorname{ST}(0,1,\alpha=2,\nu=20)$, using the skew-$t$ distribution implemented in the R package \texttt{sn}. The value $\nu=20$ was chosen so that the departure from normality in this scenario is mainly due to skewness rather than heavy tails.

For each of 500 equidistant values of $\pi$ between $0.001$ and $0.999$, we numerically maximize the population $F_1$-score, MCC, and Cohen's $\kappa$, as well as their robustified versions, over the density-ratio threshold $\delta$. The class-conditional probabilities $\pi_{1|1}(\delta)$ and $\pi_{0|0}(\delta)$ are evaluated by importance sampling from the balanced mixture $(\rP_0+\rP_1)/2$, using $N_{\mathrm{MC}}=2\,000\,000$ observations, half from each component. The optimization is carried out directly over the empirical
likelihood-ratio cut points and does not require the likelihood ratio to be monotone in $x$.

For comparability, the robustness parameters are calibrated to the common theoretical bound $\crr=6$. Specifically, we use $d=0.1$ for the robust MCC, $d=0.2$ for the robust Cohen's $\kappa$, and $c=0, d_0=1/6, d_1=1$ for the robust $F_1$-score. Figure \ref{fig:nonGaussian-threshold} shows that, in both non-Gaussian scenarios, the thresholds induced by the classical metrics increase sharply as $\pi$ approaches
zero. In contrast, the thresholds of all three robustified metrics remain substantially smaller and, in the relevant range $\pi\leq 1/2$, below their common distribution-free bound $\crr=6$. Thus, although skewness and heavy tails affect the precise threshold curves, the qualitative distinction between classical and
robustified metrics is the same as in the Gaussian examples.

\begin{figure}[t]
\centering
\includegraphics[width=\textwidth]{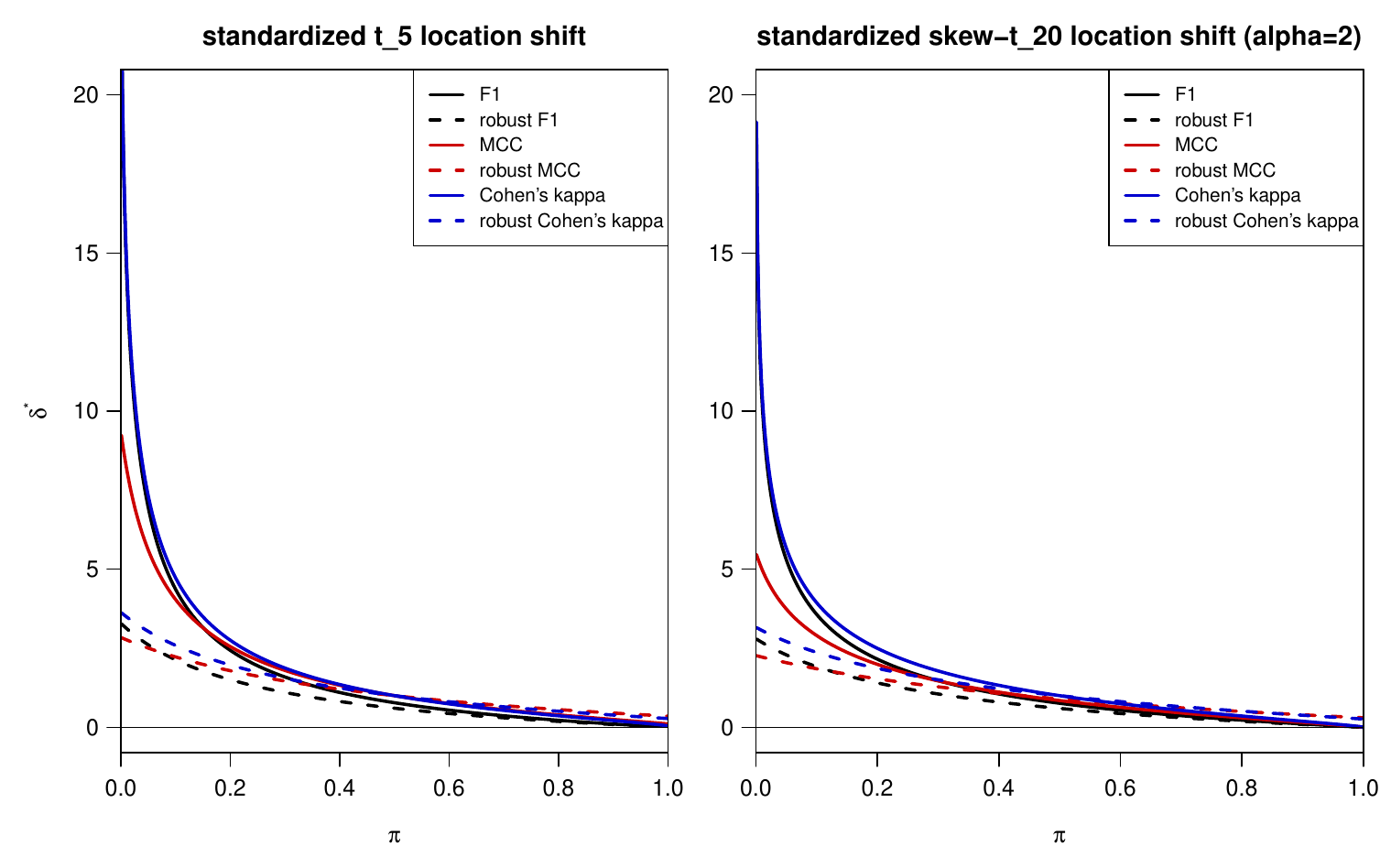}
\caption{
Plots of the optimal threshold  $\delta^*$  as a function of $\pi$ for the standardized $t_5$ location-shift model (left) and the standardized skew-$t_{20}$ location-shift model with shape parameter $\alpha=2$ (right). Solid lines represent the classical metrics and dashed lines their robustified versions. The robustness parameters are calibrated to the common bound $\crr=6$.
} \label{fig:nonGaussian-threshold}
\end{figure}

\subsection{Finite-sample behavior and effect of sample size}\label{sec:finsamp}

To complement the population-based analysis, we briefly investigate the finite-sample behavior of the considered performance measures and the associated classification thresholds. In particular, we examine how the empirical maximizers and the resulting classifiers stabilize as the sample size increases. \\
The simulation setup follows Example~\ref{ex-scen2}, with a fixed class proportion $\hat{\pi}=0.025$. We consider the classical MCC and its robust counterpart $\mathrm{MCC}_{\mathrm{rb}}^{0.1}$. We simulate independent samples for sample sizes 
$n \in \{500, 2000, 8000, 32000, 128000\}${, that is, with about $12$, $50$, $200$, $800$ and $3200$ positive cases,} and determine, for each realization, the empirically optimal threshold $\tilde{\delta}$ based on the regression function, as well as the corresponding density-ratio threshold $\delta$. For each configuration, we report the mean and standard deviation of the maximized metric value, $\tilde{\delta}$, and $\delta$ over 1000 repetitions; see Table~\ref{tab:sec5-3}. \\
The results show that the variability of the maximized metric values decreases at a rate consistent with $n^{-1/2}$, as expected. In contrast, the behavior of the corresponding thresholds is less regular. This reflects the fact that the threshold is defined as the maximizer of an empirical criterion, and is therefore sensitive to local flatness of the objective function and to discretization effects. Moreover, the nonlinear transformation from the regression-function threshold $\tilde{\delta}$ to the density-ratio threshold $\delta$ amplifies variability, particularly in the imbalanced setting considered here. \\
Notably, the thresholds induced by the robust metric exhibit substantially lower variability and stabilize more quickly across sample sizes. {In relative terms, measured by the coefficient of variation of $\delta$ (in parentheses in Table~\ref{tab:sec5-3}), $\mathrm{MCC}_{\mathrm{rb}}^{0.1}$ attains from $n = 2000$ on about the stability that the MCC attains with four times the sample size, for example $0.46$ at $n = 2000$ against $0.46$ for the MCC at $n = 8000$. With about $12$ and $50$ positive cases, at $n = 500$ and $n = 2000$, the thresholds of both metrics still vary considerably, with coefficients of variation of $0.77$ and $0.65$ for the MCC and $0.78$ and $0.46$ for $\mathrm{MCC}_{\mathrm{rb}}^{0.1}$.} Thus, while the qualitative differences between classical and robust metrics persist in finite samples, the robust formulation also leads to more stable threshold selection.

\begin{table}
\centering
\begin{tabular}{llcc} \hline
$n$ &  & MCC & $\mathrm{MCC}_{\mathrm{rb}}^{0.1}$ \\ \hline
\multirow{3}{*}{500} & value & $0.525 \pm 0.108$ & $0.610 \pm 0.094$ \\
 & $\tilde{\delta}$ & $0.271 \pm 0.139$ {(0.51)} & $0.082 \pm 0.054$ {(0.66)} \\
 & $\delta$ & $17.566 \pm 13.453$ {(0.77)} & $3.822 \pm 2.990$ {(0.78)} \\
\hline
\multirow{3}{*}{2000} & value & $0.464 \pm 0.055$ & $0.562 \pm 0.050$ \\
 & $\tilde{\delta}$ & $0.272 \pm 0.115$ {(0.42)} & $0.069 \pm 0.028$ {(0.41)} \\
 & $\delta$ & $16.135 \pm 10.422$ {(0.65)} & $2.920 \pm 1.330$ {(0.46)} \\
\hline
\multirow{3}{*}{8000} & value & $0.433 \pm 0.028$ & $0.539 \pm 0.024$ \\
 & $\tilde{\delta}$ & $0.250 \pm 0.080$ {(0.32)} & $0.064 \pm 0.017$ {(0.27)} \\
 & $\delta$ & $13.698 \pm 6.338$ {(0.46)} & $2.660 \pm 0.780$ {(0.29)} \\
\hline
\multirow{3}{*}{32000} & value & $0.422 \pm 0.015$ & $0.531 \pm 0.013$ \\
 & $\tilde{\delta}$ & $0.242 \pm 0.050$ {(0.21)} & $0.063 \pm 0.010$ {(0.16)} \\
 & $\delta$ & $12.676 \pm 3.550$ {(0.28)} & $2.608 \pm 0.466$ {(0.18)} \\
\hline
\multirow{3}{*}{128000} & value & $0.418 \pm 0.007$ & $0.528 \pm 0.006$ \\
 & $\tilde{\delta}$ & $0.240 \pm 0.032$ {(0.13)} & $0.061 \pm 0.007$ {(0.11)} \\
 & $\delta$ & $12.374 \pm 2.159$ {(0.17)} & $2.531 \pm 0.294$ {(0.12)} \\ \hline
\end{tabular}
\caption{Finite-sample behavior of the classical MCC and the robust metric $\mathrm{MCC}_{\mathrm{rb}}^{0.1}$ for varying sample sizes $n$ at class proportion $\hat{\pi}=0.025$. Reported are the mean and standard deviation (over 1000 repetitions) of the maximized metric value, the corresponding regression-function threshold $\tilde{\delta}$, and the density-ratio threshold $\delta$. {In parentheses: coefficient of variation 
of the thresholds.}}  \label{tab:sec5-3}
\end{table}

\section{Modified precision-recall curves and connections to ROC curves and performance metrics}\label{sec:rocprecisionrecall}

The receiver operating characteristic (ROC) curve as well as the precision-recall curve are important tools to assess the overall performance of a classification method, independently of the particular choice of the threshold value $\delta$. See \citet{cali2015some, gneiting2022receiver, davis2006relationship} for reviews,  mathematical properties and connections of these curves. 

The ROC curve is the plot of the true positive rate (TPR) against the false positive rate (FPR). %
In other words, it plots $(1-\pi_{0|0}(\delta), \pi_{1|1}(\delta))$, $\delta \in (0,\infty)$, or equivalently $(1-\pi_{0|0}(\tilde \delta), \pi_{1|1}(\tilde \delta))$ when parameterized by the regression function $\eta$ with threshold $\tilde \delta$, $\tilde \delta \in (0,1)$. Thus, the ROC curve is independent of the weight $\pi$ of the positive class. 

The precision-recall curve plots the  precision $\preci(\delta) = \pi_{1|1}(\delta) \pi /(\pi_{1|1}(\delta) \pi + (1-\pi_{0|0}(\delta))(1-\pi))$, which is the proportion of true positives among positives, against recall (TPR), resulting in the curve $(\pi_{1|1}(\delta), \preci(\delta))$, $\delta \in (0,\infty)$. Evidently, given the proportion $\pi$, there is a one-to-one correspondence between $(1- \pi_{0|0}(\delta), \pi_{1|1}(\delta))$ and $(\pi_{1|1}(\delta), \preci(\delta))$, so that one curve can be transformed into the other \citep{davis2006relationship}. 

Precision strongly depends on $\pi$, and for small proportions of positives is often very low except for very large $\delta$ and hence small values of $\pi_{1|1}(\delta)$. 
{ Moreover, the precision-recall curve reflects the proportion of positives in the test data: if this proportion was set by the sampling design or by subsampling, then the displayed curve does not properly represent the population in which the classifier is applied.}
Somewhat surprisingly though, the precision-recall curve is still often recommended instead of the ROC curve for imbalanced classification problems \citep{cook2020consult}.

{In order to make both curves comparable, we recommend to replace the precision-recall curve by a plot of recall (TPR) against $1-$precision, that is $(1-\preci(\delta), \pi_{1|1}(\delta))$, $\delta \in (0,\infty)$. The dependent variable is then the same as for the ROC curve, and small values of the $x$-coordinate $1-\preci(\delta)$ and hence a fast increase becomes desirable, as for the ROC curve. It is then useful to plot both curves into a single figure, with the recall (TPR) against $1-$precision curve using the proportion of positives expected in the target population, and potentially for various values of $\pi$. }

Figure \ref{fig:lda-roc} shows the population ROC curve for LDA with $\Delta = 2$, see \eqref{eq:LDAdependdelta} in Example \ref{ex:lda}, in black, solid line. Further, for each value of the weight $\pi=0.5,0.2,0.05,0.02$, we plot the corresponding recall against $1-$precision curve with different line styles. The slopes decrease for decreasing values of the weight $\pi$, and for small values of $\pi$ an appropriately high value of TPR $\pi_{1|1}$ can only be achieved by allowing the precision to become small.

The ROC and recall against $1-$precision curves and aggregate measures such as the area under the curve (AUC) do not give immediate advice on how to select a particular classifier, that is including the threshold $\delta$ \citep{McDermott:2024}. To this end,  performance metrics come into play. 
A given performance metric such as the MCC results in an optimal value $\delta^*$ depending on the metric and hence in a certain point $(1-\pi_{0|0}(\delta^*), \pi_{1|1}(\delta^*))$ on the ROC curve and $(1-\preci(\delta^*), \pi_{1|1}(\delta^*))$ on the recall against $1-$precision curve. Use of a robust metric ensures that the $y$-coordinate of this point, its recall, {is at least $\rP_1(f_1/f_0 > \crr)$ irrespective of the proportion $\pi$.}%

 In Figure \ref{fig:lda-roc} the MCC optimal points for weights $\pi=0.5,0.2,0.05,0.01$ are shown as circles on the ROC curve as well as on the recall against $1-$precision curve corresponding to the value of $\pi$.  Triangles show the corresponding optimal points for the robust MCC with $d=0.15$. While the MCC optimal points strongly move towards the origin for small values of $\pi$, the optimal values for the robust MCC only move slightly in this direction.  

\vspace{0.2cm}

\begin{figure}   
\centerline{\includegraphics[scale=0.6]{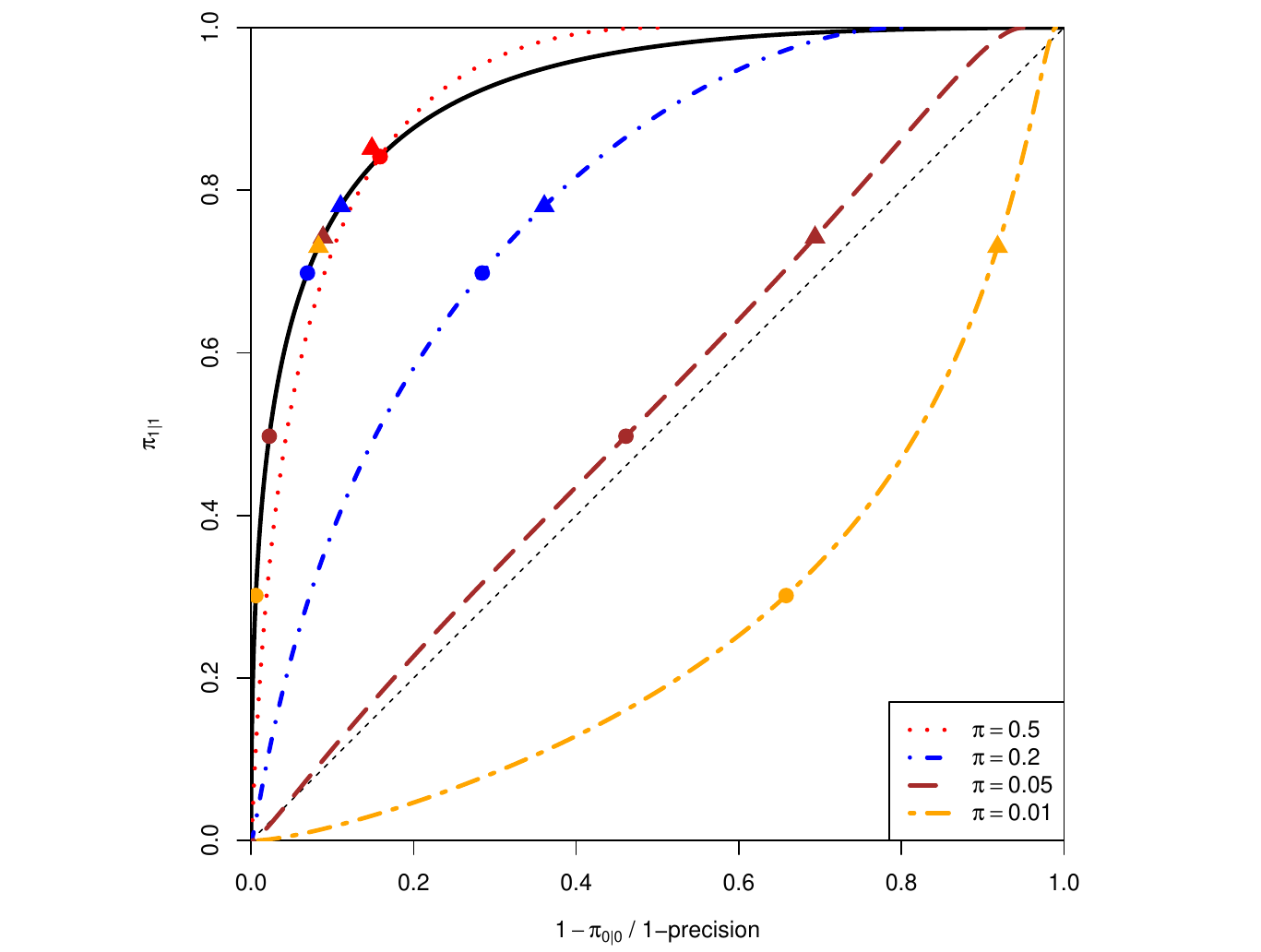}}
\caption{Population ROC curve for Example \ref{ex:lda} in black, solid line. Plot of recall against 1-precision in color with different line styles for varying $\pi$. Circles show the corresponding MCC optimal points; triangles show the points optimal with respect to the robust MCC with $d=0.15$.} \label{fig:lda-roc}
\end{figure}

\section{Data example} \label{data-ex}

In this section, we consider a dataset on credit default risk, available under \url{https://www.kaggle.com/competitions/GiveMeSomeCredit/data},
which contains {10} features and 150000 observations. The binary response \verb+SeriousDlqin2yrs+ indicates if a person experienced 90 days past due delinquency or worse. For this dataset, the percentage of positives is about 6.7\%. The data is analyzed on different occasions, a recent reference with further description is \cite{Dube:2023}. 
Two explanatory variables contain missing numbers; specifically, 19.8\% of the variable {\verb+MonthlyIncome+} and 2.62\% of the variable \verb+NumberOfDependents+ are missing. Since several trials with multiple imputation algorithms did not lead to improvements, we finally omitted both variables. We split the original dataset for training and testing, reserving 80\% for the training set and 20\% for the test set, therein retaining the proportion of positives.
Besides logistic regression models, we used random forests as implemented in the R package \verb+randomForest+ \citep{rf:2002} and gradient boosting as implemented  in the R package  \verb+xgboost+ \citep{xgboost:2026}. 

{To increase the imbalance, we also constructed subsets of the training and of the test set, each retaining all negatives and a random 20\% of the positives. This results in severely imbalanced data sets with a proportion of positives of $\hat\pi=0.014$, called {\itshape subset.train} and {\itshape subset.test} in the following.}

The results for the test set and the {\itshape subset.test} data for the logistic models can be found in the upper two panels of Tables \ref{tab-data-ex4}-\ref{tab-data-ex6}, while the middle and lower panels show the outcomes for the random forest and gradient boosting  classification on the test data. Corresponding results for the full training set and the {\itshape subset.train} data are shown in the Tables \ref{tab-data-ex1}-\ref{tab-data-ex3} in Appendix \ref{appsec:dataexampletrainingdata}. Computational details and computing times are presented in Appendix \ref{sec:computational-details}.

In each panel, we give the maximal value of the performance measure, the optimal cutoff $\tilde\delta$, true positive rate $\tilde\pi_{1|1}$ and the true negative rate $\tilde\pi_{0|0}$ of the pertaining classifiers. 
Tables \ref{tab-data-ex4} and \ref{tab-data-ex1} (in Appendix \ref{appsec:dataexampletrainingdata})  show the results for the performance measures used in Example \ref{ex:lda}, i.e. JAC, MCC,  $F_{1.5}$, $F_{0.5}$ and Kappa. Tables \ref{tab-data-ex5}, \ref{tab-data-ex6} and Tables \ref{tab-data-ex2}, \ref{tab-data-ex3} (in Appendix \ref{appsec:dataexampletrainingdata}) give the results for the robust F-score and the robust MCC defined in Section \ref{ex:fgen} and Section \ref{ex:mccgen}, respectively.
{For brevity, we do not report results for the robust Cohen's $\kappa$ of Section~\ref{ex:kappagen}; it is implemented in the R package mentioned in the introduction and can be applied in the same way.}

Overall the random forest classifier outperforms logistic regression, while gradient boosting and random forests perform rather similarly.  For all three methods, however, the observed patterns are valid. For all measures, the optimal threshold $\tilde\delta$ is smaller for the more imbalanced sample with $\hat\pi=0.014$ compared to $\hat\pi=0.067$. Nevertheless, the true positive rate $\hat\pi_{1|1}$, which is already quite low for most measures in Table \ref{tab-data-ex4} (and Table \ref{tab-data-ex1} in Appendix \ref{appsec:dataexampletrainingdata}) for $\hat\pi=0.067$, decreases sharply for the more imbalanced samples. Since a typical goal in predicting such data is to minimize or limit default risk, this does not seem desirable. 
The behavior is different for the robust F-score and the robust MCC measure, where the TPR is generally higher and decreases less or is stable depending on the choice of the parameters.  

In addition, Figures \ref{fig:table789-full} and \ref{fig:table789-subset} combine the empirical true positive rates $\hat{\pi}_{1|1}$ and true negative rates $\hat{\pi}_{0|0}$ in a single display, separately for the full dataset and the imbalanced subset. These figures are particularly useful for the practical choice of tuning parameters, as they make the trade-off between sensitivity and specificity explicit. In line with the previous results, robust metrics achieve substantially higher true positive rates at the cost of moderately reduced true negative rates, an effect that is especially pronounced in the imbalanced setting.

Further details are provided in Appendix \ref{appsec:dataexampletrainingdata}, where Figures \ref{fig:table789-pi11} and \ref{fig:table789-pi00} display the empirical true positive and true negative rates separately. These plots confirm that for classical metrics (JAC, MCC, $F_{1.5}$, $F_{0.5}$, Kappa), the true positive rate is already low and decreases substantially under stronger imbalance, whereas for the robust metrics it remains markedly higher and more stable across both settings. At the same time, classical metrics yield higher specificity, while the robust metrics trade a moderate decrease in $\hat{\pi}_{0|0}$ for a substantial gain in $\hat{\pi}_{1|1}$.

\begin{table}
\centering
\begin{tabular}{llrrrrr}
  \hline
 &  & JAC & MCC & $F_{1.5}$ & $F_{0.5}$ & Kappa \\ 
  \hline
test & value & 0.183 & 0.262 & 0.318 & 0.341 & 0.261 \\ 
  $\hat{\pi}=0.067$ & $\tilde\delta$ & 0.123 & 0.135 & 0.104 & 0.169 & 0.135 \\ 
   & $\hat\pi_{1|1}$ & 0.312 & 0.288 & 0.383 & 0.209 & 0.288 \\ 
   & $\hat\pi_{0|0}$ & 0.949 & 0.958 & 0.909 & 0.978 & 0.958 \\   \hline
subset.test & value & 0.066 & 0.123 & 0.149 & 0.131 & 0.105 \\ 
  $\hat{\pi}=0.014$ & $\tilde\delta$ & 0.043 & 0.036 & 0.035 & 0.061 & 0.050 \\ 
   & $\hat\pi_{1|1}$ & 0.147 & 0.200 & 0.204 & 0.102 & 0.122 \\ 
   & $\hat\pi_{0|0}$ & 0.983 & 0.974 & 0.971 & 0.991 & 0.986 \\ 
 \hline\hline
test & value & 0.289 & {0.407} & {0.487} & {0.461} & 0.407 \\
  $\hat{\pi}=0.067$ & $\tilde\delta$ & {0.220} & {0.190} & {0.154} & {0.383} & {0.238} \\
   & $\hat\pi_{1|1}$ & {0.500} & {0.540} & {0.594} & {0.290} & {0.480} \\
   & $\hat\pi_{0|0}$ & {0.948} & {0.937} & {0.920} & {0.982} & {0.953} \\ 
   \hline
subset.test & value & 0.101 & 0.184 & 0.221 & 0.169 & 0.169 \\ 
  $\hat{\pi}=0.014$ & $\tilde\delta$ & 0.126 & 0.114 & 0.107 & 0.153 & 0.126 \\ 
   & $\hat\pi_{1|1}$ & 0.212 & 0.249 & 0.272 & 0.155 & 0.212 \\ 
   & $\hat\pi_{0|0}$ & 0.984 & 0.981 & 0.979 & 0.989 & 0.984 \\ 
 \hline\hline
test & value & 0.287 & 0.407 & 0.482 & 0.462 & 0.406 \\ 
  $\hat{\pi}=0.067$ & $\tilde\delta$ & 0.218 & 0.231 & 0.160 & 0.355 & 0.233 \\ 
   & $\hat\pi_{1|1}$ & 0.490 & 0.478 & 0.566 & 0.344 & 0.475 \\ 
   & $\hat\pi_{0|0}$ & 0.950 & 0.953 & 0.928 & 0.976 & 0.954 \\  
   \hline
subset.test & value & 0.104 & 0.137 & 0.207 & 0.116 & 0.170 \\ 
  $\hat{\pi}=0.014$ & $\tilde\delta$ & 0.145 & 0.182 & 0.119 & 0.279 & 0.151 \\ 
   & $\hat\pi_{1|1}$ & 0.192 & 0.117 & 0.237 & 0.040 & 0.180 \\ 
   & $\hat\pi_{0|0}$ & 0.988 & 0.993 & 0.982 & 0.998 & 0.989 \\ 
   \hline
\end{tabular}
\caption{Test data: Maximal value of performance measure, optimal cutoff $\tilde\delta$,  TPR $\hat\pi_{1|1}$ and TNR $\hat\pi_{0|0}$ of classifiers that empirically maximize the various performance measures used in Section \ref{data-ex}. \\
Upper panel: Logistic regression. Middle panel: Random forests. Lower panel: Gradient boosting.}  \label{tab-data-ex4}
\end{table}

\begin{table}
\centering
\begin{tabular}{llrrrrrr}
  \hline
 & $(d_0,d_1)$ & $(0,1)$ & $(0.1,1)$ & $(0.2,1)$ & $(0.5,1)$ & $(0.8,1)$ & $(0.2,2)$ \\ 
  \hline
test & value & 0.309 & 0.323 & 0.350 & 0.436 & 0.500 & 0.371 \\ 
  $\hat{\pi}=0.067$ & $\tilde\delta$ & 0.123 & 0.103 & 0.088 & 0.057 & 0.048 & 0.080 \\ 
   & $\hat\pi_{1|1}$ & 0.312 & 0.388 & 0.471 & 0.699 & 0.767 & 0.523 \\ 
   & $\hat\pi_{0|0}$ & 0.949 & 0.906 & 0.839 & 0.570 & 0.449 & 0.790 \\ 
  \hline
subset.test & value & 0.125 & 0.227 & 0.277 & 0.406 & 0.492 & 0.284 \\ 
  $\hat{\pi}=0.014$ & $\tilde\delta$ & 0.043 & 0.025 & 0.019 & 0.014 & 0.011 & 0.019 \\ 
   & $\hat\pi_{1|1}$ & 0.147 & 0.324 & 0.459 & 0.653 & 0.781 & 0.459 \\ 
   & $\hat\pi_{0|0}$ & 0.983 & 0.935 & 0.840 & 0.668 & 0.504 & 0.840 \\ 
 \hline\hline
test & value & {0.449} & {0.499} & 0.529 & 0.616 & {0.664} & 0.551 \\
  $\hat{\pi}=0.067$ & $\tilde\delta$ & {0.220} & {0.146} & {0.104} & 0.066 & {0.049} & {0.085} \\
   & $\hat\pi_{1|1}$ & {0.500} & {0.614} & {0.687} & 0.802 & {0.847} & {0.739} \\
   & $\hat\pi_{0|0}$ & {0.948} & {0.915} & {0.871} & 0.781 & {0.714} & {0.835} \\
   \hline
subset.test & value & 0.183 & 0.351 & 0.439 & 0.545 & 0.610 & 0.447 \\ 
  $\hat{\pi}=0.014$ & $\tilde\delta$ & 0.126 & 0.047 & 0.031 & 0.014 & 0.010 & 0.031 \\ 
   & $\hat\pi_{1|1}$ & 0.212 & 0.501 & 0.621 & 0.763 & 0.815 & 0.621 \\ 
   & $\hat\pi_{0|0}$ & 0.984 & 0.937 & 0.899 & 0.783 & 0.714 & 0.899 \\ 
 \hline\hline
test & value & 0.446 & 0.496 & 0.531 & 0.617 & 0.668 & 0.554 \\ 
  $\hat{\pi}=0.067$ & $\tilde\delta$ & 0.218 & 0.136 & 0.105 & 0.066 & 0.044 & 0.099 \\ 
   & $\hat\pi_{1|1}$ & 0.490 & 0.617 & 0.675 & 0.795 & 0.859 & 0.692 \\ 
   & $\hat\pi_{0|0}$ & 0.950 & 0.911 & 0.879 & 0.790 & 0.704 & 0.871 \\
   \hline
subset.test & value & 0.189 & 0.365 & 0.445 & 0.557 & 0.623 & 0.454 \\ 
  $\hat{\pi}=0.014$ & $\tilde\delta$ & 0.145 & 0.041 & 0.025 & 0.017 & 0.013 & 0.025 \\ 
   & $\hat\pi_{1|1}$ & 0.192 & 0.519 & 0.646 & 0.728 & 0.786 & 0.646 \\ 
   & $\hat\pi_{0|0}$ & 0.988 & 0.938 & 0.891 & 0.832 & 0.778 & 0.891 \\ 
   \hline
\end{tabular}
\caption{Test data: Maximal value of performance measure, optimal cutoff $\tilde\delta$,  TPR $\hat\pi_{1|1}$ and TNR $\hat\pi_{0|0}$ of classifiers that empirically maximize the robust F-score with $c=0$ and different choices of $(d_0,d_1)$ used in Section \ref{data-ex}. \\
Upper panel: Logistic regression. Middle panel: Random forests. Lower panel: Gradient boosting.} \label{tab-data-ex5}
\end{table}

\begin{table}
\centering
\begin{tabular}{llrrrrrr}
  \hline
 & $d$ & $0$ & $0.01$ & $0.05$ & $0.1$ & $0.5$ & $1.0$ \\ 
  \hline
test & value & 0.262 & 0.260 & 0.260 & 0.264 & 0.288 & 0.298 \\ 
  $\hat{\pi}=0.067$ & $\tilde\delta$ & 0.135 & 0.127 & 0.123 & 0.104 & 0.088 & 0.088 \\ 
   & $\hat\pi_{1|1}$ & 0.288 & 0.303 & 0.312 & 0.383 & 0.471 & 0.471 \\ 
   & $\hat\pi_{0|0}$ & 0.958 & 0.953 & 0.949 & 0.909 & 0.839 & 0.839 \\ 
  \hline
subset.test & value & 0.123 & 0.153 & 0.194 & 0.216 & 0.269 & 0.282 \\ 
  $\hat{\pi}=0.014$ & $\tilde\delta$ & 0.036 & 0.029 & 0.025 & 0.025 & 0.019 & 0.019 \\ 
   & $\hat\pi_{1|1}$ & 0.200 & 0.277 & 0.324 & 0.324 & 0.459 & 0.459 \\ 
   & $\hat\pi_{0|0}$ & 0.974 & 0.955 & 0.935 & 0.935 & 0.840 & 0.840 \\ 
 \hline\hline
test & value & {0.407} & 0.414 & {0.443} & 0.470 & 0.529 & {0.548} \\
  $\hat{\pi}=0.067$ & $\tilde\delta$ & {0.190} & {0.190} & {0.154} & {0.117} & 0.085 & {0.085} \\
   & $\hat\pi_{1|1}$ & {0.540} & {0.540} & {0.594} & {0.669} & 0.739 & {0.739} \\
   & $\hat\pi_{0|0}$ & {0.937} & {0.937} & {0.920} & {0.887} & 0.835 & {0.835} \\ 
   \hline
subset.test & value & 0.184 & 0.246 & 0.343 & 0.396 & 0.486 & 0.513 \\ 
  $\hat{\pi}=0.014$ & $\tilde\delta$ & 0.114 & 0.063 & 0.033 & 0.031 & 0.017 & 0.017 \\ 
   & $\hat\pi_{1|1}$ & 0.249 & 0.426 & 0.603 & 0.621 & 0.728 & 0.728 \\ 
   & $\hat\pi_{0|0}$ & 0.981 & 0.955 & 0.906 & 0.899 & 0.819 & 0.819 \\ 
   \hline\hline

test & value & 0.407 & 0.415 & 0.445 & 0.468 & 0.534 & 0.556 \\ 
  $\hat{\pi}=0.067$ & $\tilde\delta$ & 0.231 & 0.198 & 0.136 & 0.116 & 0.085 & 0.076 \\ 
   & $\hat\pi_{1|1}$ & 0.478 & 0.518 & 0.617 & 0.654 & 0.734 & 0.765 \\ 
   & $\hat\pi_{0|0}$ & 0.953 & 0.944 & 0.911 & 0.894 & 0.843 & 0.820 \\ 
   \hline
subset.test & value & 0.137 & 0.231 & 0.343 & 0.387 & 0.500 & 0.527 \\ 
  $\hat{\pi}=0.014$ & $\tilde\delta$ & 0.182 & 0.065 & 0.041 & 0.036 & 0.017 & 0.017 \\ 
   & $\hat\pi_{1|1}$ & 0.117 & 0.374 & 0.519 & 0.549 & 0.728 & 0.728 \\ 
   & $\hat\pi_{0|0}$ & 0.993 & 0.962 & 0.938 & 0.928 & 0.832 & 0.832 \\
   \hline
\end{tabular}
\caption{Test data: Maximal value of performance measure, optimal cutoff $\tilde\delta$,  TPR $\hat\pi_{1|1}$ and TNR $\hat\pi_{0|0}$ of classifiers that empirically maximize the robust MCC with different choices of $d$ used in Section \ref{data-ex}. \\
Upper panel: Logistic regression. Middle panel: Random forests. Lower panel: Gradient boosting.} \label{tab-data-ex6}
\end{table}

\begin{figure}   
\centerline{\includegraphics[width=\textwidth]{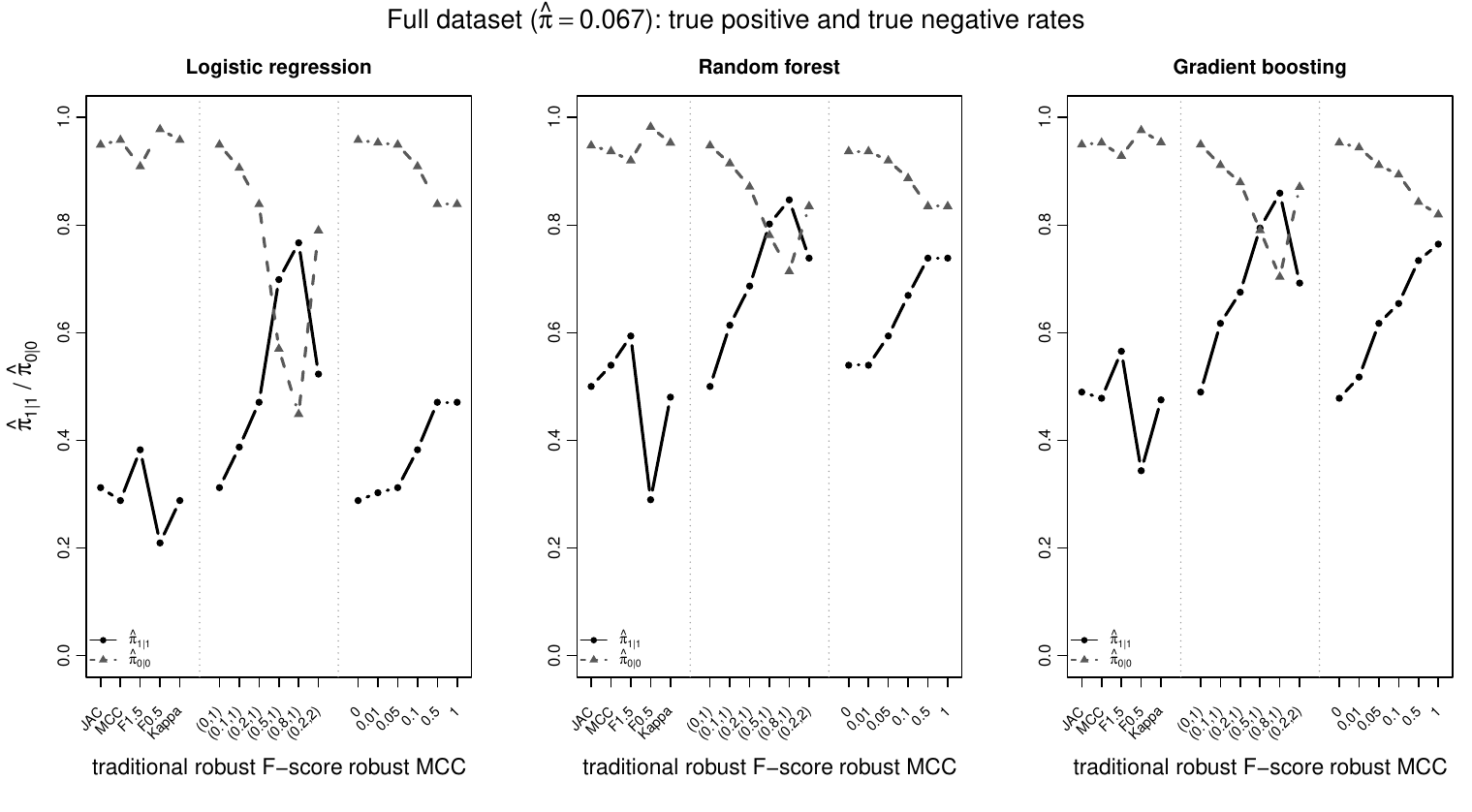}}
\caption{Empirical true positive rates $\hat{\pi}_{1|1}$ (solid black) and true negative rates $\hat{\pi}_{0|0}$ (dashed gray) of classifiers that maximize the respective performance measures for the full dataset ($\hat{\pi}=0.067$), shown for logistic regression, random forest, and gradient boosting. The left block displays classical metrics (JAC, MCC, $F_{1.5}$, $F_{0.5}$, Kappa), while the middle and right blocks correspond to the robust $F$-score and robust MCC with varying parameter choices.} \label{fig:table789-full}
\end{figure}

\begin{figure}   
\centerline{\includegraphics[width=\textwidth]{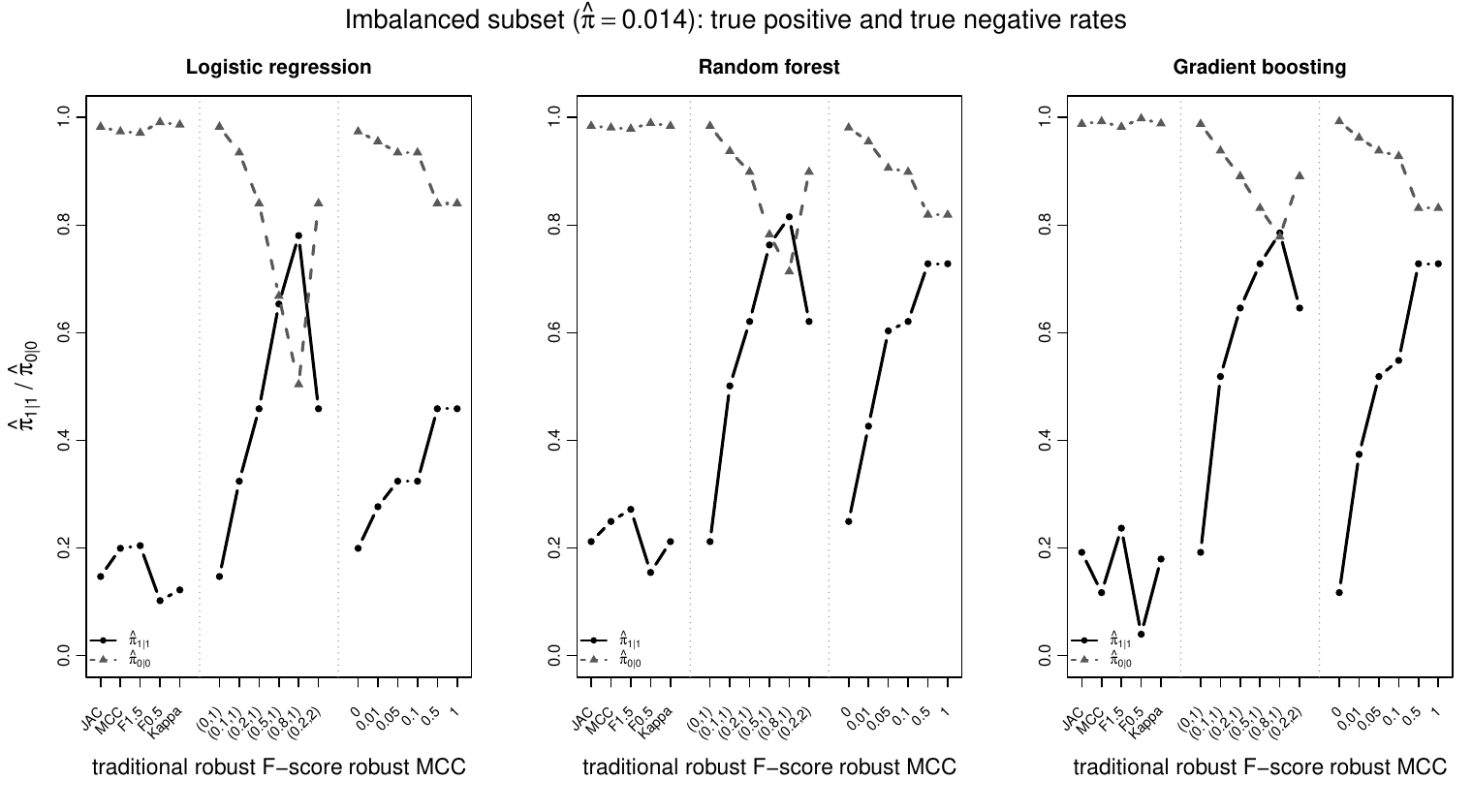}}
\caption{Same plot as Figure \ref{fig:table789-full} for the imbalanced subset ($\hat{\pi}=0.014$).} \label{fig:table789-subset}
\end{figure}

\medskip
Figure \ref{fig:data-ex-roc} shows the ROC curve of the gradient boosting classifier %
for the complete test data (in black) and the subset of test data (in grey) as solid lines, which are quite similar as expected. Graphs of recall against  1-precision (modified precision / recall curves)  are shown as dotted lines, which in contrast differ substantially: For the strongly imbalanced subset of test data a precision of less than 20\% has to be allowed to achieve non-zero TPR, while for the full test data the TPR becomes non-zero already at a precision of about 70\%.  

The MCC optimal points on the ROC curve as well as on the 1-precision / recall - curves are shown as circles. Triangles, diamonds and squares show the corresponding optimal points for the robust MCC with $d=0.01,0.05,0.1$, respectively, marking points with higher TPR but also higher values of FPR and lower precision values with increasing parameter $d$.  

\begin{figure}   
\centerline{\includegraphics[scale=0.55]{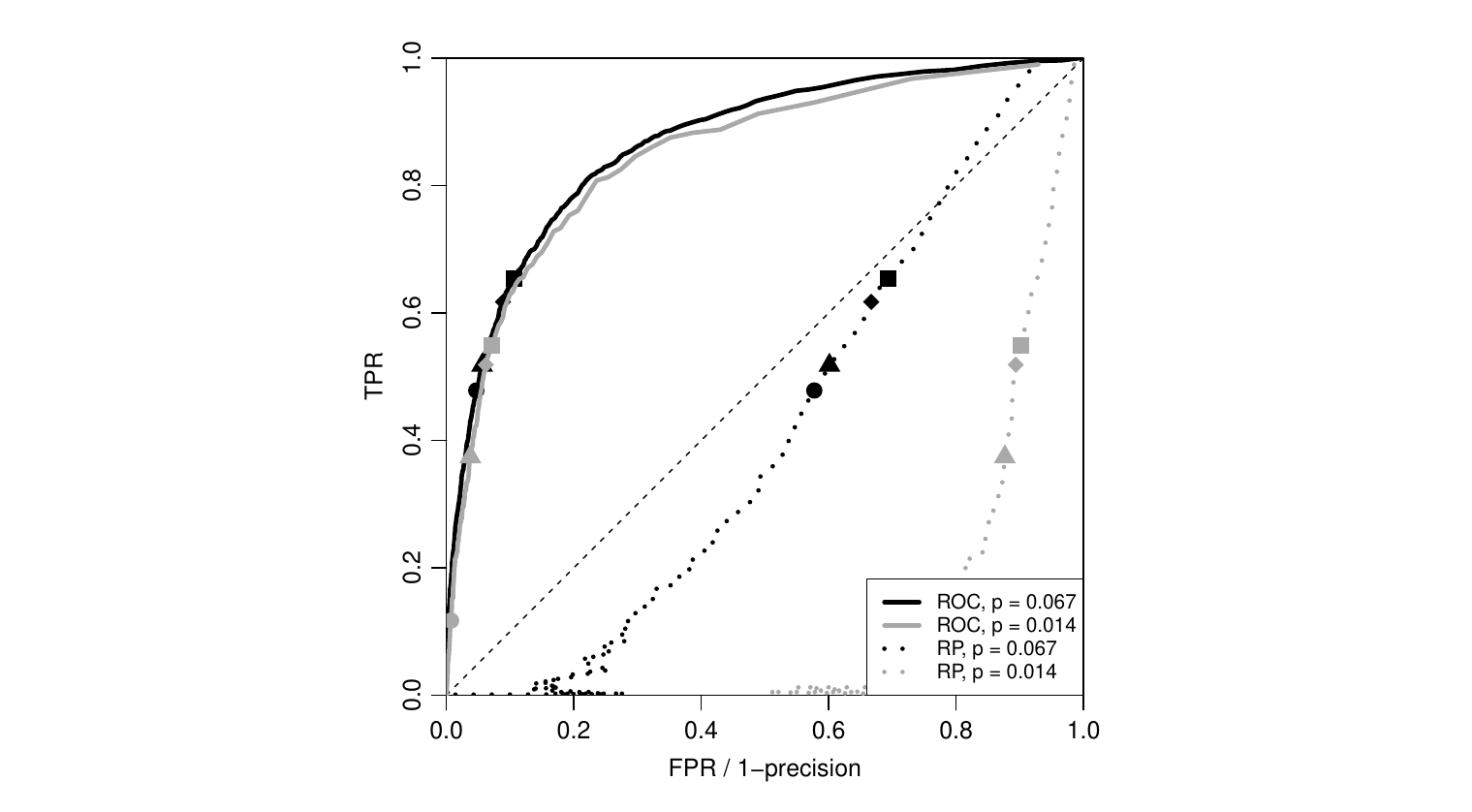}}
\caption{Solid lines: ROC curves of the gradient boosting classifier for the complete test data (black) and the subset of test data (grey) used in Section \ref{data-ex}. Dotted lines: Recall against 1-precision for the complete test data (black) and the subset of test data (grey). Circles show the corresponding MCC optimal points; points optimal with respect to the robust MCC with $d=0.01,0.05,0.1$ are shown as triangles, diamonds and squares.}\label{fig:data-ex-roc}
\end{figure}

\medskip
The left panel of Figure \ref{fig:data-ex-roc2} shows the ROC curves of logistic regression (black, solid line), the random forest classifier (dark gray, solid line) and the gradient boosting classifier (light grey, solid line) for the complete test data. Graphs of 1-precision against recall are shown as dotted lines (logistic regression in black, random forest and gradient boosting classifier in dark and light gray). The MCC optimal points on the ROC curve as well as on the recall / 1-precision curves are shown as circles. Triangles, diamonds and squares show the corresponding optimal points for the robust MCC with $d=0.01,0.05,0.1$, respectively. 
The right panel of Figure \ref{fig:data-ex-roc2} shows the same plot for the subset of the test data. 
In both cases, the curves for the random forest classifier lie above those for logistic regression, while the curves for gradient boosting and random forest classifiers are very similar and essentially overlap.

\begin{figure}   
\centerline{\includegraphics[width=0.95\textwidth]{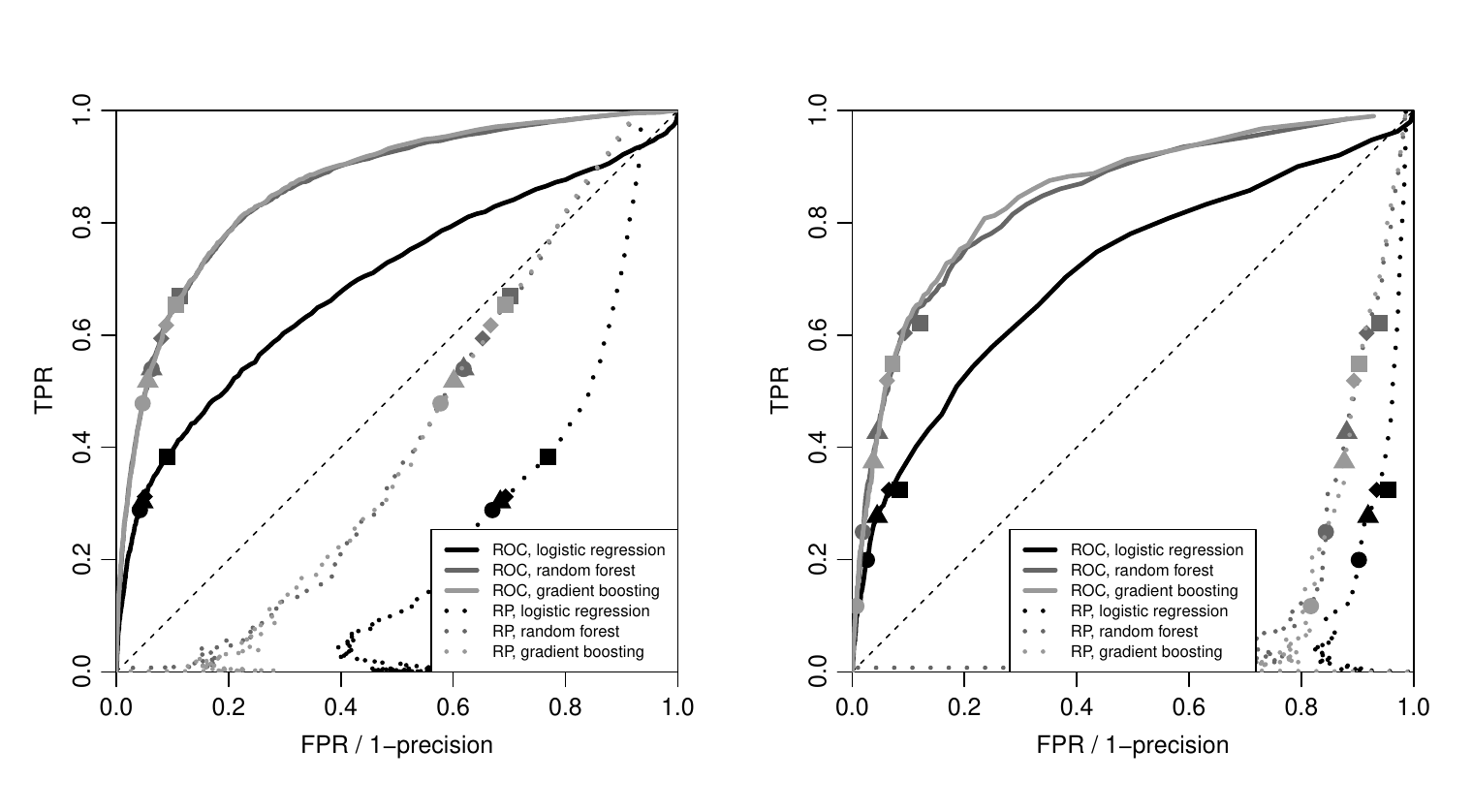}}
\caption{Left panel: Solid lines show ROC curves of logistic regression (black), random forest (dark gray) and gradient boosting classifier (light gray) for the complete test data used in Section \ref{data-ex}. Dotted lines show recall against 1-precision curve of logistic regression (black), random forest (dark gray) and gradient boosting classifier (light gray) for this data set.
Circles show the corresponding MCC optimal points; points optimal with respect to the robust MCC with $d=0.01,0.05,0.1$ are shown as triangles, diamonds and squares.
Right panel: Same plot for the subset of test data used in Section \ref{data-ex}.} \label{fig:data-ex-roc2}
\end{figure}

\section{Conclusions}\label{sec:conclude}

The overall performance of classification methods, regardless of the choice of the threshold, can be assessed using ROC and precision-recall curves. We recommend starting with the  ROC curve, which plots TPR against FPR and thus does not depend on the weight $\pi$ of the positive class, and then adding plots of  recall (TPR) against 1-precision for potentially several prescribed values of the weight $\pi$ which might be of interest in practical settings. 

Sometimes, as in Section \ref{data-ex},  one method dominates overall, in that the ROC curve and hence all recall against 1-precision curves of that method lie above the corresponding curves of the competing method. More often, the curves intersect, making the choice of the method less obvious. While aggregate measures such as the area under the curve (AUC) are frequently used, in order to choose specific classifiers - that is threshold values together with the overall method - performance metrics can be more useful. We recommend using robust metrics, potentially with various settings for the robustness parameters. The resulting classifiers can be visualized as points on the ROC and recall against 1-precision curves to finally choose a classifier. 

Let us also mention that for traditional performance metrics such as the MCC, a correlation, or the Jaccard similarity coefficient, which captures the similarity of two events in terms of the probability of intersection divided by probability of union, the actual value of the performance metric is sometimes considered and interpreted. However, as stressed by our study, such performance metrics should mainly be used to determine the associated Bayes-optimal threshold and then rank different classification methods such as logistic regression vs.~gradient boosting, each with the associated optimal threshold. 
For our proposed robustified versions of the MCC, Cohen's $\kappa$ and the F-score, we do scale their values  to the intervals $[-1,1]$ and $[0,1]$, respectively, to allow for some interpretability and comparability of the values, also across different values  of the tuning parameters. However, the main use of the performance metrics is the determination and selection of the best - that is Bayes-optimal - classifier for this metric.

Potential extensions of our methods include constructions for multiclass classification problems, as well as the construction of robust modifications for further performance metrics in binary classification.

\subsection*{Data availability statement}

The data used in Section \ref{data-ex}
is publicly available at \url{https://www.kaggle.com/competitions/GiveMeSomeCredit/data}. 

\subsection*{Funding declaration}

There was no third-party funding which supported the preparation of the manuscript. 

\subsection*{Competing Interest declaration}

The authors have no conflicts of interest to declare that are relevant to the content of this article.

\subsection*{Author contribution statement}

The authors contributed equally to the preparation of the manuscript.

{\subsection*{Declaration of use of AI}

The original manuscript and the main theoretical results were developed without the use of generative AI. During later revisions, the authors used OpenAI's ChatGPT and Anthropic's Claude for language editing, checking selected proofs and mathematical arguments, assisting with standard R simulation code, and as a critical reading and discussion tool. This included prompting a re-examination of the treatment of Yule’s coefficients and several clarifications made in the final revision. All AI-assisted suggestions, mathematical arguments, computations, simulations, and resulting revisions were independently checked by the authors, who take full responsibility for the content of the manuscript.}

\clearpage

\bibliographystyle{apalike}
\bibliography{biblio}

\clearpage
\appendix
\small

\section{Additional technical material for Sections \ref{sec:perfmetrics} and \ref{robust-metrics}}\label{sec:appperfmetrics}

{This appendix collects the proofs and technical results for Sections~\ref{sec:perfmetrics} and~\ref{robust-metrics}. Appendix~\ref{app:proofth3} contains the proof of Theorem~\ref{th:decisiobnounddens} and the quotient \eqref{eq:thequotient} of partial derivatives for the classical performance metrics. Appendix~\ref{app:ldaqda} derives sensitivity and specificity of the threshold classifiers under LDA and QDA. Appendix~\ref{app:yule} treats Yule's coefficients and their relation to the diagnostic odds ratio: under LDA they admit no Bayes classifier, and the solution $\delta = 1$ of \eqref{eq:formofdelta} minimizes them among the threshold classifiers. Appendix~\ref{app:robust} contains the existence of Bayes classifiers for the robust performance metrics.}%

\subsection[Proof of Theorem \ref{th:decisiobnounddens} and quotients of partial derivatives]{{Proof of Theorem \ref{th:decisiobnounddens} and quotients of partial derivatives}}\label{app:proofth3}

\begin{proof}[Proof of Theorem \ref{th:decisiobnounddens}]
For classifier $\phi$ and densities $f_i$, $i=0,1$ of the class-conditional distribution w.r.t.~the dominating measure $\mu$, which can be assumed to be finite, we have
\begin{align}\label{eq:mapsagain}
	\begin{split}
	\pi_{1|1}(\phi,f_1) = \pi_{1|1}(\phi) = \PP\big(\phi(X)=1 | Y=1 \big) & =  \int_\cX \phi(x)\, f_1(x)\, \dd \mu(x),\\
	\pi_{0|0}(\phi,f_0) = \pi_{0|0}(\phi) = \PP\big(\phi(X)=0 | Y=0 \big) & =  \int_\cX \big(1-\phi(x)\big)\, f_0(x)\, \dd \mu(x).
	\end{split}
\end{align}

The utility of $\phi$ for the metric $M$ with given $\pi \in (0,1)$ is defined by
\begin{equation}\label{eq:utility}
 \UU(\phi;\pi) = \UU(\phi;f_1,f_0,\pi) = M\big(\pi_{1|1}(\phi),\pi_{0|0}(\phi),\pi \big).
 \end{equation}
Let $\cF = \{ \varphi: \cX \to [0,1] \text{ measurable}\}$, a convex subset of $\LL^2(\cX,\mu)$. If $M$ is differentiable in $\pi_{1|1}$ and $\pi_{0|0}$, then $\UU(\cdot;\pi)$ is Fr\'echet differentiable with derivative $\nabla \UU(\cdot;\pi)$, and a necessary condition for a local maximum $\phi^*$ of $\UU$ on $\cF$ is
\begin{equation}\label{eq:extremecond}	
 \int_\cX \nabla \UU(\phi^*;\pi)(x)\, \big(\phi^*(x) - \phi(x) \big)\, \dd \mu(x) \geq 0 \qquad \forall \phi \in \cF.
\end{equation}
Now, the integral operators in \eqref{eq:mapsagain} are linear and affine linear, so that the derivatives are integral operators with kernel $f_1$ and $- f_0$, respectively. Therefore, letting
$$ m_1\big(\pi_{1|1},\pi_{0|0},\pi \big) = \partial_{\pi_{1|1}}\,  M\big(\pi_{1|1},\pi_{0|0},\pi \big)
, \qquad m_2\big(\pi_{1|1},\pi_{0|0},\pi \big) = \partial_{\pi_{0|0}}\,  M\big(\pi_{1|1},\pi_{0|0},\pi \big),$$
we can evaluate $\nabla \UU(\cdot;\pi)$ by the chain rule as the integral operator with kernel
$$ k(x;\phi,\pi,f_0,f_1) = m_1\big(\pi_{1|1}(\phi),\pi_{0|0}(\phi),\pi \big)\, f_1(x)  - m_2\big(\pi_{1|1}(\phi),\pi_{0|0}(\phi),\pi \big)\, f_0(x).$$
Setting
\begin{align*}
	A^+ & = \big\{x \in \cX \mid k(x;\phi^*,\pi,f_0,f_1)>0 \big\} = \Big\{x \in \cX \mid \frac{f_1(x)}{f_0(x)} > \frac{m_2\big(\pi_{1|1}(\phi^*),\pi_{0|0}(\phi^*),\pi \big)}{m_1\big(\pi_{1|1}(\phi^*),\pi_{0|0}(\phi^*),\pi \big)} \Big\},\\
	A^- & = \big\{x \in \cX \mid k(x;\phi^*,\pi,f_0,f_1)<0 \big\} = \Big\{x \in \cX \mid \frac{f_1(x)}{f_0(x)} < \frac{m_2\big(\pi_{1|1}(\phi^*),\pi_{0|0}(\phi^*),\pi \big)}{m_1\big(\pi_{1|1}(\phi^*),\pi_{0|0}(\phi^*),\pi \big)} \Big\},\\
	A^0 & = \big\{x \in \cX \mid k(x;\phi^*,\pi,f_0,f_1)=0 \big\} = \Big\{x \in \cX \mid \frac{f_1(x)}{f_0(x)} = \frac{m_2\big(\pi_{1|1}(\phi^*),\pi_{0|0}(\phi^*),\pi \big)}{m_1\big(\pi_{1|1}(\phi^*),\pi_{0|0}(\phi^*),\pi \big)} \Big\},
\end{align*}	
we can write \eqref{eq:extremecond} as
\begin{align*}
	\int_{A^+} \, k(x;\phi^*,\pi,f_0,f_1)\, \big(\phi^*(x) - \phi(x) \big)\, \dd \mu(x) + \int_{A^-} \, k(x;\phi^*,\pi,f_0,f_1)\, \big(\phi^*(x) - \phi(x) \big)\, \dd \mu(x) \geq 0,
\end{align*}
for $ \phi \in \cF$. Hence an optimal classifier has to be $1$ on $A^+$ and $0$ on $A^-$.
\end{proof}

\begin{example}[The quotient \eqref{eq:thequotient} for various performance metrics] \label{ex-metric-deriv1}
\ 
\begin{itemize}
    \item g-mean score: 
    $$
    M_{\text{G}}(\pi_{1|1}, \pi_{0|0}, \pi) = \sqrt{\pi_{1|1}\pi_{0|0}}  
    $$
    is independent of $\pi$ and 
    $$
    \frac{\partial_{\pi_{0|0}} M_{\text{G}}(\pi_{1|1}, \pi_{0|0}, \pi)}{\partial_{\pi_{1|1}} M_{\text{G}}(\pi_{1|1}, \pi_{0|0}, \pi)} = 
    \frac{\pi_{1|1}}{\pi_{0|0}}.
    $$
	\item JAC:
	$$
    M_{\text{JAC}}(\pi_{1|1}, \pi_{0|0}, \pi) = \frac{\pi \, \pi_{1|1}}{1-\pi_{0|0}\, (1-\pi)}, 
    $$
	for which
	$$\frac{\partial_{\pi_{0|0}} M_{\text{JAC}}(\pi_{1|1}, \pi_{0|0}, \pi)}{\partial_{\pi_{1|1}} M_{\text{JAC}}(\pi_{1|1}, \pi_{0|0}, \pi)} = \frac{\pi_{1|1}\, (1-\pi)}{1-\pi_{0|0}\, (1-\pi)}.$$
	\item $F_{\beta}$-score: Here, %
	$$M_{\text{F}_{\beta}}(\pi_{1|1}, \pi_{0|0}, \pi) = \frac{(1+\beta^2)\,\pi_{11}}{\beta^2\, (\pi_{11}+\pi_{10}) + \pi_{11} + \pi_{01}} = \frac{(1+\beta^2)\, \pi \, \pi_{1|1}}{\beta^2\, \pi + \pi \, \pi_{1|1} + (1-\pi_{0|0})\, (1-\pi)}~. $$
	Then, we obtain
	$$\frac{\partial_{\pi_{0|0}} M_{\text{F}_{\beta}}(\pi_{1|1}, \pi_{0|0}, \pi)}{\partial_{\pi_{1|1}} M_{\text{F}_{\beta}}(\pi_{1|1}, \pi_{0|0}, \pi)} = \frac{\pi_{1|1}\, (1-\pi)}{\beta^2\, \pi + (1-\pi_{0|0})\, (1-\pi)}.$$

	\item JAC and $F_1$-score: The quotient $\partial_{\pi_{0|0}} M_{\text{F}_1} / \partial_{\pi_{1|1}} M_{\text{F}_1}$ coincides with the corresponding ratio for $M_{\text{JAC}}$, {so that both metrics have the same Bayes classifiers. Comparing this quotient with $M_{\text{JAC}}$ shows that it equals $(1-\pi)\, M_{\text{JAC}}(\pi_{1|1}, \pi_{0|0}, \pi)/\pi$. Hence for a Bayes classifier $\phi_{\delta^*}$ as in Theorem~\ref{th:decisiobnounddens},
\eqref{eq:formofdelta} reads
$\delta^* = (1-\pi)\, M_{\text{JAC}}\big(\pi_{1|1}(\phi_{\delta^*}), \pi_{0|0}(\phi_{\delta^*}), \pi\big)/\pi$,
the maximal value of the JAC times $(1-\pi)/\pi$, which is the characterization in
\citet[Theorem~1]{Lipton:2014}.}%

	\item MCC:
	\begin{align*}
	M_{\text{MCC}}&(\pi_{1|1}, \pi_{0|0}, \pi)  = \frac{\pi_{11}\, \pi_{00} - \pi_{01}\, \pi_{10}}{\big(\gamma \, (1-\gamma)\,\pi\, (1-\pi)  \big)^{1/2}}\\
	& = \big(\pi (1-\pi) \big)^{1/2}   \, \frac{\pi_{1|1}\, \pi_{0|0} - (1-\pi_{1|1})\, (1-\pi_{0|0})}{\Big(\big(\pi \pi_{1|1} + (1-\pi) (1- \pi_{0|0})\big) \big(\pi (1-\pi_{1|1}) + (1-\pi) \pi_{0|0}\big) \Big)^{1/2}}~.
	\end{align*}
	Then,
$$\frac{\partial_{\pi_{0|0}} M_{\text{MCC}}(\pi_{1|1}, \pi_{0|0}, \pi)}{\partial_{\pi_{1|1}} M_{\text{MCC}}(\pi_{1|1}, \pi_{0|0}, \pi)}
= \frac{(\pi_{0|0} + \pi_{1|1} - 1)(2\pi_{1|1} - 1)\pi + (1-2\pi_{1|1})\pi_{0|0} + \pi_{1|1} - 1}{(1-2\pi_{0|0})(\pi_{0|0} + \pi_{1|1} - 1)\pi + 2\pi_{0|0}^2 - 2\pi_{0|0}}.
$$

	\item  Cohen's $\kappa$:
\begin{align*}
	M_{\kappa}&(\pi_{1|1}, \pi_{0|0}, \pi)  = \frac{2(\pi_{11}\, \pi_{00} - \pi_{01}\, \pi_{10})}{\pi \, (1-\gamma) + (1-\pi)\, \gamma}\\
	& = \frac{2\pi(1-\pi)\, (\pi_{1|1}+\pi_{0|0}-1)}{2\pi(1-\pi)\, (\pi_{1|1}+\pi_{0|0}-1) - \pi \pi_{1|1} - (1-\pi) \pi_{0|0} + 1}.
\end{align*}
	Then,
$$\frac{\partial_{\pi_{0|0}} M_{\kappa}(\pi_{1|1}, \pi_{0|0}, \pi)}{\partial_{\pi_{1|1}} M_{\kappa}(\pi_{1|1}, \pi_{0|0}, \pi)}
= \frac{\pi_{1|1}+\pi\, (1-2\pi_{1|1})}{1-\pi_{0|0}-\pi\, (1-2\pi_{0|0})}.
$$

\end{itemize}
\end{example}

\subsection[Sensitivity and specificity under LDA and QDA]{{Sensitivity and specificity under LDA and QDA}}\label{app:ldaqda}

\begin{example}[Sensitivity and specificity for LDA and QDA] \label{ex:ldaqdaetchnical}
First consider LDA as in Example \ref{ex:lda}: $\rP_i = \mathcal N(\mu_i,\Sigma)$ with $\Sigma$ invertible. For the classifier $\phi_\delta = 1(f_1/f_0 > \delta)$, 
we have that 
$$ \phi(x) = 1 \quad \Longleftrightarrow (x-\bar \mu)^\top \Sigma^{-1} (\mu_1 - \mu_0) > \log \delta$$
where $\bar \mu = (\mu_1 + \mu_0)/2$. Set $\Delta^2 = (\mu_1 - \mu_0)^\top\, \Sigma^{-1}\, (\mu_1 - \mu_0) $. Under $\mathcal N(\mu_1,\Sigma)$, the discriminant function is distributed $\mathcal N(\Delta^2/2,\Delta^2)$, and under $\mathcal N(\mu_0,\Sigma)$ it is distributed $\mathcal N(-\Delta^2/2,\Delta^2)$. Therefore, we obtain the formulas \eqref{eq:LDAdependdelta}, 
\begin{equation*}
 \pi_{0|0}(\delta) = \Phi\Big( \frac{\log \delta + \Delta^2/2}{\Delta}\Big),\qquad \pi_{1|1}(\delta) = \Phi\Big( \frac{-\log \delta + \Delta^2/2}{\Delta}\Big),
\end{equation*} 

Next let us turn to QDA as in Example \ref{ex:qda}: $\rP_i = \mathcal N(\mu_i,\Sigma_i)$ with $\Sigma_i$ invertible. For the classifier $\phi_\delta(x) = 1(f_1(x)/f_0(x) > \delta)$, we have that $\phi(x) = 1$ if and only if
\begin{align*}
& x^\top \big(\Sigma_1^{-1} - \Sigma_0^{-1} \big) x - 2\, x^\top \big(\Sigma_1^{-1}\mu_1 - \Sigma_0^{-1}\mu_0 \big) + \mu_1^\top \Sigma_1^{-1}\mu_1 - \mu_0^\top \Sigma_0^{-1}\mu_0  \\
< & 2\, \log (1/\delta) + \log\Big(\frac{|\Sigma_0|}{|\Sigma_1|} \Big)    
\end{align*}

Under $\rP_1$, $X$ has the distribution of $\Sigma_1^{1/2} Y + \mu_1$ with $Y$ multivariate standard normal, and $\phi(x) = 1$ if and only if
\begin{align*}
& y^\top \big(I - \Sigma_1^{1/2}\Sigma_0^{-1} \Sigma_1^{1/2}\big) y - 2\, y^\top \Sigma_1^{1/2}  \Sigma_0^{-1}(\mu_1 - \mu_0) \\
< & 2\, \log (1/\delta) + \log\Big(\frac{|\Sigma_0|}{|\Sigma_1|} \Big) + (\mu_1 - \mu_0)^\top \Sigma_0^{-1}(\mu_1 - \mu_0)  
\end{align*}
Take $Q$ orthogonal such that 
$$Q^\top \big(I - \Sigma_1^{1/2}\Sigma_0^{-1} \Sigma_1^{1/2}\big) Q = \text{diag}(\lambda_1, \ldots, \lambda_d).$$
Set 
$$ z = Q^\top y, \qquad a = - \, Q^\top \, \Sigma_1^{1/2}  \Sigma_0^{-1}(\mu_1 - \mu_0),  $$
then the left side equals
$$\sum_{i=1}^d (\lambda_i z_i^2 + 2\, a_i z_i) = \sum_{i : \lambda_i \not=0} \lambda_i (z_i + a_i /\lambda_i)^2 - \sum_{i : \lambda_i \not=0}^d a_i^2/\lambda_i + 2\, \sum_{i : \lambda_i =0}^d a_i z_i,$$
and $(z_i + a_i /\lambda_i)^2$ has a non-central $\chi^2$ - distribution with non-centrality parameter $a_i^2 /\lambda_i^2$. 

\vspace{0.5cm}

Similarly, under $\rP_0$, $X$ has the distribution of $\Sigma_0^{1/2} Y + \mu_0$ with $Y$ multivariate standard normal, and $\phi(x) = 0$ if and only if
\begin{align*}
& y^\top \big(I - \Sigma_0^{1/2}\Sigma_1^{-1} \Sigma_0^{1/2}\big) y - 2\, y^\top \Sigma_0^{1/2}  \Sigma_1^{-1}(\mu_0 - \mu_1) \\
\leq & 2\, \log (\delta) + \log\Big(\frac{|\Sigma_1|}{|\Sigma_0|} \Big) + (\mu_0 - \mu_1)^\top \Sigma_1^{-1}(\mu_0 - \mu_1)  
\end{align*}
Now take $\tilde Q$ orthogonal such that 
$$\tilde Q^\top \big(I - \Sigma_0^{1/2}\Sigma_1^{-1} \Sigma_0^{1/2}\big) \tilde Q = \text{diag}(\tilde \lambda_1, \ldots, \tilde \lambda_d).$$
Set 
$$ z = \tilde Q^\top y, \qquad \tilde a = - \, \tilde Q^\top \, \Sigma_0^{1/2}  \Sigma_1^{-1}(\mu_0 - \mu_1),  $$
then the left side equals
$$\sum_{i=1}^d (\tilde \lambda_i z_i^2 + 2\, \tilde a_i z_i) = \sum_{i : \tilde \lambda_i \not=0} \tilde \lambda_i (z_i + \tilde a_i / \tilde \lambda_i)^2 - \sum_{i : \tilde \lambda_i \not=0}^d \tilde a_i^2/\tilde \lambda_i + 2\, \sum_{i : \tilde \lambda_i =0}^d \tilde a_i z_i,$$
and $(z_i + \tilde a_i /\tilde \lambda_i)^2$ has a non-central $\chi^2$ - distribution with non-centrality parameter $\tilde a_i^2 /\tilde \lambda_i^2$. 
\end{example}

\subsection[Yule's coefficients and the diagnostic odds ratio]{{Yule's coefficients and the diagnostic odds ratio}}\label{app:yule}

{\begin{example}[Yule's $Q$ and $Y$ and diagnostic odds ratio] \label{ex-yule-deriv}
    Consider    Yule's coefficient of association \citep{Yule:1912}, Yule's $Q$, 
    $$
    M_{\text{Q}}(\pi_{1|1}, \pi_{0|0}, \pi) = \frac{\pi_{1|1}\pi_{0|0} - (1-\pi_{1|1})(1-\pi_{0|0})}{\pi_{1|1}\pi_{0|0}+(1-\pi_{1|1})(1-\pi_{0|0})}, 
    $$
    as well as his coefficient of colligation \citep{Yule:1912}, Yule's $Y$,
    $$
    M_{\text{Y}}(\pi_{1|1}, \pi_{0|0}, \pi) = \frac{ \sqrt{\pi_{1|1}\pi_{0|0}} - \sqrt{(1-\pi_{1|1})(1-\pi_{0|0})} }{ \sqrt{\pi_{1|1}\pi_{0|0}} + \sqrt{(1-\pi_{1|1})(1-\pi_{0|0})}}, 
    $$
    both independent of $\pi$. 
    The coefficients can be continuously extended to boundary values of $(\pi_{1|1}, \pi_{0|0})$ except at $(0,1), (1,0)$, with values $=1$ %
    on $\{1\} \times (0,1]$, $  (0,1] \times \{1\}$, and $=-1$ on $\{0\} \times [0,1)$, $[0,1) \times \{0\}$ for $M_{\text{Q}}$ and  $M_{\text{Y}}$.

\smallskip

    \textbf{Non-existence of a Bayes classifier for Yule's Q and Y.} Suppose that the class-conditional distributions $\rP_0$ and $\rP_1$ are mutually absolutely continuous, so that 
    the only attainable points on the boundary of $[0,1]^2$ %
    of $(\pi_{1|1}, \pi_{0|0})$ that any classifier may achieve are the pairs $(0,1)$ and $(1,0)$ for the trivial classifiers, for which the coefficients are not defined. Then, in particular, no classifier attains the value $=1$ for the Yule coefficients.  

    Further, from \eqref{eq:inequyule} below, if $\pi_{1|1}(\phi_\delta)>0$ for every $\delta>0$ (as is the case for LDA), 
    $$M_{\text{Q}}\big(\pi_{1|1}(\phi_\delta), \pi_{0|0}(\phi_\delta), \pi\big) \to 1, \qquad \delta \to \infty,$$
    analogously for $M_{\text{Y}}$, and if $\pi_{0|0}(\phi_\delta)>0$ for every $\delta>0$ then 
    $$M_{\text{Q}}\big(\pi_{1|1}(\phi_\delta), \pi_{0|0}(\phi_\delta), \pi\big) \to 1, \qquad \delta \to 0.$$
Hence there does not exist a Bayes classifier for $M_{\text{Q}}$ nor for $M_{\text{Y}}$ under LDA. 

\smallskip

\textbf{Relation to diagnostic odds ratio.} \quad The Yule coefficients can be expressed for  $\pi_{1|1}\,\pi_{0|0} \in [0,1)$ in the \textit{diagnostic odds ratio}
    $$M_{\text{DOR}}(\pi_{1|1}, \pi_{0|0}, \pi) = \frac{\pi_{1|1}\,\pi_{0|0}}{(1-\pi_{1|1})\,(1-\pi_{0|0})}$$
   as
  \begin{align}\label{eq:expressyuleodds}
 \begin{split}
  & M_{\text{Q}}(\pi_{1|1}, \pi_{0|0}, \pi) = \frac{M_{\text{DOR}}(\pi_{1|1}, \pi_{0|0}, \pi)-1}{M_{\text{DOR}}(\pi_{1|1}, \pi_{0|0}, \pi)+1},\\
& M_{\text{Y}}(\pi_{1|1}, \pi_{0|0}, \pi) = \frac{\sqrt{M_{\text{DOR}}(\pi_{1|1}, \pi_{0|0}, \pi)}-1}{\sqrt{M_{\text{DOR}}(\pi_{1|1}, \pi_{0|0}, \pi)}+1}.
     \end{split}
    \end{align}

\textbf{Minimizer at $\boldsymbol{\delta=1}$ among threshold classifiers under LDA.} 
Under LDA, for $M_{\text{Q}}$, $M_{\text{Y}}$ and $M_{\text{DOR}}$ the equation \eqref{eq:formofdelta} does have the solution $\delta=1$, but it is a minimum among the threshold classifiers. See also \citet[Theorem 1]{Bohning:2011}.

Indeed, for all three performance metrics the quotient of partial derivatives \eqref{eq:thequotient} equals $\pi_{1|1}(1-\pi_{1|1})/\{\pi_{0|0}(1-\pi_{0|0})\}$, which under LDA takes the value $1$ at $\delta=1$, since $\pi_{1|1}(\phi_1)=\pi_{0|0}(\phi_1)=\Phi(\Delta/2)$ by \eqref{eq:LDAdependdelta}. 

Now, by \eqref{eq:expressyuleodds} $M_{\text{Q}}$ and $M_{\text{Y}}$ are strictly increasing functions of the diagnostic odds ratio $M_{\text{DOR}}$, so it suffices to show that $M_{\text{DOR}}$, or its logarithm, has a minimum at {$\delta=1$} among threshold classifiers. 
Again using \eqref{eq:LDAdependdelta}, the  logarithm $\ell$ of $M_{\text{DOR}}$ for the classifier $\phi_\delta$ under LDA, parametrized in $s = \log(\delta)/\Delta$, is
$$
\ell(s) = g(\Delta/2-s) + g(\Delta/2+s), \qquad s = \log(\delta)/\Delta, \quad g(u) = \log\big(\Phi(u)/(1-\Phi(u))\big),
$$
where $g$ is the logarithmic odds ratio of the standard normal distribution. 
It remains to show that $\ell$ has a minimizer at $s=0$, and for this we show that $\ell'(s)>0$ for $s>0$ and $\ell'(s)<0$ for $s<0$. This will follow from analytic properties of $g$:   
The logarithmic odds ratio $g$ is odd, so that its derivative $g'(u) = \Phi'(u)/(\Phi(u)\,(1-\Phi(u)))$ is even. Moreover, $g'$ is strictly increasing on $[0,\infty)$, see \citet[Lemma 1]{SchumannTripathi2018}.
Then, since for $s>0$ we have that $\Delta/2+s > |\Delta/2-s|$ as $\Delta >0$, it follows that $\ell'(s) = g'(\Delta/2+s) - g'(|\Delta/2-s|) > 0$ (where $g'(|\Delta/2-s|) = g'(\Delta/2-s)$ since $g'$ is even). Similarly $\ell'(s)<0$ for $s<0$, which concludes the argument.  

    \begin{lemma}\label{lem:bounds}
        For every $\delta >0$,
        $$ \pi_{1|1}(\phi_\delta)\, \pi_{0|0}(\phi_\delta) \geq \big(\max(\delta, \delta^{-1}) - 1\big)\,\big(1-\pi_{1|1}(\phi_\delta)\big)\, \big(1-\pi_{0|0}(\phi_\delta)\big)  .$$
        Moreover, $\pi_{1|1}(\phi_\delta)\, \pi_{0|0}(\phi_\delta) + \big(1-\pi_{1|1}(\phi_\delta)\big)\, \big(1-\pi_{0|0}(\phi_\delta)\big) >0$ (or equivalently $(\pi_{0|0}(\phi_\delta), \pi_{1|1}(\phi_\delta)) \not\in \{(1,0),(0,1)\}$) if and only if $\pi_{1|1}(\phi_\delta)>0$ in case $\delta >1$, and if and only if $\pi_{0|0}(\phi_\delta)>0$ in case $\delta <1$.
    \end{lemma}
    \begin{proof}[Proof of Lemma \ref{lem:bounds}]
    Suppose $\delta \geq 1$, so $\delta = \max(\delta, \delta^{-1})$. We have
    \begin{align*}
        \pi_{1|1}(\phi_\delta) = \rP_1\big(f_1 / f_0 >\delta \big) = \int_{\{f_1 / f_0 >\delta\}}\, f_1/f_0\, \dd \rP_0 \geq \, \delta\,\big(1 - \pi_{0|0}(\phi_\delta) \big). 
    \end{align*}
    Since $\pi_{1|1}(\phi_\delta) \leq 1$, this also implies $\pi_{0|0}(\phi_\delta) \geq 1- \delta^{-1}$, and the inequality follows since $1-\pi_{1|1}(\phi_\delta) \leq 1$. The case $\delta \leq 1$ is analogous.%

    For the second statement, from the inequality, for $\delta \neq 1$ we have that $\big(1-\pi_{1|1}(\phi_\delta)\big)\, \big(1-\pi_{0|0}(\phi_\delta)\big)>0$ already implies that $\pi_{1|1}(\phi_\delta)\, \pi_{0|0}(\phi_\delta) >0$. From the above proof, if $\delta>1$ it automatically follows that $\pi_{0|0}(\phi_\delta)>0$, so it remains that $\pi_{1|1}(\phi_\delta)>0$. $\delta <1$ is again similar. 
    \end{proof}    
    From  Lemma \ref{lem:bounds} and the definitions of $M_{\text{Q}}$ and $M_{\text{Y}}$ it follows that if $\pi_{1|1}(\phi_\delta)>0$ for $\delta>1$ or if $\pi_{0|0}(\phi_\delta)>0$ for $\delta <1$, then
    \begin{align}\label{eq:inequyule}
    \begin{split}
    M_{\text{Q}}\big(\pi_{1|1}(\phi_\delta), \pi_{0|0}(\phi_\delta), \pi\big) & \geq 1- \frac{2}{\max(\delta, \delta^{-1})},\\ M_{\text{Y}}\big(\pi_{1|1}(\phi_\delta), \pi_{0|0}(\phi_\delta), \pi\big) & \geq 1- \frac{2}{1+\sqrt{\max(\delta, \delta^{-1}) -1}}.
    \end{split}
    \end{align}

\end{example}}

\subsection[Robust performance metrics: existence of Bayes classifiers]{{Robust performance metrics: existence of Bayes classifiers}}\label{app:robust}

{Here we show the existence of Bayes classifiers for the robust performance metrics, as announced at the end of Section~\ref{ex:fgen}.}

{\begin{lemma}[Existence for robust MCC and robust Cohen's $\kappa$]\label{lem:exist-mcc-kappa}
Let $\pi \in (0,1)$, and let $\rP_0 \neq \rP_1$ be mutually absolutely continuous. %
 Then for the robust MCC, $M = M_{\text{MCC}_{\text{rb}}}$ as well as for the robust Cohen's $\kappa$, $M = M_{\kappa_{\mathrm{rb}}}$ with $d > 0$, there exists a randomized Bayes classifier $\phi^*$, that is, a maximizer of $M\big(\pi_{1|1}(\phi), \pi_{0|0}(\phi), \pi\big)$ over $\phi \in \{ \varphi: \cX \to [0,1] \text{ measurable}\}$. Furthermore, every Bayes classifier for these metrics satisfies $\pi_{1|1}(\phi^*), \pi_{0|0}(\phi^*) \in (0,1)$.
\end{lemma}
In particular, under these assumptions, Theorem~\ref{th:decisiobnounddens} applies, so that every Bayes classifier is of the form $\phi_{\delta^*}$, with $\delta^* \leq \crr$ for $\pi \in (0,1/2]$. For the robust F-score see Lemma~\ref{lem:exist-frb} below. 

\begin{proof}
\emph{Existence.} Set $\mu = \rP_0 + \rP_1$. We may identify randomized classifiers in $\cF$ that are equal $\mu$-almost everywhere since these then have the same parameters $\pi_{1|1}(\phi)$ and $\pi_{0|0}(\phi)$. Thus we may consider $\cF$ as a subset of $L^\infty(\mu)$, which is the dual space of $L^1(\mu)$: $L^\infty(\mu) = L^1(\mu)^*$. The Banach--Alaoglu theorem then provides that the unit ball in $L^\infty(\mu)$ is weak-$*$ compact, with the weak-$*$ topology being the coarsest topology that makes all functionals $L^\infty(\mu) \to \R$, $\phi \mapsto \int \phi\, g \, \dd \mu$ continuous, where $g \in L^1(\mu)$. Since $\cF$ is a  weak-$*$ closed subset of the unit ball of $L^\infty(\mu)$ it is also compact. In the situation of a countably generated $\sigma$-field of $\cX$,  \citet[Theorem~A.5.1]{Lehmann:2005} give a  direct argument for sequence compactness of $\cF$, which is then equivalent to compactness, together with a hint towards this more general argument via the Banach--Alaoglu theorem.

Now, since $f_0, f_1 \in L^1(\mu)$, the maps $\phi \mapsto \pi_{1|1}(\phi) = \int \phi\, f_1\, \dd\mu$ and $\phi \mapsto \pi_{0|0}(\phi) = \int (1-\phi)\, f_0\, \dd\mu$ are weak-$*$ continuous. For fixed $\pi$, both metrics are continuous functions of $(\pi_{1|1}, \pi_{0|0}) \in [0,1]^2$, since their denominators are bounded below by a positive constant: by $d^{1/2}$ for $M_{\text{MCC}_{\text{rb}}}$, and by $d$ for $M_{\kappa_{\mathrm{rb}}}$, as $D \geq 0$ in \eqref{eq:convcomb}. %
Hence $\phi \mapsto M\big(\pi_{1|1}(\phi), \pi_{0|0}(\phi), \pi\big)$ is weak-$*$ continuous on the compact set $\cF$ and attains its maximum, so that a Bayes classifier exists.

\smallskip

\emph{Exclusion of rates on the boundary.} If $\pi_{1|1}(\phi) = 1$, then $\phi = 1$ holds $\rP_1$-almost surely, hence $\rP_0$-almost surely by mutual absolute continuity, so that $\pi_{0|0}(\phi) = 0$. In the same way, $\pi_{1|1}(\phi) = 0$ implies $\pi_{0|0}(\phi) = 1$, and $\pi_{0|0}(\phi) \in \{0,1\}$ implies $\pi_{1|1}(\phi) = 1 - \pi_{0|0}(\phi)$. Thus a conditional rate in $\{0,1\}$ occurs only for the trivial classifiers $\phi \equiv 0$ and $\phi \equiv 1$, with $(\pi_{1|1}, \pi_{0|0}) = (0,1)$ and $(1,0)$, respectively. It remains to show that these classifiers do not maximize $M$.

 Since $\pi_{1|1}\pi_{0|0} - (1-\pi_{1|1})(1-\pi_{0|0}) = \pi_{1|1} + \pi_{0|0} - 1 = J$ and the remaining factors of both metrics are positive, $M$ has the sign of Youden's $J$, and $M = 0$ for the two trivial classifiers. For $\phi = \ind(f_1 > f_0)$, we have $J = \rP_1(f_1 > f_0) - \rP_0(f_1 > f_0) > 0$, which equals the total variation distance of $\rP_0 \neq \rP_1$. Hence the maximum of $M$ is positive and therefore cannot be attained by the trivial classifiers.

\end{proof}
}

{Finally we turn to existence of the Bayes classifier for the robust F-score. We focus on the parameter setting with $c=0$; a more technical statement and proof also hold for $c >0$.

\begin{lemma}[Existence for robust F-score]\label{lem:exist-frb}
Again suppose that  $\pi \in (0,1)$, and let $\rP_0 \neq \rP_1$ be mutually absolutely continuous. 
Consider the robust F-score $M_{\mathrm{F}_{\mathrm{rb}}}$ from Section \ref{ex:fgen} with $c=0$, $ d_1 \geq 0$ and $d_0 > 0$.
Then a Bayes classifier for $M_{\mathrm{F}_{\mathrm{rb}}}$ exists, and every Bayes classifier satisfies $\pi_{1|1}(\phi^*), \pi_{0|0}(\phi^*) \in (0,1)$, provided that
\begin{equation}\label{eq:upperbound}
\rP_0\Big(\frac{f_1}{f_0} < \frac{(1-\pi)}{d_0 + d_1\pi + (1-\pi)}\Big) > 0 .
\end{equation}
\end{lemma}

\begin{proof}
Set $M = M_{\mathrm{F}_{\mathrm{rb}}}$. Again we only need to exclude the trivial classifiers $\phi=1$ and $\phi=0$ as candidates for the Bayes classifier. 

We start with $\phi=1$, for which $(\pi_{1|1}, \pi_{0|0}) = (1,0)$. Now the quotient of partial derivatives from the proof of Theorem \ref{th:robustFscore} reduces for $c=0$ to 
\[
\frac{\partial_{\pi_{0|0}} M(\pi_{1|1},\pi_{0|0},\pi)}{\partial_{\pi_{1|1}} M(\pi_{1|1},\pi_{0|0},\pi)} = \frac{(1-\pi)\pi_{1|1}}{d_0 + d_1\pi + (1-\pi)(1-\pi_{0|0})},
\]
and $\partial_{\pi_{1|1}} M$ is a positive multiple of the denominator on the right, which is $\geq d_0 >0$, hence 
$\partial_{\pi_{1|1}} M >0$. Now at $(\pi_{1|1}, \pi_{0|0}) = (1,0)$ this quotient equals the upper bound in \eqref{eq:upperbound}. 
Choose $\lambda$ smaller than this bound such that $B = \{f_1/f_0 \leq \lambda\}$ still satisfies $\rP_0(B) > 0$. We let $\phi_\varepsilon = 1 - \varepsilon\, \ind_B$ for $\varepsilon \in (0,1]$, and show that the metric $M$ is strictly increasing in $\phi_\varepsilon$ for small $\varepsilon$, thus excluding $\phi=1 = \phi_0$ as a candidate for a Bayes classifier. To this end, observe that 
$\pi_{1|1}(\phi_\varepsilon) = 1 - \varepsilon\, \rP_1(B)$, $\pi_{0|0}(\phi_\varepsilon) = \varepsilon\, \rP_0(B)$, as well as $\rP_1(B) = \int_B (f_1/f_0)\, \dd\rP_0 \leq \lambda\, \rP_0(B)$. Then 
\begin{align*}
\frac{\dd}{\dd\varepsilon} M\big(\pi_{1|1}(\phi_\varepsilon), \pi_{0|0}(\phi_\varepsilon), \pi\big)\Big|_{\varepsilon = 0} &= -\partial_{\pi_{1|1}} M(1,0,\pi)\, \rP_1(B) + \partial_{\pi_{0|0}} M(1,0,\pi)\, \rP_0(B) \\
&\geq \partial_{\pi_{1|1}} M(1,0,\pi)\, \rP_0(B) \Big(\frac{(1-\pi)}{d_0 + d_1\pi + (1-\pi) } - \lambda\Big) > 0,
\end{align*}
by choice of $\lambda$, where we used $\partial_{\pi_{1|1}} M >0$, $\rP_1(B)  \leq \lambda\, \rP_0(B)$ and the expression for the quotient $\partial_{\pi_{0|0}} M(1,0,\pi) / \partial_{\pi_{1|1}} M(1,0,\pi)$ in the inequality. 

This concludes the argument for $\phi=1$. For $\phi=0$ with  $(\pi_{1|1}, \pi_{0|0}) = (0,1)$ the robust F-score is $=0$ if $c=0$. But by the first part there exist classifiers for which the robust F-score is positive. This concludes the argument.  

\end{proof}
}

{Finally, the following inequality is used in the proof of the robustness of the robust MCC in Section~\ref{ex:mccgen}.}

\begin{lemma}\label{lem:covinequbinary}
    Let $Y$ and $Z$ be two binary random variables. Then $|\text{Cov}(Y,Z)| \leq \min\big(\text{Var}(Y),\text{Var}(Z)\big) $. 
\end{lemma}
\begin{proof}
    Let $p=\PP(Y=1)$, $q= \PP(Z=1)$, $r=\PP(Y=1 | Z=1)$, $s=\PP(Y=1 | Z=0)$. For the first inequality it suffices to show that $|\text{Cov}(Y,Z)| \leq \text{Var}(Z)$, that is $|r\, q - p\, q| \leq q\, (1-q)$. Now, $p = r\,q + s\, (1-q)$, so $r\, q - p\, q = r\, q - q \, (r\,q + s\, (1-q)) = q\, (1-q)\, (r-s)$. Since $|r-s| \leq 1$ the result follows.  
\end{proof}

\renewcommand\textfraction{0.01}
\renewcommand{\floatpagefraction}{0.99}
\setcounter{totalnumber}{5}
\setlength{\tabcolsep}{3pt}

\clearpage
\section{Additional tables for Example \ref{ex:lda}} \label{app:ex:lda}

\begin{table}[ht]
\scriptsize 
\centering
\begin{tabular}{rr|rrrrrrrrrrrrrrrr} \hline
  & & \multicolumn{16}{c}{$\pi=\PP(Y=1)$}   \\
$\Delta$ & & $0.0001$ && $0.001$ && $0.01$ && $0.1$ && $0.9$ && $0.99$ && $0.999$ && $0.9999$  \\ 
  \hline
   & JAC & 4.41 &  & 3.16 &  & 2.11 &  & 1.19 &  & 0.10 &  & 0.01 &  & 0.00 &  & 0.00 &  \\ 
   &  & 0.00 & 1.00 & 0.02 & 0.99 & 0.11 & 0.96 & 0.46 & 0.73 & 1.00 & 0.00 & 1.00 & 0.00 & 1.00 & 0.00 & 1.00 & 0.00 \\ 
   & $F_{1.5}$ & 3.94 &  & 2.77 &  & 1.78 &  & 0.92 &  & 0.05 &  & 0.00 &  & 0.00 &  & 0.00 &  \\ 
0.5&  & 0.01 & 1.00 & 0.04 & 0.99 & 0.18 & 0.92 & 0.66 & 0.54 & 1.00 & 0.00 & 1.00 & 0.00 & 1.00 & 0.00 & 1.00 & 0.00 \\ 
   & $F_{0.5}$ & 5.26 &  & 3.88 &  & 2.72 &  & 1.70 &  & 0.31 &  & 0.04 &  & 0.00 &  & 0.00 &  \\ 
   &  & 0.00 & 1.00 & 0.01 & 1.00 & 0.04 & 0.99 & 0.21 & 0.91 & 1.00 & 0.02 & 1.00 & 0.00 & 1.00 & 0.00 & 1.00 & 0.00 \\ 
   & $\kappa$ & 4.55 &  & 3.34 &  & 2.34 &  & 1.53 &  & 0.65 &  & 0.43 &  & 0.30 &  & 0.22 &  \\ 
   &  & 0.00 & 1.00 & 0.02 & 1.00 & 0.07 & 0.97 & 0.27 & 0.86 & 0.86 & 0.27 & 0.97 & 0.07 & 1.00 & 0.02 & 1.00 & 0.00 \\ 
  \hline
   & JAC & 18.80 &  & 10.10 &  & 4.85 &  & 1.87 &  & 0.10 &  & 0.01 &  & 0.00 &  & 0.00 &  \\ 
   &  & 0.01 & 1.00 & 0.03 & 1.00 & 0.14 & 0.98 & 0.45 & 0.87 & 1.00 & 0.04 & 1.00 & 0.00 & 1.00 & 0.00 & 1.00 & 0.00 \\ 
   & $F_{1.5}$ & 15.30 &  & 7.93 &  & 3.61 &  & 1.26 &  & 0.05 &  & 0.00 &  & 0.00 &  & 0.00 &  \\ 
1  &  & 0.01 & 1.00 & 0.06 & 0.99 & 0.22 & 0.96 & 0.61 & 0.77 & 1.00 & 0.01 & 1.00 & 0.00 & 1.00 & 0.00 & 1.00 & 0.00 \\ 
   & $F_{0.5}$ & 26.30 &  & 14.80 &  & 7.66 &  & 3.34 &  & 0.32 &  & 0.04 &  & 0.00 &  & 0.00 &  \\ 
   &  & 0.00 & 1.00 & 0.01 & 1.00 & 0.06 & 0.99 & 0.24 & 0.96 & 0.95 & 0.26 & 1.00 & 0.00 & 1.00 & 0.00 & 1.00 & 0.00 \\ 
   & $\kappa$ & 19.00 &  & 10.40 &  & 5.21 &  & 2.29 &  & 0.44 &  & 0.19 &  & 0.10 &  & 0.05 &  \\ 
   &  & 0.01 & 1.00 & 0.03 & 1.00 & 0.12 & 0.98 & 0.37 & 0.91 & 0.91 & 0.37 & 0.98 & 0.13 & 1.00 & 0.03 & 1.00 & 0.01 \\ 
  \hline
   & JAC & 215.00 &  & 68.90 &  & 18.90 &  & 3.97 &  & 0.10 &  & 0.01 &  & 0.00 &  & 0.00 &  \\ 
   &  & 0.05 & 1.00 & 0.13 & 1.00 & 0.32 & 0.99 & 0.62 & 0.95 & 0.98 & 0.45 & 1.00 & 0.10 & 1.00 & 0.01 & 1.00 & 0.00 \\ 
   & $F_{1.5}$ & 146.00 &  & 44.70 &  & 11.50 &  & 2.19 &  & 0.05 &  & 0.00 &  & 0.00 &  & 0.00 &  \\ 
2  &  & 0.07 & 1.00 & 0.18 & 1.00 & 0.41 & 0.99 & 0.73 & 0.92 & 0.99 & 0.30 & 1.00 & 0.04 & 1.00 & 0.00 & 1.00 & 0.00 \\ 
   & $F_{0.5}$ & 401.00 &  & 139.00 &  & 42.00 &  & 10.10 &  & 0.36 &  & 0.04 &  & 0.00 &  & 0.00 &  \\ 
   &  & 0.02 & 1.00 & 0.07 & 1.00 & 0.19 & 1.00 & 0.44 & 0.98 & 0.93 & 0.69 & 1.00 & 0.27 & 1.00 & 0.04 & 1.00 & 0.00 \\ 
   & $\kappa$ & 215.00 &  & 69.40 &  & 19.40 &  & 4.38 &  & 0.23 &  & 0.05 &  & 0.01 &  & 0.00 &  \\ 
   &  & 0.05 & 1.00 & 0.13 & 1.00 & 0.31 & 0.99 & 0.60 & 0.96 & 0.96 & 0.60 & 0.99 & 0.31 & 1.00 & 0.13 & 1.00 & 0.05 \\ 
  \hline
   & JAC & 3820.00 &  & 565.00 &  & 73.70 &  & 7.95 &  & 0.11 &  & 0.01 &  & 0.00 &  & 0.00 &  \\ 
   &  & 0.48 & 1.00 & 0.66 & 1.00 & 0.82 & 1.00 & 0.93 & 0.99 & 0.99 & 0.93 & 1.00 & 0.80 & 1.00 & 0.61 & 1.00 & 0.38 \\ 
   & $F_{1.5}$ & 1980.00 &  & 280.00 &  & 35.10 &  & 3.67 &  & 0.05 &  & 0.00 &  & 0.00 &  & 0.00 &  \\ 
4  &  & 0.54 & 1.00 & 0.72 & 1.00 & 0.87 & 1.00 & 0.95 & 0.99 & 1.00 & 0.89 & 1.00 & 0.74 & 1.00 & 0.53 & 1.00 & 0.33 \\ 
   & $F_{0.5}$ & 11400.00 &  & 1810.00 &  & 254.00 &  & 29.20 &  & 0.43 &  & 0.04 &  & 0.00 &  & 0.00 &  \\ 
   &  & 0.37 & 1.00 & 0.55 & 1.00 & 0.73 & 1.00 & 0.88 & 1.00 & 0.99 & 0.96 & 1.00 & 0.88 & 1.00 & 0.73 & 1.00 & 0.52 \\ 
   & $\kappa$ & 3820.00 &  & 565.00 &  & 74.00 &  & 8.07 &  & 0.12 &  & 0.01 &  & 0.00 &  & 0.00 &  \\ 
   &  & 0.48 & 1.00 & 0.66 & 1.00 & 0.82 & 1.00 & 0.93 & 0.99 & 0.99 & 0.93 & 1.00 & 0.82 & 1.00 & 0.66 & 1.00 & 0.48 \\ 
  \hline
\end{tabular}
\caption{Lines 1,3,\ldots: Optimal threshold $\delta^*$ for standard metrics and  several values of $\Delta$. Lines 2,4,\ldots: Corresponding conditional probabilities $\pi_{1|1}$ and $\pi_{0|0}$}\label{tab:LDAspesen}
\end{table}

\pagebreak
\begin{table}[ht]
\small 
\centering
\begin{tabular}{c|rrrrrrrrrr}
  \hline
    & & \multicolumn{8}{c}{$\pi=\PP(Y=1)$}   \\
$\Delta$ & $10^{-10}$ & $10^{-9}$ & $10^{-8}$ & $10^{-7}$ & $10^{-6}$ & $10^{-5}$ & $10^{-4}$ & $10^{-3}$ & $0.01$ & $0.1$ \\ 
  \hline
  1.00 & 237.21 & 162.92 & 109.27 & 71.22 & 44.80 & 26.94 & 15.25 & 7.93 & 3.61 & 1.26 \\ 
  1.50 & 2825.82 & 1622.27 & 900.43 & 480.03 & 243.61 & 116.22 & 51.16 & 20.19 & 6.79 & 1.70 \\ 
  2.00 & 27356.75 & 13175.23 & 6079.02 & 2665.88 & 1099.51 & 420.39 & 146.00 & 44.70 & 11.49 & 2.19 \\ 
  3.00 & 1354744.38 & 464321.86 & 150125.64 & 45357.40 & 12653.67 & 3210.71 & 727.02 & 143.52 & 23.96 & 3.10 \\ 
  \hline
\end{tabular}
\caption{Limiting behavior of the optimal threshold $\delta^\ast$ of $F_{1.5}$ for $\pi\to 0$ and several values of $\Delta$. \label{lim-F15}}
\end{table}
\bigskip
\setlength{\tabcolsep}{2pt}
\begin{table}[ht]
\small 
\centering
\begin{tabular}{c|rrrrrrrrrr}
  \hline
    & & \multicolumn{8}{c}{$\pi=\PP(Y=1)$}   \\
$\Delta$ & $10^{-10}$ & $10^{-9}$ & $10^{-8}$ & $10^{-7}$ & $10^{-6}$ & $10^{-5}$ & $10^{-4}$ & $10^{-3}$ & $0.01$ & $0.1$ \\ 
  \hline
  1.00 & 333.17 & 233.27 & 160.06 & 107.22 & 69.78 & 43.82 & 26.29 & 14.83 & 7.66 & 3.35 \\ 
  1.50 & 4671.72 & 2756.78 & 1580.39 & 875.71 & 465.90 & 235.84 & 112.15 & 49.15 & 19.21 & 6.12 \\ 
  2.00 & 53081.39 & 26479.61 & 12729.65 & 5860.90 & 2563.68 & 1054.07 & 401.43 & 138.66 & 41.98 & 10.11 \\ 
  3.00 & 3589564.48 & 1291462.21 & 441531.99 & 142346.77 & 42862.63 & 11910.11 & 3007.53 & 676.76 & 131.88 & 20.34 \\ 
  \hline
\end{tabular}
\caption{Limiting behavior of the optimal threshold $\delta^\ast$ of $F_{0.5}$ for $\pi\to 0$ and several values of $\Delta$. \label{lim-F05}}
\end{table}
\bigskip
\setlength{\tabcolsep}{3pt}
\begin{table}[ht]
\small 
\centering
\begin{tabular}{c|rrrrrrrrrr}
  \hline
    & & \multicolumn{8}{c}{$\pi=\PP(Y=1)$}   \\
$\Delta$ & $10^{-10}$ & $10^{-9}$ & $10^{-8}$ & $10^{-7}$ & $10^{-6}$ & $10^{-5}$ & $10^{-4}$ & $10^{-3}$ & $0.01$ & $0.1$ \\ 
  \hline
  1.00 & 269.58 & 186.60 & 126.33 & 83.30 & 53.19 & 32.65 & 19.05 & 10.40 & 5.21 & 2.29 \\ 
  1.50 & 3411.85 & 1979.74 & 1112.83 & 602.41 & 311.60 & 152.37 & 69.40 & 28.84 & 10.62 & 3.26 \\ 
  2.00 & 35067.15 & 17119.42 & 8025.62 & 3587.06 & 1514.13 & 596.00 & 215.02 & 69.43 & 19.43 & 4.38 \\ 
  3.00 & 1950933.76 & 681242.81 & 225050.90 & 69722.95 & 20036.29 & 5267.30 & 1244.92 & 259.00 & 46.22 & 6.59 \\
  \hline
\end{tabular}
\caption{Limiting behavior of the optimal threshold $\delta^\ast$ of Kappa for $\pi\to 0$ and several values of $\Delta$. \label{lim-Kap}}
\end{table}
\bigskip
\begin{table}[ht]
\small 
\centering
\begin{tabular}{c|rrrrrrrrrr}
  \hline
    & & \multicolumn{8}{c}{$\pi=\PP(Y=1)$}   \\
$\Delta$ & $10^{-10}$ & $10^{-9}$ & $10^{-8}$ & $10^{-7}$ & $10^{-6}$ & $10^{-5}$ & $10^{-4}$ & $10^{-3}$ & $0.01$ & $0.1$ \\ 
  \hline
  1 & 2.58 & 2.58 & 2.58 & 2.58 & 2.58 & 2.58 & 2.57 & 2.56 & 2.46 & 1.85 \\ 
1.5 & 10.51 & 10.51 & 10.51 & 10.51 & 10.51 & 10.51 & 10.46 & 10.01 & 7.46 & 3.08 \\ 
  2 & 110.26 & 110.26 & 110.26 & 110.25 & 110.15 & 109.13 & 100.38 & 62.89 & 20.92 & 4.55 \\ 
  3 & 178272.20 & 177701.27 & 172284.45 & 136127.76 & 58364.38 & 14096.18 & 2541.43 & 394.69 & 55.79 & 6.93 \\ 
   \hline
\end{tabular}
\caption{Limiting behavior of the optimal threshold $\delta^\ast$ of MCC for $\pi\to 0$ and several values of $\Delta$. \label{lim-MCC}}
\end{table}

\clearpage
\section{Additional tables for Subsection \ref{ex:fgen}} \label{tab:ex:fgen}

\begin{table}[ht]
\scriptsize 
\centering
\begin{tabular}{rr|rrrrrrrrrrrrrrrr}
  \hline
    & & \multicolumn{16}{c}{$\pi=\PP(Y=1)$}   \\
$\Delta$ & $d_0$ & $0.0001$ && $0.001$ && $0.01$ && $0.1$ && $0.9$ && $0.99$ && $0.999$ && $0.9999$  \\ 
\hline
   & 0.00 & 4.41 &  & 3.16 &  & 2.11 &  & 1.19 &  & 0.10 &  & 0.01 &  & 0.00 &  & 0.00 &  \\ 
   &  & 0.00 & 1.00 & 0.02 & 0.99 & 0.11 & 0.96 & 0.46 & 0.73 & 1.00 & 0.00 & 1.00 & 0.00 & 1.00 & 0.00 & 1.00 & 0.00 \\ 
   & 0.10 & 1.23 &  & 1.23 &  & 1.19 &  & 0.96 &  & 0.09 &  & 0.01 &  & 0.00 &  & 0.00 &  \\ 
0.5&  & 0.43 & 0.75 & 0.43 & 0.75 & 0.46 & 0.73 & 0.63 & 0.57 & 1.00 & 0.00 & 1.00 & 0.00 & 1.00 & 0.00 & 1.00 & 0.00 \\ 
   & 0.40 & 0.77 &  & 0.77 &  & 0.76 &  & 0.68 &  & 0.07 &  & 0.01 &  & 0.00 &  & 0.00 &  \\ 
   &  & 0.78 & 0.40 & 0.78 & 0.40 & 0.78 & 0.39 & 0.85 & 0.30 & 1.00 & 0.00 & 1.00 & 0.00 & 1.00 & 0.00 & 1.00 & 0.00 \\ 
   & 1.00 & 0.51 &  & 0.51 &  & 0.50 &  & 0.45 &  & 0.05 &  & 0.00 &  & 0.00 &  & 0.00 &  \\ 
   &  & 0.95 & 0.13 & 0.95 & 0.13 & 0.95 & 0.13 & 0.97 & 0.09 & 1.00 & 0.00 & 1.00 & 0.00 & 1.00 & 0.00 & 1.00 & 0.00 \\ 
   \hline
   & 0.00 & 18.80 &  & 10.10 &  & 4.85 &  & 1.87 &  & 0.10 &  & 0.01 &  & 0.00 &  & 0.00 &  \\ 
   &  & 0.01 & 1.00 & 0.03 & 1.00 & 0.14 & 0.98 & 0.45 & 0.87 & 1.00 & 0.04 & 1.00 & 0.00 & 1.00 & 0.00 & 1.00 & 0.00 \\ 
   & 0.10 & 1.96 &  & 1.95 &  & 1.87 &  & 1.33 &  & 0.09 &  & 0.01 &  & 0.00 &  & 0.00 &  \\ 
1  &  & 0.43 & 0.88 & 0.43 & 0.88 & 0.45 & 0.87 & 0.59 & 0.78 & 1.00 & 0.03 & 1.00 & 0.00 & 1.00 & 0.00 & 1.00 & 0.00 \\ 
   & 0.40 & 0.98 &  & 0.97 &  & 0.96 &  & 0.81 &  & 0.07 &  & 0.01 &  & 0.00 &  & 0.00 &  \\ 
   &  & 0.70 & 0.68 & 0.70 & 0.68 & 0.71 & 0.68 & 0.76 & 0.61 & 1.00 & 0.02 & 1.00 & 0.00 & 1.00 & 0.00 & 1.00 & 0.00 \\ 
   & 1.00 & 0.56 &  & 0.56 &  & 0.55 &  & 0.49 &  & 0.05 &  & 0.00 &  & 0.00 &  & 0.00 &  \\ 
   &  & 0.86 & 0.47 & 0.86 & 0.47 & 0.86 & 0.46 & 0.89 & 0.42 & 1.00 & 0.01 & 1.00 & 0.00 & 1.00 & 0.00 & 1.00 & 0.00 \\ 
   \hline
   & 0.00 & 215.00 &  & 68.90 &  & 18.90 &  & 3.97 &  & 0.10 &  & 0.01 &  & 0.00 &  & 0.00 &  \\ 
   &  & 0.05 & 1.00 & 0.13 & 1.00 & 0.32 & 0.99 & 0.62 & 0.95 & 0.98 & 0.45 & 1.00 & 0.10 & 1.00 & 0.01 & 1.00 & 0.00 \\ 
   & 0.10 & 4.27 &  & 4.24 &  & 3.97 &  & 2.40 &  & 0.09 &  & 0.01 &  & 0.00 &  & 0.00 &  \\ 
2  &  & 0.61 & 0.96 & 0.61 & 0.96 & 0.62 & 0.95 & 0.71 & 0.92 & 0.99 & 0.43 & 1.00 & 0.09 & 1.00 & 0.01 & 1.00 & 0.00 \\ 
   & 0.40 & 1.53 &  & 1.53 &  & 1.49 &  & 1.18 &  & 0.07 &  & 0.01 &  & 0.00 &  & 0.00 &  \\ 
   &  & 0.78 & 0.89 & 0.78 & 0.89 & 0.79 & 0.88 & 0.82 & 0.86 & 0.99 & 0.38 & 1.00 & 0.07 & 1.00 & 0.00 & 1.00 & 0.00 \\ 
   & 1.00 & 0.73 &  & 0.73 &  & 0.72 &  & 0.62 &  & 0.05 &  & 0.00 &  & 0.00 &  & 0.00 &  \\ 
   &  & 0.88 & 0.80 & 0.88 & 0.80 & 0.88 & 0.80 & 0.89 & 0.78 & 0.99 & 0.31 & 1.00 & 0.05 & 1.00 & 0.00 & 1.00 & 0.00 \\ 
   \hline
   & 0.00 & 1240.00 &  & 258.00 &  & 45.80 &  & 6.33 &  & 0.11 &  & 0.01 &  & 0.00 &  & 0.00 &  \\ 
   &  & 0.19 & 1.00 & 0.36 & 1.00 & 0.59 & 1.00 & 0.81 & 0.98 & 0.99 & 0.77 & 1.00 & 0.49 & 1.00 & 0.21 & 1.00 & 0.06 \\ 
   & 0.10 & 6.92 &  & 6.86 &  & 6.33 &  & 3.44 &  & 0.10 &  & 0.01 &  & 0.00 &  & 0.00 &  \\ 
3  &  & 0.80 & 0.98 & 0.80 & 0.98 & 0.81 & 0.98 & 0.86 & 0.97 & 0.99 & 0.76 & 1.00 & 0.47 & 1.00 & 0.20 & 1.00 & 0.05 \\ 
   & 0.40 & 2.03 &  & 2.03 &  & 1.97 &  & 1.51 &  & 0.07 &  & 0.01 &  & 0.00 &  & 0.00 &  \\ 
   &  & 0.90 & 0.96 & 0.90 & 0.96 & 0.90 & 0.96 & 0.91 & 0.95 & 0.99 & 0.74 & 1.00 & 0.44 & 1.00 & 0.18 & 1.00 & 0.05 \\ 
   & 1.00 & 0.88 &  & 0.87 &  & 0.86 &  & 0.72 &  & 0.05 &  & 0.00 &  & 0.00 &  & 0.00 &  \\ 
   &  & 0.94 & 0.93 & 0.94 & 0.93 & 0.94 & 0.93 & 0.95 & 0.92 & 0.99 & 0.70 & 1.00 & 0.40 & 1.00 & 0.15 & 1.00 & 0.03 \\ 
   \hline
\end{tabular}
\caption{Lines 1,3,\ldots: Optimal threshold $\delta^*$ for $F_{\text{rb}}$ with varying values of $d_0$ and $\Delta$; $c=0$ and $d_1=1$ fixed. Lines 2,4,\ldots: Corresponding conditional probabilities $\pi_{1|1}$ and $\pi_{0|0}$.} \label{tab1:ex:fgen}
\end{table}
\begin{table}
\scriptsize 
\centering
\begin{tabular}{rr|rrrrrrrrrrrrrrrr}
  \hline
    & & \multicolumn{16}{c}{$\pi=\PP(Y=1)$}   \\
$\Delta$ & $d_1$ & $0.0001$ && $0.001$ && $0.01$ && $0.1$ && $0.9$ && $0.99$ && $0.999$ && $0.9999$  \\ 
\hline
   & 0.25 & 1.23 &  & 1.23 &  & 1.22 &  & 1.12 &  & 0.23 &  & 0.03 &  & 0.00 &  & 0.00 &  \\ 
   &  & 0.43 & 0.75 & 0.43 & 0.75 & 0.44 & 0.74 & 0.51 & 0.68 & 1.00 & 0.00 & 1.00 & 0.00 & 1.00 & 0.00 & 1.00 & 0.00 \\ 
   & 0.50 & 1.23 &  & 1.23 &  & 1.21 &  & 1.06 &  & 0.15 &  & 0.02 &  & 0.00 &  & 0.00 &  \\ 
0.5&  & 0.43 & 0.75 & 0.43 & 0.75 & 0.45 & 0.74 & 0.55 & 0.64 & 1.00 & 0.00 & 1.00 & 0.00 & 1.00 & 0.00 & 1.00 & 0.00 \\ 
   & 2.00 & 1.23 &  & 1.22 &  & 1.16 &  & 0.83 &  & 0.05 &  & 0.00 &  & 0.00 &  & 0.00 &  \\ 
   &  & 0.43 & 0.75 & 0.44 & 0.74 & 0.48 & 0.71 & 0.73 & 0.45 & 1.00 & 0.00 & 1.00 & 0.00 & 1.00 & 0.00 & 1.00 & 0.00 \\ 
   & 4.00 & 1.23 &  & 1.22 &  & 1.11 &  & 0.68 &  & 0.03 &  & 0.00 &  & 0.00 &  & 0.00 &  \\ 
   &  & 0.43 & 0.75 & 0.44 & 0.74 & 0.52 & 0.68 & 0.85 & 0.30 & 1.00 & 0.00 & 1.00 & 0.00 & 1.00 & 0.00 & 1.00 & 0.00 \\ 
   \hline
   & 0.25 & 1.96 &  & 1.95 &  & 1.93 &  & 1.68 &  & 0.24 &  & 0.03 &  & 0.00 &  & 0.00 &  \\ 
   &  & 0.43 & 0.88 & 0.43 & 0.88 & 0.44 & 0.88 & 0.49 & 0.85 & 0.97 & 0.18 & 1.00 & 0.00 & 1.00 & 0.00 & 1.00 & 0.00 \\ 
   & 0.50 & 1.96 &  & 1.95 &  & 1.91 &  & 1.54 &  & 0.15 &  & 0.02 &  & 0.00 &  & 0.00 &  \\ 
1  &  & 0.43 & 0.88 & 0.43 & 0.88 & 0.44 & 0.87 & 0.53 & 0.82 & 0.99 & 0.09 & 1.00 & 0.00 & 1.00 & 0.00 & 1.00 & 0.00 \\ 
   & 2.00 & 1.96 &  & 1.94 &  & 1.79 &  & 1.08 &  & 0.05 &  & 0.00 &  & 0.00 &  & 0.00 &  \\ 
   &  & 0.43 & 0.88 & 0.44 & 0.88 & 0.47 & 0.86 & 0.66 & 0.72 & 1.00 & 0.01 & 1.00 & 0.00 & 1.00 & 0.00 & 1.00 & 0.00 \\ 
   & 4.00 & 1.95 &  & 1.92 &  & 1.67 &  & 0.81 &  & 0.03 &  & 0.00 &  & 0.00 &  & 0.00 &  \\ 
   &  & 0.43 & 0.88 & 0.44 & 0.88 & 0.49 & 0.84 & 0.76 & 0.61 & 1.00 & 0.00 & 1.00 & 0.00 & 1.00 & 0.00 & 1.00 & 0.00 \\ 
   \hline
   & 0.25 & 4.28 &  & 4.27 &  & 4.17 &  & 3.38 &  & 0.26 &  & 0.03 &  & 0.00 &  & 0.00 &  \\ 
   &  & 0.61 & 0.96 & 0.61 & 0.96 & 0.61 & 0.96 & 0.65 & 0.95 & 0.95 & 0.63 & 1.00 & 0.22 & 1.00 & 0.03 & 1.00 & 0.00 \\ 
   & 0.50 & 4.27 &  & 4.26 &  & 4.10 &  & 2.96 &  & 0.16 &  & 0.02 &  & 0.00 &  & 0.00 &  \\ 
2  &  & 0.61 & 0.96 & 0.61 & 0.96 & 0.62 & 0.96 & 0.68 & 0.94 & 0.97 & 0.54 & 1.00 & 0.15 & 1.00 & 0.01 & 1.00 & 0.00 \\ 
   & 2.00 & 4.27 &  & 4.21 &  & 3.73 &  & 1.76 &  & 0.05 &  & 0.00 &  & 0.00 &  & 0.00 &  \\ 
   &  & 0.61 & 0.96 & 0.61 & 0.96 & 0.63 & 0.95 & 0.76 & 0.90 & 0.99 & 0.31 & 1.00 & 0.05 & 1.00 & 0.00 & 1.00 & 0.00 \\ 
   & 4.00 & 4.26 &  & 4.16 &  & 3.34 &  & 1.18 &  & 0.03 &  & 0.00 &  & 0.00 &  & 0.00 &  \\ 
   &  & 0.61 & 0.96 & 0.61 & 0.96 & 0.65 & 0.95 & 0.82 & 0.86 & 1.00 & 0.21 & 1.00 & 0.02 & 1.00 & 0.00 & 1.00 & 0.00 \\ 
   \hline
   & 0.25 & 6.93 &  & 6.91 &  & 6.73 &  & 5.21 &  & 0.29 &  & 0.03 &  & 0.00 &  & 0.00 &  \\ 
   &  & 0.80 & 0.98 & 0.80 & 0.98 & 0.81 & 0.98 & 0.83 & 0.98 & 0.97 & 0.86 & 1.00 & 0.62 & 1.00 & 0.33 & 1.00 & 0.11 \\ 
   & 0.50 & 6.93 &  & 6.89 &  & 6.59 &  & 4.44 &  & 0.17 &  & 0.02 &  & 0.00 &  & 0.00 &  \\ 
3  &  & 0.80 & 0.98 & 0.80 & 0.98 & 0.81 & 0.98 & 0.84 & 0.98 & 0.98 & 0.82 & 1.00 & 0.55 & 1.00 & 0.26 & 1.00 & 0.08 \\ 
   & 2.00 & 6.92 &  & 6.81 &  & 5.87 &  & 2.40 &  & 0.05 &  & 0.00 &  & 0.00 &  & 0.00 &  \\ 
   &  & 0.80 & 0.98 & 0.81 & 0.98 & 0.82 & 0.98 & 0.89 & 0.96 & 0.99 & 0.70 & 1.00 & 0.39 & 1.00 & 0.15 & 1.00 & 0.03 \\ 
   & 4.00 & 6.91 &  & 6.69 &  & 5.13 &  & 1.51 &  & 0.03 &  & 0.00 &  & 0.00 &  & 0.00 &  \\ 
   &  & 0.80 & 0.98 & 0.81 & 0.98 & 0.83 & 0.98 & 0.91 & 0.95 & 1.00 & 0.61 & 1.00 & 0.31 & 1.00 & 0.10 & 1.00 & 0.01 \\ 
   \hline
\end{tabular}
\caption{Lines 1,3,\ldots: Optimal threshold $\delta^*$ for $F_{\text{rb}}$ with varying values of $d_1$ and $\Delta$; $c=0$ and $d_0=0.1$ fixed. Lines 2,4,\ldots: Corresponding conditional probabilities $\pi_{1|1}$ and $\pi_{0|0}$.} \label{tab2:ex:fgen}
\end{table}
\begin{table}
\scriptsize 
\centering
\begin{tabular}{rr|rrrrrrrrrrrrrrrr}
  \hline
    & & \multicolumn{16}{c}{$\pi=\PP(Y=1)$}   \\
$\Delta$ & $d_0$ & $0.0001$ && $0.001$ && $0.01$ && $0.1$ && $0.9$ && $0.99$ && $0.999$ && $0.9999$  \\ 
\hline
   & 1.00 & 1.15 &  & 1.15 &  & 1.14 &  & 1.06 &  & 0.18 &  & 0.02 &  & 0.00 &  & 0.00 &  \\ 
   &  & 0.49 & 0.70 & 0.49 & 0.70 & 0.50 & 0.70 & 0.55 & 0.64 & 1.00 & 0.00 & 1.00 & 0.00 & 1.00 & 0.00 & 1.00 & 0.00 \\ 
   & 2.00 & 0.68 &  & 0.68 &  & 0.68 &  & 0.63 &  & 0.10 &  & 0.01 &  & 0.00 &  & 0.00 &  \\ 
0.5&  & 0.84 & 0.31 & 0.84 & 0.31 & 0.85 & 0.30 & 0.88 & 0.25 & 1.00 & 0.00 & 1.00 & 0.00 & 1.00 & 0.00 & 1.00 & 0.00 \\ 
   & 0.50 & 2.05 &  & 2.05 &  & 2.03 &  & 1.86 &  & 0.33 &  & 0.04 &  & 0.00 &  & 0.00 &  \\ 
   &  & 0.12 & 0.95 & 0.12 & 0.95 & 0.12 & 0.95 & 0.16 & 0.93 & 0.99 & 0.03 & 1.00 & 0.00 & 1.00 & 0.00 & 1.00 & 0.00 \\ 
   & 0.50 & 2.05 &  & 2.03 &  & 1.92 &  & 1.26 &  & 0.06 &  & 0.01 &  & 0.00 &  & 0.00 &  \\ 
   &  & 0.12 & 0.95 & 0.12 & 0.95 & 0.15 & 0.94 & 0.42 & 0.76 & 1.00 & 0.00 & 1.00 & 0.00 & 1.00 & 0.00 & 1.00 & 0.00 \\ 
   \hline
   & 1.00 & 1.30 &  & 1.30 &  & {1.29} &  & {1.20} &  & {0.18} &  & {0.02} &  & 0.00 &  & 0.00 &  \\ 
   &  & 0.59 & 0.78 & 0.59 & 0.78 & {0.60} & {0.78} & {0.63} & {0.75} & {0.99} & {0.11} & 1.00 & 0.00 & 1.00 & 0.00 & 1.00 & 0.00 \\ 
   & 2.00 & 0.74 &  & 0.74 &  & {0.73} &  & {0.68} &  & {0.10} &  & {0.01} &  & 0.00 &  & 0.00 &  \\ 
1  &  & 0.79 & 0.58 & 0.79 & 0.58 & 0.79 & 0.57 & {0.81} & {0.54} & 1.00 & {0.03} & 1.00 & 0.00 & 1.00 & 0.00 & 1.00 & 0.00 \\ 
   & 0.50 & {2.32} &  & {2.31} &  & {2.30} &  & {2.12} &  & {0.34} &  & {0.04} &  & 0.00 &  & 0.00 &  \\ 
   &  & 0.37 & 0.91 & 0.37 & 0.91 & {0.37} & {0.91} & {0.40} & {0.89} & {0.94} & {0.28} & 1.00 & 0.00 & 1.00 & 0.00 & 1.00 & 0.00 \\ 
   & 0.50 & 2.31 &  & 2.30 &  & 2.19 &  & 1.44 &  & 0.06 &  & 0.01 &  & 0.00 &  & 0.00 &  \\ 
   &  & 0.37 & 0.91 & 0.37 & 0.91 & 0.39 & 0.90 & 0.55 & 0.81 & 1.00 & 0.01 & 1.00 & 0.00 & 1.00 & 0.00 & 1.00 & 0.00 \\ 
   \hline
   & 1.00 & 1.60 &  & 1.60 &  & {1.59} &  & {1.46} &  & {0.19} &  & {0.02} &  & 0.00 &  & 0.00 &  \\ 
   &  & 0.78 & 0.89 & 0.78 & 0.89 & 0.78 & 0.89 & {0.79} & {0.88} & {0.97} & {0.57} & 1.00 & {0.17} & 1.00 & {0.02} & 1.00 & 0.00 \\ 
   & 2.00 & 0.85 &  & 0.85 &  & {0.85} &  & {0.77} &  & {0.10} &  & {0.01} &  & 0.00 &  & 0.00 &  \\ 
2  &  & 0.86 & 0.82 & 0.86 & 0.82 & 0.86 & 0.82 & {0.87} & {0.81} & {0.98} & {0.43} & 1.00 & {0.10} & 1.00 & {0.01} & 1.00 & 0.00 \\ 
   & 0.50 & {2.99} &  & {2.98} &  & {2.96} &  & {2.72} &  & {0.36} &  & {0.04} &  & 0.00 &  & 0.00 &  \\ 
   &  & {0.67} & 0.94 & {0.67} & 0.94 & {0.68} & 0.94 & {0.69} & {0.93} & {0.93} & {0.69} & 1.00 & {0.27} & 1.00 & {0.04} & 1.00 & 0.00 \\ 
   & 0.50 & 2.98 &  & 2.97 &  & 2.81 &  & 1.78 &  & 0.06 &  & 0.01 &  & 0.00 &  & 0.00 &  \\ 
   &  & 0.68 & 0.94 & 0.68 & 0.94 & 0.69 & 0.94 & 0.76 & 0.90 & 0.99 & 0.35 & 1.00 & 0.06 & 1.00 & 0.00 & 1.00 & 0.00 \\ 
   \hline
   & 1.00 & 1.82 &  & 1.82 &  & {1.80} &  & {1.65} &  & {0.19} &  & {0.02} &  & 0.00 &  & 0.00 &  \\ 
   &  & 0.90 & 0.96 & 0.90 & 0.96 & {0.90} & {0.96} & {0.91} & {0.95} & {0.98} & {0.83} & 1.00 & {0.58} & 1.00 & {0.28} & 1.00 & {0.09} \\ 
   & 2.00 & 0.94 &  & 0.93 &  & {0.93} &  & {0.84} &  & {0.10} &  & {0.01} &  & 0.00 &  & 0.00 &  \\ 
3  &  & 0.94 & 0.93 & 0.94 & 0.93 & 0.94 & 0.93 & {0.94} & {0.93} & 0.99 & {0.77} & 1.00 & {0.49} & 1.00 & {0.21} & 1.00 & {0.06} \\ 
   & 0.50 & {3.53} &  & {3.52} &  & {3.49} &  & {3.19} &  & {0.38} &  & {0.04} &  & 0.00 &  & 0.00 &  \\ 
   &  & 0.86 & 0.97 & 0.86 & 0.97 & 0.86 & 0.97 & {0.87} & {0.97} & {0.97} & {0.88} & 1.00 & {0.66} & 1.00 & {0.37} & 1.00 & {0.13} \\ 
   & 0.50 & 3.52 &  & 3.50 &  & 3.31 &  & 2.04 &  & 0.06 &  & 0.01 &  & 0.00 &  & 0.00 &  \\ 
   &  & 0.86 & 0.97 & 0.86 & 0.97 & 0.86 & 0.97 & 0.90 & 0.96 & 0.99 & 0.72 & 1.00 & 0.41 & 1.00 & 0.16 & 1.00 & 0.03 \\ 
   \hline
\end{tabular}
\caption{Lines 1,3,\ldots: Optimal threshold $\delta^*$ for $\text{F}_{\text{rb}}$ with varying values of $d_0$ and $\Delta$; $c=1$ and $d_1=1 (d_1=4 \text{ for last entry})$ fixed. Lines 2,4,\ldots: Corresponding conditional probabilities $\pi_{1|1}$ and $\pi_{0|0}$.} \label{tab3:ex:fgen}
\end{table}

\clearpage
\section{Additional tables for subsection \ref{ex:mccgen}} \label{tab:ex:mccgen}

\begin{table}[ht]
\scriptsize 
\centering
\begin{tabular}{rr|rrrrrrrrrrrrrrrr}
  \hline
  & & \multicolumn{16}{c}{$\pi=\PP(Y=1)$}   \\
$\Delta$ & $d$ & $0.0001$ && $0.001$ && $0.01$ && $0.1$ && $0.9$ && $0.99$ && $0.999$ && $0.9999$  \\ 
  \hline
   & 0.00 & 1.25 &  & 1.25 &  & 1.24 &  & 1.19 &  & 0.84 &  & 0.80 &  & 0.80 &  & 0.80 &  \\ 
   &  & 0.42 & 0.76 & 0.42 & 0.76 & 0.43 & 0.75 & 0.46 & 0.73 & 0.72 & 0.46 & 0.75 & 0.43 & 0.76 & 0.42 & 0.76 & 0.42 \\ 
   & 0.01 & 1.22 &  & 1.22 &  & 1.21 &  & 1.16 &  & 0.86 &  & 0.83 &  & 0.82 &  & 0.82 &  \\ 
   &  & 0.44 & 0.74 & 0.44 & 0.74 & 0.45 & 0.74 & 0.48 & 0.71 & 0.71 & 0.48 & 0.74 & 0.45 & 0.74 & 0.44 & 0.74 & 0.44 \\ 
   & 0.05 & 1.15 &  & 1.15 &  & 1.14 &  & 1.11 &  & 0.90 &  & 0.88 &  & 0.87 &  & 0.87 &  \\ 
0.5&  & 0.49 & 0.70 & 0.49 & 0.70 & 0.50 & 0.70 & 0.52 & 0.68 & 0.68 & 0.51 & 0.70 & 0.49 & 0.70 & 0.49 & 0.70 & 0.49 \\ 
   & 0.10 & 1.11 &  & 1.10 &  & 1.10 &  & 1.08 &  & 0.92 &  & 0.91 &  & 0.90 &  & 0.90 &  \\ 
   &  & 0.52 & 0.68 & 0.52 & 0.67 & 0.52 & 0.67 & 0.54 & 0.66 & 0.66 & 0.54 & 0.67 & 0.52 & 0.67 & 0.52 & 0.67 & 0.52 \\ 
   & 0.50 & 1.03 &  & 1.03 &  & 1.03 &  & 1.03 &  & 0.97 &  & 0.97 &  & 0.97 &  & 0.97 &  \\ 
   &  & 0.58 & 0.62 & 0.58 & 0.62 & 0.58 & 0.62 & 0.58 & 0.62 & 0.62 & 0.58 & 0.62 & 0.57 & 0.62 & 0.57 & 0.62 & 0.57 \\ 
   & 1.00 & 1.02 &  & 1.02 &  & 1.02 &  & 1.01 &  & 0.99 &  & 0.98 &  & 0.98 &  & 0.98 &  \\ 
   &  & 0.58 & 0.61 & 0.58 & 0.61 & 0.58 & 0.61 & 0.59 & 0.61 & 0.61 & 0.59 & 0.61 & 0.59 & 0.61 & 0.58 & 0.61 & 0.58 \\ 
   \hline
   & 0.00 & 2.57 &  & 2.56 &  & 2.46 &  & 1.86 &  & 0.54 &  & 0.41 &  & 0.39 &  & 0.39 &  \\ 
   &  & 0.33 & 0.93 & 0.33 & 0.93 & 0.34 & 0.92 & 0.45 & 0.87 & 0.87 & 0.45 & 0.92 & 0.34 & 0.93 & 0.33 & 0.93 & 0.33 \\ 
   & 0.01 & 2.11 &  & 2.11 &  & 2.06 &  & 1.71 &  & 0.58 &  & 0.48 &  & 0.47 &  & 0.47 &  \\ 
   &  & 0.40 & 0.89 & 0.40 & 0.89 & 0.41 & 0.89 & 0.49 & 0.85 & 0.85 & 0.48 & 0.89 & 0.41 & 0.89 & 0.40 & 0.89 & 0.40 \\ 
   & 0.05 & 1.62 &  & 1.62 &  & 1.61 &  & 1.46 &  & 0.69 &  & 0.62 &  & 0.62 &  & 0.62 &  \\ 
1  &  & 0.51 & 0.84 & 0.51 & 0.84 & 0.51 & 0.84 & 0.55 & 0.81 & 0.81 & 0.55 & 0.83 & 0.51 & 0.84 & 0.51 & 0.84 & 0.51 \\ 
   & 0.10 & 1.43 &  & 1.42 &  & 1.41 &  & 1.32 &  & 0.76 &  & 0.71 &  & 0.70 &  & 0.70 &  \\ 
   &  & 0.56 & 0.80 & 0.56 & 0.80 & 0.56 & 0.80 & 0.59 & 0.78 & 0.78 & 0.59 & 0.80 & 0.56 & 0.80 & 0.56 & 0.80 & 0.56 \\ 
   & 0.50 & 1.13 &  & 1.13 &  & 1.12 &  & 1.10 &  & 0.91 &  & 0.89 &  & 0.89 &  & 0.89 &  \\ 
   &  & 0.65 & 0.73 & 0.65 & 0.73 & 0.65 & 0.73 & 0.66 & 0.72 & 0.72 & 0.66 & 0.73 & 0.65 & 0.73 & 0.65 & 0.73 & 0.65 \\ 
   & 1.00 & 1.07 &  & 1.07 &  & 1.07 &  & 1.05 &  & 0.95 &  & 0.94 &  & 0.94 &  & 0.94 &  \\ 
   &  & 0.67 & 0.71 & 0.67 & 0.71 & 0.67 & 0.71 & 0.67 & 0.71 & 0.71 & 0.67 & 0.71 & 0.67 & 0.71 & 0.67 & 0.71 & 0.67 \\ 
   \hline
   & 0.00 & 100.00 &  & 62.90 &  & 20.90 &  & 4.55 &  & 0.22 &  & 0.05 &  & 0.02 &  & 0.01 &  \\ 
   &  & 0.10 & 1.00 & 0.14 & 1.00 & 0.30 & 0.99 & 0.60 & 0.96 & 0.96 & 0.60 & 0.99 & 0.30 & 1.00 & 0.14 & 1.00 & 0.10 \\ 
   & 0.01 & 8.56 &  & 8.47 &  & 7.63 &  & 3.81 &  & 0.26 &  & 0.13 &  & 0.12 &  & 0.12 &  \\ 
   &  & 0.47 & 0.98 & 0.47 & 0.98 & 0.49 & 0.98 & 0.63 & 0.95 & 0.95 & 0.63 & 0.98 & 0.49 & 0.98 & 0.47 & 0.98 & 0.47 \\ 
   & 0.05 & 3.75 &  & 3.73 &  & 3.60 &  & 2.64 &  & 0.38 &  & 0.28 &  & 0.27 &  & 0.27 &  \\ 
2  &  & 0.63 & 0.95 & 0.63 & 0.95 & 0.64 & 0.95 & 0.70 & 0.93 & 0.93 & 0.70 & 0.95 & 0.64 & 0.95 & 0.63 & 0.95 & 0.63 \\ 
   & 0.10 & 2.65 &  & 2.65 &  & 2.59 &  & 2.11 &  & 0.47 &  & 0.39 &  & 0.38 &  & 0.38 &  \\ 
   &  & 0.70 & 0.93 & 0.70 & 0.93 & 0.70 & 0.93 & 0.73 & 0.92 & 0.92 & 0.73 & 0.93 & 0.70 & 0.93 & 0.70 & 0.93 & 0.70 \\ 
   & 0.50 & 1.42 &  & 1.42 &  & 1.41 &  & 1.33 &  & 0.76 &  & 0.71 &  & 0.70 &  & 0.70 &  \\ 
   &  & 0.80 & 0.88 & 0.80 & 0.88 & 0.80 & 0.88 & 0.80 & 0.87 & 0.87 & 0.80 & 0.88 & 0.80 & 0.88 & 0.79 & 0.88 & 0.79 \\ 
   & 1.00 & 1.22 &  & 1.22 &  & 1.22 &  & 1.17 &  & 0.85 &  & 0.82 &  & 0.82 &  & 0.82 &  \\ 
   &  & 0.82 & 0.86 & 0.82 & 0.86 & 0.82 & 0.86 & 0.82 & 0.86 & 0.86 & 0.82 & 0.86 & 0.82 & 0.86 & 0.82 & 0.86 & 0.82 \\ 
   \hline
   & 0.00 & 5860.00 &  & 700.00 &  & 80.30 &  & 8.23 &  & 0.12 &  & 0.01 &  & 0.00 &  & 0.00 &  \\ 
   &  & 0.43 & 1.00 & 0.64 & 1.00 & 0.82 & 1.00 & 0.93 & 0.99 & 0.99 & 0.93 & 1.00 & 0.82 & 1.00 & 0.64 & 1.00 & 0.44 \\ 
   & 0.01 & 37.10 &  & 35.80 &  & 26.60 &  & 6.98 &  & 0.14 &  & 0.04 &  & 0.03 &  & 0.03 &  \\ 
   &  & 0.86 & 1.00 & 0.87 & 1.00 & 0.88 & 1.00 & 0.94 & 0.99 & 0.99 & 0.93 & 1.00 & 0.88 & 1.00 & 0.87 & 1.00 & 0.86 \\ 
   & 0.05 & 9.23 &  & 9.15 &  & 8.44 &  & 4.56 &  & 0.22 &  & 0.12 &  & 0.11 &  & 0.11 &  \\ 
4  &  & 0.93 & 0.99 & 0.93 & 0.99 & 0.93 & 0.99 & 0.95 & 0.99 & 0.99 & 0.95 & 0.99 & 0.93 & 0.99 & 0.93 & 0.99 & 0.93 \\ 
   & 0.10 & 5.27 &  & 5.24 &  & 5.01 &  & 3.38 &  & 0.30 &  & 0.20 &  & 0.19 &  & 0.19 &  \\ 
   &  & 0.94 & 0.99 & 0.94 & 0.99 & 0.94 & 0.99 & 0.96 & 0.99 & 0.99 & 0.96 & 0.99 & 0.94 & 0.99 & 0.94 & 0.99 & 0.94 \\ 
   & 0.50 & 1.90 &  & 1.89 &  & 1.87 &  & 1.66 &  & 0.60 &  & 0.54 &  & 0.53 &  & 0.53 &  \\ 
   &  & 0.97 & 0.98 & 0.97 & 0.98 & 0.97 & 0.98 & 0.97 & 0.98 & 0.98 & 0.97 & 0.98 & 0.97 & 0.98 & 0.97 & 0.98 & 0.97 \\ 
   & 1.00 & 1.45 &  & 1.45 &  & 1.44 &  & 1.35 &  & 0.74 &  & 0.69 &  & 0.69 &  & 0.69 &  \\ 
   &  & 0.97 & 0.98 & 0.97 & 0.98 & 0.97 & 0.98 & 0.97 & 0.98 & 0.98 & 0.97 & 0.98 & 0.97 & 0.98 & 0.97 & 0.98 & 0.97 \\ 
   \hline
\end{tabular}
\caption{Lines 1,3,\ldots: Optimal threshold $\delta^*$ for $\text{MCC}_{\text{rb}}$ with varying values of $d$ and $\Delta$. Lines 2,4,\ldots: Corresponding conditional probabilities $\pi_{1|1}$ and $\pi_{0|0}$.} \label{tab1:ex:mccgen}
\end{table}

\clearpage
\section{Additional plots for subsections \ref{ex:fgen} and \ref{ex:mccgen}} \label{app-qda}

\begin{figure}[h]
\centerline{\includegraphics[scale=0.7]{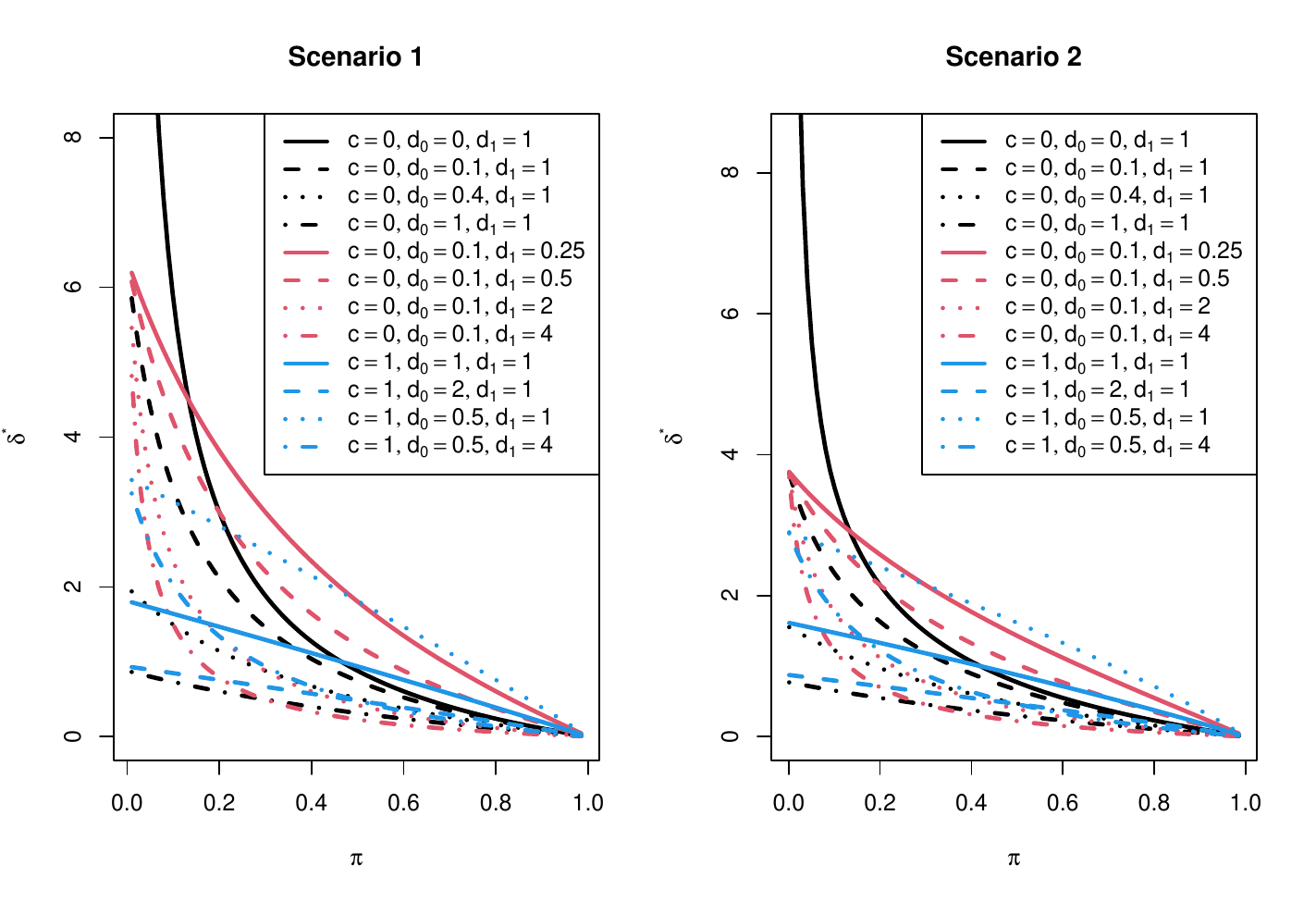}}
\caption{Plots of the optimal threshold $\delta^\ast$ as a function of $\pi$ for $\text{F}_{\text{rb}}$ with different choices of the parameters in the setting of Example \ref{ex:qda}.}  \label{fig:qda-frb}
\end{figure}

\begin{figure}   
\centerline{\includegraphics[scale=0.7]{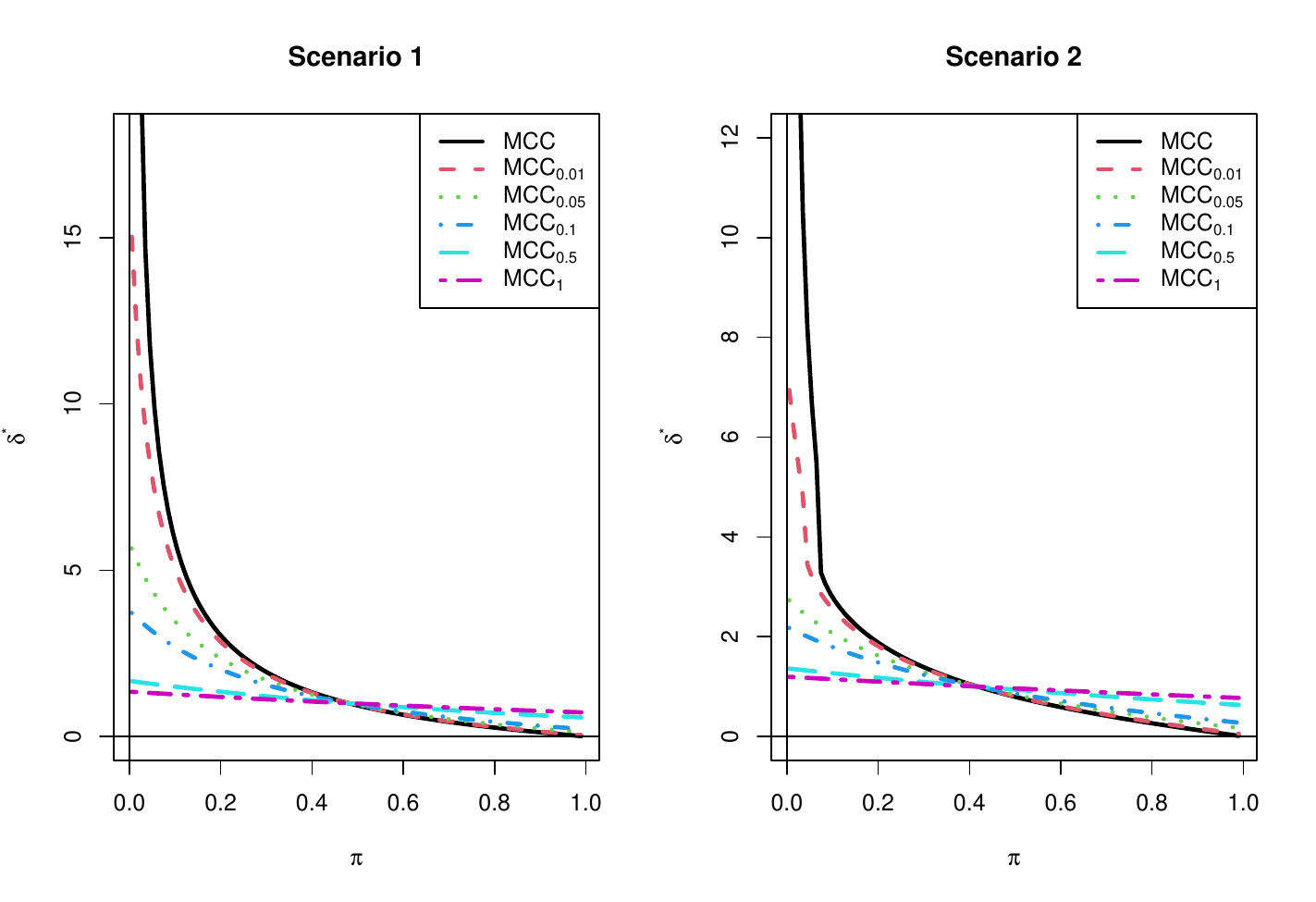}}
\caption{Plots of the optimal threshold $\delta^\ast$ as a function of $\pi$ for $\text{MCC}_{\text{rb}}$ with different choices of the parameter $d$ in the setting of Example \ref{ex:qda}.}  \label{fig:qda-mcc}
\end{figure}

\clearpage
\setlength{\tabcolsep}{12pt} %
\section{Additional tables and figures for Section 7} \label{appsec:dataexampletrainingdata}

\begin{table}[hb]
\centering
\begin{tabular}{llrrrrr}
  \hline
 &  & JAC & MCC & $F_{1.5}$ & $F_{0.5}$ & Kappa \\ 
  \hline
training & value & 0.184 & 0.264 & 0.322 & 0.339 & 0.263 \\ 
  $\hat{\pi}=0.067$ & $\tilde\delta$ & 0.123 & 0.135 & 0.104 & 0.169 & 0.135 \\ 
   & $\hat\pi_{1|1}$ & 0.314 & 0.286 & 0.385 & 0.209 & 0.286 \\ 
   & $\hat\pi_{0|0}$ & 0.950 & 0.959 & 0.911 & 0.978 & 0.959 \\ 
  \hline
subset.train & value & 0.086 & 0.148 & 0.179 & 0.165 & 0.145 \\ 
  $\hat{\pi}=0.014$ & $\tilde\delta$ & 0.043 & 0.036 & 0.035 & 0.061 & 0.050 \\ 
   & $\hat\pi_{1|1}$ & 0.195 & 0.236 & 0.244 & 0.133 & 0.168 \\ 
   & $\hat\pi_{0|0}$ & 0.982 & 0.974 & 0.972 & 0.991 & 0.986 \\ 
 \hline\hline
training & value & 0.283 & 0.400 & 0.477 & 0.457 & 0.398 \\ 
  $\hat{\pi}=0.067$ & $\tilde\delta$ & 0.220 & 0.190 & 0.154 & 0.383 & 0.238 \\ 
   & $\hat\pi_{1|1}$ & 0.489 & 0.528 & 0.577 & 0.286 & 0.468 \\ 
   & $\hat\pi_{0|0}$ & 0.948 & 0.938 & 0.921 & 0.982 & 0.953 \\ 
   \hline
subset.train & value & 0.130 & 0.221 & 0.252 & 0.218 & 0.217 \\ 
  $\hat{\pi}=0.014$ & $\tilde\delta$ & 0.126 & 0.114 & 0.107 & 0.153 & 0.126 \\ 
   & $\hat\pi_{1|1}$ & 0.274 & 0.300 & 0.314 & 0.216 & 0.274 \\ 
   & $\hat\pi_{0|0}$ & 0.984 & 0.981 & 0.979 & 0.989 & 0.984 \\ 
   \hline\hline
training & value & 0.302 & 0.423 & 0.494 & 0.485 & 0.423 \\ 
  $\hat{\pi}=0.067$ & $\tilde\delta$ & 0.218 & 0.231 & 0.160 & 0.355 & 0.233 \\ 
   & $\hat\pi_{1|1}$ & 0.505 & 0.491 & 0.579 & 0.351 & 0.489 \\ 
   & $\hat\pi_{0|0}$ & 0.952 & 0.955 & 0.930 & 0.978 & 0.955 \\ 
   \hline
subset.train & value & 0.173 & 0.286 & 0.313 & 0.357 & 0.285 \\ 
  $\hat{\pi}=0.014$ & $\tilde\delta$ & 0.145 & 0.182 & 0.119 & 0.279 & 0.151 \\ 
   & $\hat\pi_{1|1}$ & 0.311 & 0.256 & 0.358 & 0.143 & 0.302 \\ 
   & $\hat\pi_{0|0}$ & 0.989 & 0.993 & 0.984 & 0.998 & 0.989 \\ 
   \hline
\end{tabular}
\caption{Training data: Maximal value of performance measure, optimal cutoff $\tilde\delta$,  TPR $\hat\pi_{1|1}$ and TNR $\hat\pi_{0|0}$ of classifiers that empirically maximize the various  performance measures used in the data example. \\
Upper panel: Logistic regression. Middle panel: Random forests. Lower panel: Gradient boosting.} \label{tab-data-ex1}
\end{table}

\bigskip

\begin{table}
\centering
\begin{tabular}{llrrrrrr}
  \hline
 & $(d_0,d_1)$ & $(0,1)$ & $(0.1,1)$ & $(0.2,1)$ & $(0.5,1)$ & $(0.8,1)$ & $(0.2,2)$ \\ 
  \hline
training & value & 0.311 & 0.326 & 0.352 & 0.436 & 0.503 & 0.370 \\ 
  $\hat{\pi}=0.067$ & $\tilde\delta$ & 0.123 & 0.103 & 0.088 & 0.057 & 0.048 & 0.080 \\ 
   & $\hat\pi_{1|1}$ & 0.314 & 0.389 & 0.473 & 0.700 & 0.773 & 0.523 \\ 
   & $\hat\pi_{0|0}$ & 0.950 & 0.908 & 0.839 & 0.567 & 0.445 & 0.789 \\ 
  \hline
subset.train & value & 0.158 & 0.249 & 0.289 & 0.387 & 0.464 & 0.295 \\ 
  $\hat{\pi}=0.014$ & $\tilde\delta$ & 0.043 & 0.025 & 0.019 & 0.014 & 0.011 & 0.019 \\ 
   & $\hat\pi_{1|1}$ & 0.195 & 0.352 & 0.478 & 0.625 & 0.738 & 0.478 \\ 
   & $\hat\pi_{0|0}$ & 0.982 & 0.936 & 0.840 & 0.666 & 0.499 & 0.840 \\ 
 \hline\hline
training & value & 0.441 & 0.483 & 0.522 & 0.605 & 0.662 & 0.544 \\ 
  $\hat{\pi}=0.067$ & $\tilde\delta$ & 0.220 & 0.146 & 0.104 & 0.066 & 0.049 & 0.085 \\ 
   & $\hat\pi_{1|1}$ & 0.489 & 0.590 & 0.675 & 0.790 & 0.845 & 0.728 \\ 
   & $\hat\pi_{0|0}$ & 0.948 & 0.915 & 0.872 & 0.778 & 0.711 & 0.836 \\ 
   \hline
subset.train & value & 0.230 & 0.354 & 0.429 & 0.553 & 0.620 & 0.436 \\ 
  $\hat{\pi}=0.014$ & $\tilde\delta$ & 0.126 & 0.047 & 0.031 & 0.014 & 0.010 & 0.031 \\ 
   & $\hat\pi_{1|1}$ & 0.274 & 0.504 & 0.604 & 0.776 & 0.832 & 0.604 \\ 
   & $\hat\pi_{0|0}$ & 0.984 & 0.938 & 0.900 & 0.780 & 0.709 & 0.900 \\ 
   \hline\hline
training & value & 0.463 & 0.501 & 0.539 & 0.619 & 0.673 & 0.560 \\ 
  $\hat{\pi}=0.067$ & $\tilde\delta$ & 0.218 & 0.136 & 0.105 & 0.066 & 0.044 & 0.099 \\ 
   & $\hat\pi_{1|1}$ & 0.505 & 0.622 & 0.685 & 0.798 & 0.867 & 0.699 \\ 
   & $\hat\pi_{0|0}$ & 0.952 & 0.912 & 0.881 & 0.790 & 0.702 & 0.872 \\ 
   \hline
subset.train & value & 0.296 & 0.417 & 0.486 & 0.608 & 0.667 & 0.495 \\ 
  $\hat{\pi}=0.014$ & $\tilde\delta$ & 0.145 & 0.041 & 0.025 & 0.017 & 0.013 & 0.025 \\ 
   & $\hat\pi_{1|1}$ & 0.311 & 0.594 & 0.703 & 0.796 & 0.840 & 0.703 \\ 
   & $\hat\pi_{0|0}$ & 0.989 & 0.939 & 0.893 & 0.832 & 0.779 & 0.893 \\ 
   \hline
\end{tabular}
\caption{Training data: Maximal value of performance measure, optimal cutoff $\tilde\delta$,  TPR $\hat\pi_{1|1}$ and TNR $\hat\pi_{0|0}$ of classifiers that empirically maximize the robust F-score with $c=0$ and different choices of $(d_0,d_1)$ used in the data example. \\
Upper panel: Logistic regression. Middle panel: Random forests. Lower panel: Gradient boosting.} \label{tab-data-ex2}
\end{table}

\bigskip

\begin{table}
\centering
\begin{tabular}{llrrrrrr}
  \hline
 & $d$ & $0$ & $0.01$ & $0.05$ & $0.1$ & $0.5$ & $1.0$ \\ 
  \hline
training & value & 0.264 & 0.262 & 0.262 & 0.268 & 0.291 & 0.300 \\ 
  $\hat{\pi}=0.067$ & $\tilde\delta$ & 0.135 & 0.127 & 0.123 & 0.104 & 0.088 & 0.088 \\ 
   & $\hat\pi_{1|1}$ & 0.286 & 0.303 & 0.314 & 0.385 & 0.473 & 0.473 \\ 
   & $\hat\pi_{0|0}$ & 0.959 & 0.954 & 0.950 & 0.911 & 0.839 & 0.839 \\ 
  \hline
subset.train & value & 0.148 & 0.171 & 0.217 & 0.241 & 0.286 & 0.301 \\   
  $\hat{\pi}=0.014$ & $\tilde\delta$ & 0.036 & 0.029 & 0.025 & 0.025 & 0.019 & 0.019 \\ 
   & $\hat\pi_{1|1}$ & 0.236 & 0.303 & 0.352 & 0.352 & 0.478 & 0.478 \\ 
   & $\hat\pi_{0|0}$ & 0.974 & 0.956 & 0.936 & 0.936 & 0.840 & 0.840 \\ 
 \hline\hline
training & value & 0.400 & 0.407 & 0.431 & 0.453 & 0.520 & 0.539 \\
  $\hat{\pi}=0.067$ & $\tilde\delta$ & 0.190 & 0.190 & 0.154 & 0.117 & 0.085 & 0.085 \\ 
   & $\hat\pi_{1|1}$ & 0.528 & 0.528 & 0.577 & 0.647 & 0.728 & 0.728 \\ 
   & $\hat\pi_{0|0}$ & 0.938 & 0.938 & 0.921 & 0.888 & 0.836 & 0.836 \\    \hline
subset.train & value & 0.221 & 0.254 & 0.334 & 0.384 & 0.489 & 0.518 \\    
  $\hat{\pi}=0.014$ & $\tilde\delta$ & 0.114 & 0.063 & 0.033 & 0.031 & 0.017 & 0.017 \\ 
   & $\hat\pi_{1|1}$ & 0.300 & 0.438 & 0.588 & 0.604 & 0.736 & 0.736 \\ 
   & $\hat\pi_{0|0}$ & 0.981 & 0.956 & 0.907 & 0.900 & 0.816 & 0.816 \\ 
   \hline\hline
training & value & 0.423 & 0.429 & 0.450 & 0.474 & 0.537 & 0.558 \\ 
  $\hat{\pi}=0.067$ & $\tilde\delta$ & 0.231 & 0.198 & 0.136 & 0.116 & 0.085 & 0.076 \\ 
   & $\hat\pi_{1|1}$ & 0.491 & 0.528 & 0.622 & 0.660 & 0.736 & 0.765 \\ 
   & $\hat\pi_{0|0}$ & 0.955 & 0.946 & 0.912 & 0.894 & 0.844 & 0.821 \\ 
   \hline
subset.train & value & 0.286 & 0.314 & 0.399 & 0.448 & 0.560 & 0.591 \\ 
  $\hat{\pi}=0.014$ & $\tilde\delta$ & 0.182 & 0.065 & 0.041 & 0.036 & 0.017 & 0.017 \\ 
   & $\hat\pi_{1|1}$ & 0.256 & 0.491 & 0.594 & 0.622 & 0.796 & 0.796 \\ 
   & $\hat\pi_{0|0}$ & 0.993 & 0.964 & 0.939 & 0.929 & 0.832 & 0.832 \\ 
   \hline
\end{tabular}
\caption{Training data: Maximal value of performance measure, optimal cutoff $\tilde\delta$,  TPR $\hat\pi_{1|1}$ and TNR $\hat\pi_{0|0}$ of classifiers that empirically maximize the robust MCC with different choices of $d$ used in the data example. \\
Upper panel: Logistic regression. Middle panel: Random forests. Lower panel: Gradient boosting.} \label{tab-data-ex3}
\end{table}

\begin{figure}   
\centerline{\includegraphics[width=\textwidth]{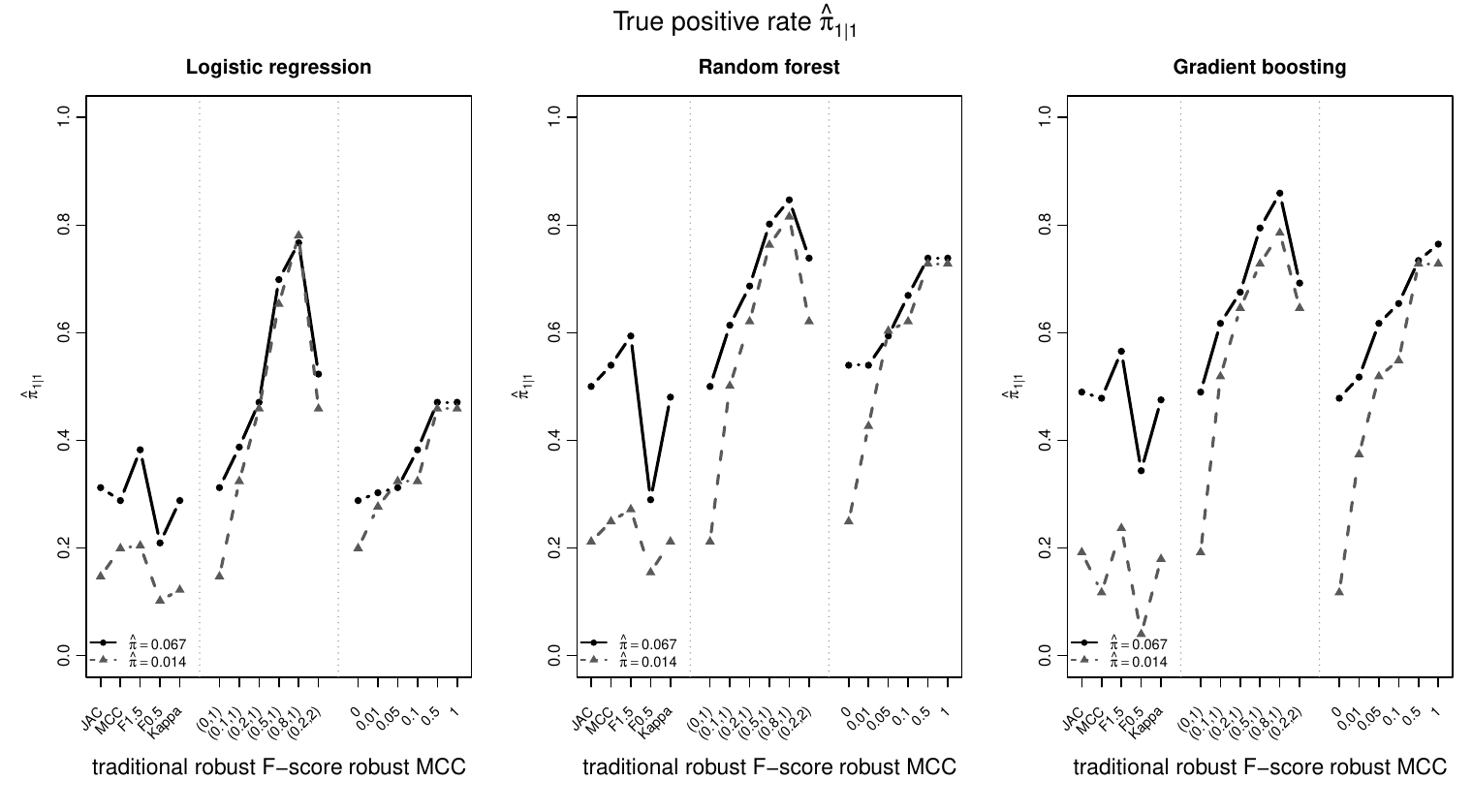}}
\caption{Empirical true positive rates $\hat{\pi}_{1|1}$ of classifiers that maximize the respective performance measures, shown for logistic regression, random forest, and gradient boosting. The left block displays classical metrics (JAC, MCC, $F_{1.5}$, $F_{0.5}$, Kappa), while the middle and right blocks correspond to the robust $F$-score and robust MCC with varying parameter choices. Solid lines represent the full dataset ($\hat{\pi}=0.067$), dashed lines the imbalanced subset ($\hat{\pi}=0.014$).} \label{fig:table789-pi11}
\end{figure}

\begin{figure}   
\centerline{\includegraphics[width=\textwidth]{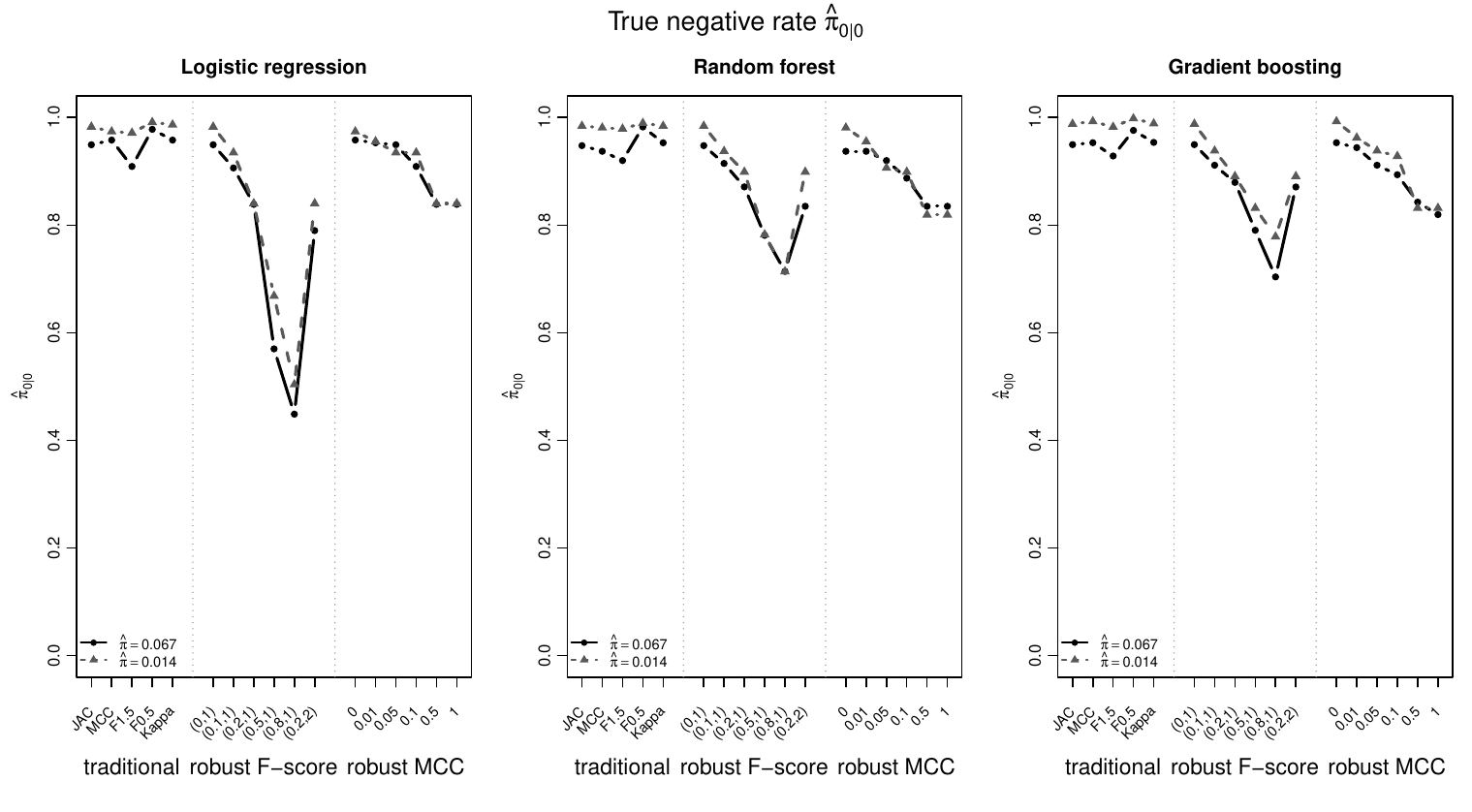}}
\caption{Empirical true negative rates $\hat{\pi}_{0|0}$ corresponding to the classifiers in Figure \ref{fig:table789-pi11}.} \label{fig:table789-pi00}
\end{figure}

\clearpage
\section{Section 7: Computational details and reproducibility} \label{sec:computational-details}

All computations were performed in R version~4.5.1 on a Windows 11 computer equipped with AMD Ryzen 7 7700X processor and 32 GB RAM. The core functions for the proposed performance metrics are available in the R package \texttt{RobustMetrics} on CRAN, \url{https://cran.r-project.org/package=RobustMetrics}. The scripts used for the data analysis are available at \\
\url{https://github.com/BernhardKlar/RobustMetrics-Rscripts}.
The principal package versions were \texttt{RobustMetrics}~{1.0.0},  \texttt{caret}~7.0.1, \texttt{randomForest}~4.7.1.2, \texttt{xgboost}~3.2.1.1, and \texttt{sn}~2.1.1.

For the credit-default application, the identifier variable was removed, as were \texttt{MonthlyIncome} and
\texttt{NumberOfDependents} because of their missing values. The data were divided into a stratified training sample containing 80\% of the observations and an untouched test sample containing the remaining 20\%. The more strongly imbalanced samples were constructed separately within the training and test samples by retaining all observations from class 0 and a random 20\% of the observations from class 1.

No tuning of the classification methods was performed. Logistic regression was fitted with the standard logit link using \texttt{glm}. The random-forest scores were obtained from a regression forest fitted to the numerical binary response, with 1000 trees and a bootstrap sample size of 10,000 observations per tree. The remaining parameters were left at their defaults, in particular \texttt{mtry}=2 and a minimum terminal-node size of 5. Gradient boosting was fitted using the objective \texttt{binary:logistic}, \texttt{max\_depth}=4, learning rate 0.05, \texttt{subsample}=0.8, \texttt{colsample\_bytree}=0.8, and 300 boosting iterations. To obtain reproducible execution times, XGBoost was run with one thread.

{Metric-specific thresholds were selected using in-sample training scores for logistic regression and gradient boosting, and out-of-bag predictions for random forests.} The metrics were evaluated on the grid $\widetilde\delta\in\{0.001,0.002,\ldots,0.999\}$.
In the event of ties, the smallest maximizing threshold was selected. The fitted model and the resulting training-selected thresholds were then applied unchanged to the test data. The test labels were therefore not used for either model fitting or threshold selection.

For $n$ observations and $K$ candidate thresholds, the direct grid search used here requires $O(Kn)$ operations. The evaluation of any classical or robustified metric for a given confusion matrix requires
$O(1)$ operations, so that the proposed robustifications have the same asymptotic evaluation cost as their classical counterparts. Here, $K=999$.

Table~\ref{tab:computation-times} reports elapsed times for the full data setting with $\widehat\pi=0.067$. %
The test-grid evaluation includes the calculation of the confusion matrices for all 999 thresholds, as
required for the graphical analyses.

\begin{table}[ht]
\centering \small
\caption{Elapsed computation times in seconds for the full credit-default data set.}
\label{tab:computation-times}
\begin{tabular}{lccc}
\hline
 & Logistic regression & Regression forest & XGBoost \\
\hline
Model fitting                         &  0.33 & 64.81 & 1.73 \\
Prediction on training data           &     -- &  {--}   & 0.19 \\
Threshold search on training data     & 51.38 & 52.80 & 49.71 \\
Prediction on test data               &     0 &  1.22 & 0.05 \\
Evaluation on the test threshold grid & 13.72 & 13.51 & 13.48 \\
\hline
\end{tabular}
\end{table}

{For random forests, out-of-bag training predictions are computed during model fitting, and their cost is included in the fitting time. The dash indicates that no separate training-prediction time is reported.}

\end{document}